\documentclass{article}

\usepackage{microtype}
\usepackage{graphicx}
\usepackage{subfigure}
\usepackage{booktabs} 
\usepackage{amssymb,amsmath, amsthm}
\usepackage{epstopdf,algorithm,algorithmic}
\usepackage{hyperref}
\input{mysymbol.sty}

\numberwithin{equation}{section}

\usepackage[dvipsnames]{xcolor}

\newcommand{\inclu}[0] {\ar@{^{(}->}}

\newtheorem{theorem}{Theorem}[section]

\newtheorem{lemma}[theorem]{Lemma}
\newtheorem{remark}[theorem]{Remark}

\newtheorem{corollary}[theorem]{Corollary}

\newcommand{\INDSTATE}[1][1]{\STATE\hspace{3mm}}
\newcommand{\INDSTATED}[1][1]{\STATE\hspace{6mm}}

\usepackage{mathtools}

\usepackage{jmlr2e} 

 \makeatother

\begin{document}

%
%
%
%
%
%
%
%
%
%

\jmlrheading{1}{2022}{pp}{mm/dd}{mm/dd}{Amrit Singh Bedi, Dheeraj Peddireddy,  Vaneet Aggarwal, Brian Sadler,  and Alec Koppel}


\ShortHeadings{Gaussian Process Bandits via Information Thresholding}{Bedi, Peddireddy, Aggarwal, Sadler, Koppel}
\firstpageno{1}

\title{\blue{Regret and Belief Complexity Trade-off in \\Gaussian Process Bandits via Information Thresholding}}

\author{\name Amrit Singh Bedi \email amritbd@umd.edu\\
       \addr Institute of Systems Research, \\
        University of Maryland,  \\
        College Park, MD, USA 20783 
       \AND
	       \name Dheeraj Peddireddy \email dpeddire@purdue.edu\\
       \addr School of Industrial Engineering \\
	Purdue University \\
	315 N. Grant Street, West Lafayette, IN 47907
              \AND
        \name Vaneet Aggarwal \email vaneet@purdue.edu  \\
       \addr School of Industrial Engineering \\
Purdue University \\
315 N. Grant Street, West Lafayette, IN 47907
		\AND
\name Brian M. Sadler \email brian.m.sadler6.civ@army.mil\\
\addr US Army Research Laboratory \\
Adelphi, MD, USA 20783
       \AND
			\name Alec Koppel$\dagger$ \email aekoppel@amazon.com\\
		\addr Supply Chain Optimization Technologies, \\
		Amazon. 320 108th Avenue NE, Bellevue, WA 98004 \thanks{A preliminary version of this work has appeared in \citep{pmlr-v120-bedi20a}. \\ $\dagger$Work completed when at the US Army Research Laboratory, Adelphi, MD, USA.}
}
\editor{}

\maketitle




\begin{abstract}
Bayesian optimization is a framework for global search via maximum a posteriori updates rather than simulated annealing, and has gained prominence for decision-making under uncertainty. In this work, we cast Bayesian optimization as a multi-armed bandit problem, where the payoff function is sampled from a Gaussian process (GP). Further, we focus on action selections via upper confidence bound (UCB) or expected improvement (EI) due to their prevalent use in practice. Prior works using GPs for bandits cannot allow the iteration horizon $T$ to be large, as the complexity of computing the posterior parameters scales cubically with the number of past observations. To circumvent this computational burden, we propose a simple statistical test: only incorporate an action into the GP posterior when its conditional entropy exceeds an $\epsilon$ threshold. Doing so permits us to \blue{precisely characterize the tradeoff between regret bounds of GP bandit algorithms and complexity of the posterior distributions depending on the compression parameter $\epsilon$ for both discrete and continuous action sets. To best of our knowledge, this is the first result which allows us to obtain sublinear regret bounds while still maintaining sublinear growth rate of the complexity of the posterior which is linear in the existing literature. Moreover, a provably finite bound on the complexity could be achieved but the algorithm would result in $\epsilon$-regret which means $\textbf{Reg}_T/T \rightarrow \mathcal{O}(\epsilon)$ as $T\rightarrow \infty$.} Experimentally, we observe state of the art accuracy and complexity trade-offs for GP bandit algorithms applied to global optimization, suggesting the merits of compressed GPs in bandit settings.

\end{abstract}
 
%
\section{Introduction}\label{sec:intro}

Bayesian optimization is a framework for global optimization of a black box function via noisy evaluations \citep{frazier2018tutorial}, and provides an alternative to simulated annealing \citep{kirkpatrick1983optimization,bertsimas1993simulated} or exhaustive search \citep{davis1991handbook}. These methods have proven adept at hyper-parameter tuning of machine learning models \citep{snoek2012practical,li2017hyperband}, nonlinear system identification \citep{srivastava2013optimal}, experimental design \citep{chaloner1995bayesian,press2009bandit}, and semantic mapping \citep{shotton2008semantic}.

More specifically, denote the function $f: \ccalX \rightarrow \reals $ we seek to optimize through noisy samples, i.e., for a given choice $\bbx_t\in\ccalX$, we observe $y_t = f(\bbx_t) + \epsilon_t$ sequentially. We make no assumptions for now on the convexity, smoothness, or other properties of $f$, other than each function evaluation must be selected judiciously. 
Our goal is to select a sequence of actions $\{\bbx_t\}$ that eventuate in competitive performance with respect to the optimal selection $\bbx^*=\argmax_{\bbx \in \ccalX} f(\bbx)$. For sequential decision making, a canonical performance metric is \emph{regret}, which quantifies the performance of a sequence of decisions $\{\bbx_t\}$ as compared with the optimal action $\bbx^*$:
\begin{equation}\label{eq:regret}
	\textbf{Reg}_T:=\sum_{t=1}^T (f(\bbx^*) - f(\bbx_t)).
\end{equation}
Regret in \eqref{eq:regret} is natural because at each time we quantify how far decision $\bbx_t$ was from optimal through the difference $r_t:=f(\bbx^*) - f(\bbx_t)$. An algorithm eventually learns the optimal strategy if it is no-regret: $\textbf{Reg}_T/T \rightarrow 0$ as $T\rightarrow \infty$.

In this work, we focus on Bayesian settings in which a likelihood model is hypothesized to relate the unknown function $f(\bbx)$ and action selection $\bbx\in\ccalX$. Then upon selecting an action $\bbx$, one tracks a posterior distribution, or \emph{belief model} \citep{powell2012optimal}, over possible outcomes $y=f(\bbx) + \epsilon$ which informs how the next action is selected. In classical Bayesian inference, posterior distributions do not influence which samples $(\bbx, y)$ are observed next \citep{ghosal2000convergence}. In contrast, in multi-armed bandits, action selection $\bbx$ determines which observations form the posterior, which is why it is also referred to as \emph{active learning} \citep{jamieson2015next}. 

Two key questions in this setting are how to specify a (i) likelihood and (ii) action selection strategy. These specifications come with their own merits and drawbacks in terms of optimality and computational efficiency. Regarding (i) the likelihood model, when the action space $\ccalX$ is discrete and of moderate size $X=|\ccalX|$, one may track a probability for each element of $\ccalX$, as in Thompson (posterior) sampling \citep{russo2018tutorial}, Gittins indices \citep{gittins2011multi}, and the Upper Confidence Bound (UCB) \citep{auer2002nonstochastic}. These methods differ in their manner of action selection, but not distributional representation. 

However, when the range of possibilities $X$ is large, computational challenges arise. This is because the number of parameters one needs to define a posterior distribution over $\ccalX$ is proportional to $X$, an instance of the curse of dimensionality in nonparametric statistics. One way to circumvent this issue for continuous spaces is to discretize the action space according to a pre-defined time-horizon that determines the total number of selected actions \citep{bubeck2011lipschitz,magureanu2014lipschitz}, and carefully tune the discretization to the time-horizon $T$. The drawback of these approaches is that as $T\rightarrow \infty$, the number of parameters in the posterior grows intractably large.

An alternative is to define a history-dependent distribution directly over the large (possibly continuous) space using, e.g., Gaussian Processes (GPs) \citep{rasmussen2004gaussian} or Monte Carlo (MC) methods \citep{smith2013sequential}. Bandit action selection strategies based on such distributional representations have been shown to be no-regret in recent years -- see   \citep{srinivas2012information,gopalan2014thompson}. While MC methods permit the most general priors on the unknown function $f$, computational and technical challenges arise when the prior/posterior no longer posses conjugacy properties \citep{gopalan2014thompson}.  By contrast, GPs, stochastic processes with any finite collection of realizations of which are jointly Gaussian \citep{krige1951statistical}, have a conjugate prior and posterior, and thus their parametric updates admit a closed-form -- see  \citep{rasmussen2004gaussian}[Ch. 2]. 

The conjugacy of the GP prior and posterior has driven its use in bandit action selection. In particular, by connecting regret to maximum information-gain based exploration, which upper-bounds the posterior variance \citep{srinivas2012information,de2012exponential}, no-regret algorithms may be derived through variance maximization. Doing so yields actions which over-prioritize exploration, which may be balanced through, e.g., upper-confidence bound (UCB) based action selection. GP-UCB algorithms, and variants such as expected improvement (EI) \citep{wang2014theoretical,nguyen2017regret}, and step-wise uncertainty reduction (SUR) \citep{villemonteix2009informational}, including knowledge gradient \citep{frazier2008knowledge}, have been shown to be no-regret or statistically consistent \citep{bect2019supermartingale}  in recent years.

However, these convergence results hinge upon requiring the use of the dense GP whose posterior distribution [cf. \eqref{eq:posterior_parameters}], has complexity cubic in $T$ due to the inversion of a Gram (kernel) matrix formed from the entire training set. Hence, the major limiting factor associated with the use of the GP-UCB algorithm in practice is its computational cost. For a given $|\mathcal{X}|$ number of actions (finite decision set) to choose from, GP-UCB algorithm exhibit $\mathcal{O}(|\mathcal{X}|T^2)$ per step runtime and $\mathcal{O}(|\mathcal{X}|T)$ space (memory) complexity for each new decision when run for a time horizon of length $T$. Numerous efforts to reduce the complexity of GPs exist in the literature -- see \citep{csato2002sparse,bauer2016understanding,bui2017streaming}. These methods all fix the complexity of the posterior and ``project" all additional points onto a fixed likelihood ``subspace." Doing so, however, may cause uncontrollable statistical bias and divergence. In \cite{calandriello2019gaussian}, by the proposing to use a number of inducing points of the order of effective dimension  $\mathcal{O}(d_{\text{eff}})$, one may reduce the complexity to $\mathcal{O}(|\mathcal{X}|d_{\text{eff}}^2)$ and $\mathcal{O}(|\mathcal{X}|d_{\text{eff}})$ for per step runtime and space complexity, respectively. However, the effective dimension  $d_{\text{eff}}$ cannot be computed in advance of runtime due to its dependence on a $T\times T$ Gram matrix of kernel evaluations, which makes the selection of the number of inducing inputs an open problem that may be partially solved via multiple training epochs \citep{calandriello2020near}. Contemporaneous work has established the effectiveness of greedy approximation of RKHS elements in the bandit setting \citep{takemori2020approximation}, but considers function space norms as the approximation criteria. Instead, here, we directly tune the approximation to the information-theoretic notion of regret.

More specifically, we explicitly design approximate GPs to ensure both small regret and moderate complexity. This goal is obtained by putting forth a compression metric based upon the conditional entropy of an action. That is, we retain actions in the GP representation whose conditional entropy exceeds an $\epsilon$-threshold. Doing so then allows us to come up with on-the-fly dynamically adjusted GP representations whose per step runtime is  $\mathcal{O}(|\mathcal{X}|T)$  and requires storing $\mathcal{O}(|\mathcal{X}|\sqrt{T})$ \blue{points via selection of  the compression metric as $\epsilon\leq \mathcal{O}(1/\sqrt{T})$}.\footnote{These dynamically self-adjusting nonparametric density estimates have recently been studied for estimation and regression in, e.g., \citep{koppel2019consistent,elvira2016adapting}.  } Importantly, one may select $\epsilon$ simply to be any positive (and often small) constant in practice, and hence does not depend on oracle knowledge of the Gram matrix of kernel evaluations as in \cite{calandriello2019gaussian,calandriello2020near}. Overall, then, our main contributions are as follows:
\begin{itemize}
	\item we propose a statistical test that operates inside GP-UCB or GP-EI which incorporates actions into the GP posterior only when conditional entropy exceeds an $\epsilon$ threshold (Sec. \ref{sec:prob}). We call these methods Compressed GP-UCB or Compressed GP-EI (Algorithm \ref{alg:cub}).
	\item derives sublinear regret bounds of GP bandit algorithms up to factors depending on the compression parameter $\epsilon$ for both discrete and continuous action sets (Sec. \ref{sec:convergence}). 
	\item \blue{establishes that the complexity of the GP posterior remains provably finite and depends on the compression budget $\epsilon$ }(Sec. \ref{sec:convergence}). 
	\item experimentally employs these approaches for optimizing non-convex functions and tuning the regularizer and step-size of a logistic regressor, which obtains a state of the art trade-off in regret versus computational efficiency relative to a few baselines \citep{srinivas2012information,calandriello2019gaussian}.  (Sec. \ref{sec:exp}). 
\end{itemize}

A preliminary version of this work has appeared in \citep{pmlr-v120-bedi20a}, but this work substantially expands upon it in the following ways: (1) the regret bounds in this work are tightened in terms of a more exact characterization of the relationship between the information gain over the points the algorithm retains. \blue{Moreover, based upon the choice of $\epsilon$ detailed in Sec. \ref{sec:convergence}, we obtain both sublinear regret and model complexity. This is strictly higher (better) than the condition $\epsilon=\mathcal{O}({1}/{{T}})$ in \citep{pmlr-v120-bedi20a}. (2) an explicit characterization of the dictionary size is provided that depends on the compression budget $\epsilon$, which is entirely absent from our earlier work, (3) all the detailed mathematical proofs are included, (4) on top of the refined analysis of UCB, we have additionally included the algorithmic and analytical details of the EI acquisition function, and (5) the experimental section is expanded to include further comparisons to other state of the art algorithms \citep{srinivas2012information,nguyen2017regret,wang2014theoretical,calandriello2019gaussian}. As compared to \cite{srinivas2012information}, our focus here is to precisely characterize the trade-off between regret and complexity. We achieve that by splitting the sum over regret, and relate this split to the design of our compression rule (cf. Sec. \ref{sec:convergence}). }
%
%

\section{Gaussian Process Bandits}\label{sec:prob}

{\bf \noindent Information Gain and Upper-Confidence Bound:} To find  $\bbx^*=\argmax_{\bbx\in\ccalX}f(\bbx)$ when $f$ is unknown, one may first globally approximate $f$ well, and then evaluate it at the maximizer. In order to formalize this approach, we propose to quantify how informative a collection of points $\{\bbx_u\}\subset \ccalX\subset\reals^p$ is through information gain \citep{cover2012elements}, a standard quantity that tracks the mutual information between $f$ and observations $y_u = f(\bbx_u) + \epsilon_u$ for all indices $u\in\mathcal{U}$ in some sampling set $\mathcal{U}\subset \mathbb{N}$, defined as
\begin{equation}\label{eq:information_gain}
	I(\{y_u\}  ; f) = H(\{y_u\} ) - H(\{y_u\} \given f)
\end{equation}
where $H(\{y_u\} ) $ denotes the entropy of observations $\{y_u\} $ and $ H(\{y_u\} \given f)$ denotes the entropy conditional on $f$. For a Gaussian $\ccalN(\mu,\Sigma)$ with mean $\mu$ and covariance $\Sigma$, the entropy is given as
\begin{equation}\label{eq:entropy_gaussian}
	H(\ccalN(\mu,\Sigma))=\frac{1}{2}\log|2\pi e \Sigma |
\end{equation}
which allows us to evaluate the information gain in closed form as
\begin{equation}\label{eq:information_gain_gaussian}
	I(\{y_u\}  ; f) = \frac{1}{2}\log|2 + \sigma^{-2}\bbK_t |.
\end{equation}
where $\bbK_t$ is a Gram matrix of kernel evaluations to be subsequently defined. 
Suppose we are tasked with finding a subset of $K$ points $\{\bbx_u\}_{u\leq T}$ that maximize the information gain. This amounts to a challenging subset selection problem whose exact solution cannot be found in polynomial time \citep{ko1995exact}. However, near-optimal solutions may be obtained via greedy maximization, as information gain is submodular \citep{krause2008near}. Maximizing information gain, i.e., selecting $\bbx_t = \argmax_{\bbx \in \ccalX} I(\{y_u\}  ; f) $, is equivalent to \citep{srinivas2012information}
\begin{equation}\label{eq:max_variance}
	\bbx_t = \argmax_{\bbx \in\ccalX} \sigma_{\bbX_{t-1}}(\bbx)
\end{equation}
where $\sigma_{\bbX_{t-1}}(\bbx)$ is the empirical standard deviation associated with a matrix $\bbX_{t-1}$ of data points $\bbX_{t-1}:=[\bbx_1 \; \cdots \bbx_{t-1} ]\in\reals^{d\times (t-1)}$.
We note that \eqref{eq:max_variance} may be shown to obtain the near-optimal selection of points in the sense that after $T$ rounds, executing \eqref{eq:max_variance} guarantees $$I(\{y_u\}_{{u=1}}^T  ; f)  \geq (1 - 1/e) I(\{y_u\}_{{u=1}}^K  ; f)$$ for some $K \leq T$ points via the theory of submodular functions  \citep{nemhauser1978analysis}. Indeed, selecting points based upon \eqref{eq:max_variance} permits one to efficiently \emph{explore} $f$ globally. However, it dictates that action selection does not move towards the actual maximizer $\bbx^*$ of $f$. Toward that end, instead $\bbx_t$ should be chosen according to prior knowledge about the function $f$, \emph{exploiting} information about where $f$ is large. To balance between these two extremes, a number of different acquisition functions $\alpha(\bbx)$ are possible based on the GP posterior -- see \citep{powell2012optimal}. Here, for simplicity, we propose to do so either based upon the upper-confidence bound (UCB):
\begin{equation}\label{eq:gp_ucb}
	\bbx_t = \argmax_{\bbx \in\ccalX}\underbrace{\mu_{\bbX_{t-1}}(\bbx) + \sqrt{\beta_t} \sigma_{\bbX_{t-1}}(\bbx)}_{\alpha^{\text{UCB}(\bbx)}}
\end{equation}
with $\beta_t$ as an exploration parameter $\beta_t$, or the expected improvement (EI) \citep{nguyen2017regret}, defined as
\begin{equation}\label{eq:expected_improvement}
	\bbx_t=\argmax_{\bbx\in\ccalX} \underbrace{\sigma_{t-1}\phi(z) + [\mu_{t-1}\!(\bbx) - y^{\text{max}}_{t-1}]\Phi(z)}_{\alpha^{\text{EI}(\bbx)}}\;,
\end{equation}
where $ y^{\text{max}}_{t-1}= \max\{y_u\}_{u\leq {t-1}}$ is the maximum observation value of past data, $z=z_{t-1}(\bbx) = (\mu_{t-1}(\bbx) - y^{\text{max}}_{t-1} )/\sigma_{t-1}(\bbx)$ is the $z$-score of $y^{\text{max}}_{t-1}$, and $\phi(z)$ and $\Phi(z)$ denote the density and distribution function of a standard Gaussian distribution. Moreover, the aforementioned mean $\mu_{t-1}(\bbx)$ and standard deviation $\sigma_{t-1}(\bbx)$ in the preceding expressions are computed via a GP, to be defined next. \medskip

{\bf \noindent Gaussian Processes:} A Gaussian Process (GP) is a stochastic process for which every finite collection of realizations is jointly Gaussian. We hypothesize a Gaussian Process prior for $f(\bbx)$, which is specified by a mean function $$\mu(\bbx)=\E{f(\bbx)}$$ and covariance kernel defined as $$\kappa(\bbx, \bbx') =  \E{(f(\bbx)-\mu(\bbx))^T(f(\bbx')-\mu(\bbx') )}.$$ Subsequently, we assume the prior is zero-mean $\mu(\bbx)=0$.
GPs play multiple roles in this work: as a way of specifying smoothness and a prior for unknown function $f$, as well as characterizing regret when $f$ is a sample from a known GP $GP(\bb0; \kappa(\bbx; \bbx'))$. GPs admit a closed form for their conditional a posteriori mean and covariance given training set $\bbX_{t}=\bbX_{t-1}\cup\{\bbx_t\}$ and $\bby_{t}=\bby_{t-1}\cup\{y_t\}$ as \citep{rasmussen2004gaussian}[Ch. 2]. 
\blue{\begin{align}\label{eq:posterior_parameters}
	\mu_{\bbX_t}(\bbx)& = \bbk_{t}(\bbx)^T(\bbK_{t} + \sigma^2 \bbI)^{-1} \bby_{t}  \\
	\sigma^2_{\bbX_t} (\bbx)& = \kappa(\bbx, \bbx)- \bbk_{t}(\bbx)^T(\bbK_{t} + \sigma^2 \bbI)^{-1} \bbk_{t}(\bbx)^T\nonumber
\end{align}}
where $\bbk_t (\bbx)=[\kappa(\bbx_1, \bbx), \cdots, \kappa(\bbx_t, \bbx)] $ denotes the empirical kernel map and $\bbK_t$ denotes the gram matrix of kernel evaluations whose entries are $\kappa(\bbx, \bbx')$ for $\bbx,\bbx'\in\{\bbx_u\}_{u\leq t}$. The $\bbX_t$ subscript underscores its role in parameterizing the mean and covariance. Further, note that \eqref{eq:posterior_parameters} depends upon a linear observation model $y_t = f(\bbx_t) + \epsilon_t$ with Gaussian noise prior $\epsilon_t \sim \ccalN (0,\sigma^2)$. 
The parametric updates \eqref{eq:posterior_parameters} depend on past actions {$\bbX_{t-1}$ and the current action $\bbx_t$}, which causes the \emph{kernel dictionary} {$\bbX_{t-1}$} to grow by one at each iteration, i.e., 
\begin{align}
	\bbX_{t} = [\bbX_{t-1} \ ; \ \bbx_{t}] \in \reals^{ d\times t}\; ,
\end{align} and the posterior at time \blue{$t$} uses \emph{all past observations} $\{\bbx_u\}_{u\leq t}$. Henceforth, the number of columns in the dictionary is called the \emph{model order}, which implies that the GP posterior at time {$t$ has model order $M_t=t$}. 

The resulting action selection strategy \eqref{eq:gp_ucb} using the GP \eqref{eq:posterior_parameters} is called GP-UCB, and its regret \eqref{eq:regret} is established in \citep{srinivas2012information}[Theorem 1 and 2] as sublinear with high probability up to factors depending on the maximum information gain $\gamma_T$ over $T$ points, which is defined as 
\begin{equation}\label{eq:max_info_gain}
	\gamma_T:=\max_{ \{\bbx_u \} } I(\{y_u\}_{{u=1}}
	^{T}  ; f) \ \text{ such that } \ |\{\bbx_u \} |=T.
\end{equation}
%

\begin{algorithm}[t]
	\caption{\textcolor{black}{Compressed GP-Bandits (CUB)}}\label{alg:cub}
	\begin{algorithmic}
		\FOR{t = 1,2...}
		\STATE Select action $\bbx_{t}$ via UCB \eqref{eq:gp_ucb} or EI \eqref{eq:expected_improvement}: $$\bbx_{t} = \arg\max_{\bbx\in\mathcal{X}} \alpha(\bbx)$$\\
		
		%
		
		%
		\STATE{\bf If} \blue{conditional entropy exceeds threshold $\frac{1}{2}\log\big(2\pi e(\sigma^{2}+{\sigma}^{2}_{\bbD_{t-1}}(\bbx_{t}))\big)>{\log\big(\sqrt{2\pi e\sigma^{2}}\big)+\epsilon}$}
		\INDSTATED {Sample: $ y_{t} = f(\bbx_{t})+\epsilon_{t}$}
		\INDSTATED Augment dictionary ${{\bbD}}_{t} = [{{\bbD}}_{t-1};\bbx_{t} ]$
		\INDSTATED Append $y_t$ to target vector $\bby_{\bbD_t}=[ \bby_{\bbD_{t-1}} ; y_t]$ 
		\INDSTATED Update posterior mean ${\mu}_{\bbD_t}(\bbx)$ \& variance ${\sigma}_{\bbD_t}(\bbx)$ 
		\begin{align*}
			\!\!\!\!&{\mu}_{\bbD_t}(\bbx)={\boldsymbol{k}}_{{{\bbD}}_{t}}(\bbx)^{T}(\bbK_{\bbD_t}+\sigma^{2}\bbI)^{-1}\bby_{{{\bbD}}_{t}}&
			\\		\!\!\!\!&\sigma_{\bbD_t}^{2}\!(\bbx)\!=\kappa(\bbx,\!\bbx)\!-\!{\boldsymbol{k}}_{{{\bbD}}_{t}}(\bbx)^T(\bbK_{\bbD_t,\bbD_t}\!\!+\!\!\sigma^{2}\bbI)^{-1}{\boldsymbol{k}}_{{{\bbD}}_{t}}(\bbx)
		\end{align*}
		\vspace{-4mm}
		\STATE{\bf else} 
		\INDSTATED Fix dict. $\bbD_t= \bbD_{t-1}$, target $\bby_{\bbD_t}= \bby_{\bbD_{t-1}}$, \& GP.
		$$(\mu_{\bbD_t}(\bbx),{\sigma}_{\bbD_t}(\bbx),{\bbD}_{t})=(\mu_{{\bbD_{t-1}}}(\bbx),{\sigma}_{\bbD_{t-1}}(\bbx),{\bbD}_{t-1})$$
		\ENDFOR
	\end{algorithmic}
\end{algorithm}
{\bf \noindent Compression Statistic:} The fundamental role of information gain in the regret of GP, using either UCB or EI,  provides a conceptual basis for finding a parsimonious GP posterior that nearly preserves no-regret properties of \eqref{eq:gp_ucb} - \eqref{eq:posterior_parameters}. To define our compression rule, first we define some key quantities related to approximate GPs. Suppose we select some other kernel dictionary $\bbD\in\reals^{d\times M}$ rather than $\bbX_t$ at time $t$, where $M$ is the  \emph{model order} of the Gaussian Process. Then, the only difference is that the kernel matrix $\bbK_t$ in \eqref{eq:posterior_parameters} and the empirical kernel map $\bbk_{t}(\cdot)$ are substituted by $\bbK_{\bbD\bbD}$ and $\bbk_{\bbD}(\cdot)$, respectively, where the entries of $[\bbK_{\bbD\bbD}]_{mn}=\kappa(\bbd_m, \bbd_n)$ and $\{\bbd_m  \}_{m=1}^M \subset \{\bbx_u\}_{u\leq t}$. Further, $\bby_\bbD$ denotes the sub-vector of $\bby_t$ associated with only the indices of training points in matrix $\bbD$. We denote the training subset associated with these indices as $\ccalS_{\bbD}:=\{\bbx_u,y_u\}_{u=1}^M$. By rewriting \eqref{eq:posterior_parameters} with $\bbD$ as the dictionary rather than $\bbX_{t+1}$, we obtain
\blue{\begin{align}\label{eq:GP_posterior_D}
	\bbmu_{\bbD}(\bbx)&=\bbk_{\bbD}(\bbx_{t+1})[\bbK_{\bbD,\bbD}+ \sigma^2 \bbI]^{-1}\bby_{\bbD}  \\
	\sigma_{\bbD}^{2}(\bbx)=&\kappa(\bbx,\bbx)-{\boldsymbol{k}}_{{\bbD}}(\bbx)^T(\bbK_{\bbD,\bbD}+\sigma^{2}\bbI)^{-1}{\boldsymbol{\kappa}}_{\bbD}(\bbx).\nonumber
\end{align}}
The question we address is how to select a sequence of dictionaries $\bbD_t \in \reals^{p\times M_t}$ whose $M_t$ columns comprise a set of points way less than $t$, i.e., $M_t \ll t$ and  approximately preserve the regret bounds of \citep{srinivas2012information}[Theorem 1 and 2].


Hence, we propose using conditional entropy as a statistic to compress against, i.e., a new data point should be appended to the Gaussian process posterior only when its conditional entropy is at least $\epsilon$, which results in the following update rule for the dictionary $\bbD_{t} \in \reals^{p\times M_t }$:
\begin{align}\label{eq:compression_rule}
	&\textbf{If } \blue{\ \mathbf{H}(y_{t}|\hat\bby_{t-1})=\frac{1}{2}\log\big(2\pi e(\sigma^{2}+{\sigma}^{2}_{\bbD_{t-1}}(\bbx_{t}))\big)>{\log\big(\sqrt{2\pi e\sigma^{2}}\big)+\epsilon}} \nonumber \\
	&\ \ \ \  {\text{Sample}:  y_{t} = f(\bbx_{t})+\epsilon_{t}} \nonumber \\
	&\ \ \ \  \text{update} \ \ \ \bbD_{t} =[\bbD_{t-1} \ ; \ \bbx_{t}] , \ \ {\bby_{\bbD_t}=[ \bby_{\bbD_{t-1}} ; y_t]}\nonumber \\
	&\textbf{else } \nonumber \\
	& 
	\ \ \ \  \text{update} \ \ \ \bbD_{t}= \bbD_{t-1} ,
\end{align}
where we define $\epsilon$ as the compression budget. This amounts to a statistical test of whether the selected action {$\bbx_t$ would yield} an informative sample $y_t$ in the sense that its conditional entropy exceeds an $\epsilon$ threshold.
We note that the posterior mean update of \eqref{eq:posterior_parameters} depends upon the sample $y_t$ via $\bby_t$ at $t$, but the covariance update is independent of the sample $y_t$ and is characterized by only the dictionary $\bbX_t$ at each $t$.  This allows us to determine the importance of the sample $y_t$ at $t$ without actually sampling it.
The modification of GP-UCB, called \emph{Compressed GP-UCB}, or \emph{CUB} for short, uses \eqref{eq:gp_ucb} with the lazy GP belief model \eqref{eq:GP_posterior_D} defined by dictionary updates \eqref{eq:compression_rule}. Similarly, the compression version of EI is called Compressed EI or CEI for short. We present them together for simplicity as Algorithm \ref{alg:cub} with the understanding that in practice, one must specify UCB \eqref{eq:gp_ucb} or EI \eqref{eq:expected_improvement}.
\begin{remark}\label{remark1}\normalfont
	To understand the intuition, consider the problem of multi-arm bandit problem with $K$ number of arms. Then, first we select an arm (we don't actually pull it), we check whether the conditional entropy  exceeds a threshold to determine whether this arm is informative enough. If the answer is yes, we pull the selected arm, otherwise, we just drop it and move forwards. Therefore, uninformative past decisions are dropped from belief formation about the present. 
	To be more precise,  as discussed in \cite{srinivas2012information}, we could select an action at each $t$ by either (1) maximizing the mean $\bbx_t=\arg\max_{\bbx\in\mathcal{X}}\mu_{\bbD_{t-1}}(\bbx)$ (pure exploitation), or (2) maximizing the covariance $\bbx_t=\arg\max_{\bbx\in\mathcal{X}}\sigma^2_{\bbD_{t-1}}(\bbx)$ (pure exploration). But a better strategy than (1)-(2) methods is to select actions using a combination of mean and covariance as mentioned in \eqref{eq:gp_ucb} which is proposed in \cite{srinivas2012information}.  In this work, we go one step further and include an additional feature to the algorithm proposed in \cite{srinivas2012information} which characterize the importance of each action mathematically. Through the use of a covariance based compression metric, we control the size of dictionary to grow boundedly (note that dictionary grows unboundedly in \cite{srinivas2012information}). \blue{For each action $\bbx_t$ (cf. Algorithm \ref{alg:cub}), we check if the posterior covariance
	\begin{align}
		\label{check}	
			\sigma^2_{\bbD_{t-1}}(\bbx_t) >\sigma^2 \left({\exp(2\epsilon)}-1\right),
	\end{align}
	and only then we perform the action $\bbx_t$ to sample $y_t$ and update the posterior distribution. For $\epsilon=0$, the condition in \eqref{check} trivially holds, and the algorithm reduces to the updates in \citep{srinivas2012information}. %
	Next, we rigorously establish how Algorithm \ref{alg:cub} trades off regret and memory (dictionary size) through the $\epsilon$ threshold dependent conditional entropy for whether a point $(\bbx_t, y_t)$ should be included in the GP.}
	
\end{remark}

%
\section{Balancing Regret and Complexity}\label{sec:convergence}

In this section, we establish that Algorithm \ref{alg:cub} achieves sublinear the regret \eqref{eq:regret} obtained by standard GP-UCB and its variants under the canonical settings of the action space $\mathcal{X}$ being a discrete finite set and a continuous compact Euclidean subset. We further establish sublinear regret of the expected improvement \eqref{eq:expected_improvement} when the action section $\mathcal{X}$ is discrete. We build upon techniques pioneered in \citep{srinivas2012information,nguyen2017regret}. The points of departure in our analysis are:  (i) the characterization of statistical bias induced by the compression rule \eqref{eq:compression_rule} in the regret bounds, and (ii) the relating of properties of the posterior \eqref{eq:compression_rule} and action selections \eqref{eq:gp_ucb}-\eqref{eq:expected_improvement} to topological properties of the action space $\ccalX$ to ensure the model order of the GP defined by \eqref{eq:GP_posterior_D} is at-worst finite for all $t$.
To evaluate the regret performance of the proposed algorithm, consider the definition of instantaneous regret $r_t=f(\bbx^*)-f(\bbx_t)$, which defines the loss we suffer $t$ by not taking the optimal action $\bbx^*$. Then after calculating the sum of $r_t$ from  $t=1$ to $T$ number of actions, we obtain the definition of regret $\textbf{Reg}_T$ given in \eqref{eq:regret}.

\blue{Next, before providing the details of regret analysis, first we establish the main merit of doing the statistical test inside a bandit algorithm which is that it controls the complexity of the belief model that decides action selections.
In particular, Theorem \ref{model_order} formalizes that the dictionary $\bbD_t$ defined by \eqref{eq:compression_rule} in Algorithm \ref{alg:cub} will always have finite number of elements $M_T(\epsilon)$ even if $T\rightarrow\infty$, which is stated next.}

\blue{\begin{theorem}\label{model_order}
	(\textbf{Model order}) 
	Suppose Algorithm \ref{alg:cub} is run with constant $\epsilon>0$. Then, the model order $M_T(\epsilon)$ of the distribution $\text{GP}(\mu_{\bbD_T}(\bbx), {\sigma_{\bbD_T}}^2(\bbx))$ is finite for all $T$, and subsequently the limiting distribution $\text{GP}(\mu_{\bbD_\infty}(\bbx), \sigma_{\bbD_\infty}^2(\bbx))$ has finite model complexity $M_{\infty}$, and it holds that $M_T(\epsilon)\leq M_{\infty}$ for all $T$. Moreover, we have 
	\begin{align}\label{main_theorem_2}
		M_T(\epsilon)\leq \mathcal{O}\left({\left(\frac{1}{\left({\exp(2\epsilon)}-1\right)}\right)^p}\right)\ \ \text{for all } T.
	\end{align}
	Furthermore, for {$\epsilon=\frac{1}{2}\log\left(1+T^{-\alpha}\right)$} with $\alpha\in(0,\frac{1}{p})$, we have {$M_T(\epsilon)\leq \mathcal{O}\left(T^{\alpha p}\right)$}.
\end{theorem}}

\blue{The implications of Theorem \ref{model_order} (proof provided in Appendix \ref{finite_model_order_proof}) are that Algorithm \ref{alg:cub} only retains significant actions in belief formation and drops extraneous points. Interestingly, this result states that despite infinitely many actions being taken in the limit, only finitely many of them are $\epsilon$-informative. The right hand side in \eqref{main_theorem_2} defines the upper bound on the model order in terms of threshold parameter $\epsilon$. Note that the upper bound is clearly infinite for the extreme case of $\epsilon=0$ which matches with out intuition.  To further solidify the understanding for a particular $T$, we choose $\epsilon=\frac{1}{2}\log\left(1+T^{-\alpha}\right)$, which results in a $T$ dependent upper bound as $\mathcal{O}\left(T^{\alpha p}\right)$ controlled by parameter $\alpha\in\left(0,\frac{1}{p}\right)$. We note that lower $\alpha$ mandates a larger value of budget parameter $\epsilon$, which implies fewer number of actions in the dictionary, and hence smaller model order. Hence, from the result in Theorem \ref{model_order}, we obtain a mathematical expression which controls the complexity of the posterior distribution. Observe, however, that we cannot make $\alpha$ arbitrary small (or make $\epsilon$ arbitrary large) to reduce the number of elements in the dictionary, as doing so results in  larger regret, as we formalize in the following theorem.  In particular, we present the regret performance of Algorithm \ref{alg:cub} when actions are selected according to the CUB (cf. Algorithm \ref{alg:cub}) next.}

\begin{theorem}\label{lemma:main}{(\bf Regret of Compressed GP-UCB)}
	Fix $\delta\in(0,1)$ and suppose the Gaussian Process prior for $f$ has zero mean with covariance kernel $\kappa(\bbx,\bbx')$. Define constant  $C_1:=8/\log(1+\sigma^{-2})$ Then under the following parameter selections and conditions on the data domain $\ccalX$, we have:
	\begin{enumerate}
		\item \blue{(\textbf{Finite decision set}) For finite cardinality $|\ccalX|=X$, with exploration parameter $\beta_t$ selected as $$\beta_t=2\log(X t^2 \pi^2/6\delta),$$ the compression budget {$\epsilon=\frac{1}{2}\log\left(1+T^{-\alpha}\right)$} with $\alpha\in(0,\frac{1}{p})$, then the accumulated regret is sublinear regret [cf. \eqref{eq:regret}] with probability $1-\delta$. 
		\begin{equation}\label{eq:finite_cardinality_template0}
			\mathbb{P}\left\{{\textbf{Reg}}_T \leq \sqrt{C_{1}T\beta_{T} \gamma_{M_T(\epsilon)} }+ {T\sqrt{\log\left(1+T^{-\alpha}\right)} \sqrt{{C_{1}\beta_{T}}}}\right\} \geq 1-\delta
		\end{equation}
		where $\gamma_{M_T(\epsilon)}$ is the information gain corresponding to $M_T(\epsilon)$ number of elements in the dictionary after $T$ number of iterations.}
		
		\item  \blue{(\textbf{General decision set}) For continuous set $\ccalX \subset [0,r]^d$, assume the derivative of the GP sample paths are bounded with high probability, i.e., for constants $a,b$,
		\begin{equation}\label{eq:GP_derivative_bounded11}
			\mathbb{P}\left\{\sup_{\bbx \in \ccalX} |\partial f / \partial \bbx_j | > L \right\} \leq a e^{-(L/b)^2} \  \text{ for } j=1,..,d.
		\end{equation}
		Then, under exploration parameter $$\beta_t=2\log(X t^2 \pi^2/3\delta) + 2 d \log(t^2 d b r\sqrt{\log(4da/\delta )  }  ),$$ the compression budget {$\epsilon=\frac{1}{2}\log\left(1+T^{-\alpha}\right)$} with $\alpha\in(0,\frac{1}{p})$, the accumulated regret  [cf. \eqref{eq:regret}] is \label{lemma:template_continuous2}
		\begin{equation}\label{eq:continuous_cardinality_template3}
			\mathbb{P}\left\{{\textbf{Reg}}_T \leq \sqrt{C_{1}T\beta_{T} \gamma_{M_T(\epsilon)} }+ {T\sqrt{\log\left(1+T^{-\alpha}\right)} \sqrt{{C_{1}\beta_{T}}}} + \frac{\pi^2}{6}\right\} \geq 1-\delta
		\end{equation}
		where $\gamma_{M_T(\epsilon)}$ is the information gain corresponding to $M_T(\epsilon)$ number of elements in the dictionary after $T$ number of iterations.}
	\end{enumerate}
\end{theorem}
%
%

Theorem \ref{lemma:main}, whose proof is Appendix \ref{proof_theorem_2}, establishes that Algorithm \ref{alg:cub} with the action selected according to \eqref{eq:gp_ucb} attains sublinear regret with high probability when the action space $\mathcal{X}$ is discrete and finite, as well as when it is a continuous compact subset of Euclidean space, up to factors depending on the maximum information gain \eqref{eq:max_info_gain} and the compression budget $\epsilon$ in \eqref{eq:compression_rule}. The sublinear dependence of the information gain on $T$ in terms of the parameter dimension $p$ is derived in \citep{srinivas2012information}[Sec. V-B] for common kernels such as the linear, Gaussian, and Mat\'{e}rn.

\blue{The proof follows a path charted in \citep{srinivas2012information}[Appendix I], except that we must contend with the compression-induced error. Specifically, we begin by computing the confidence interval for each action $ \bbx_t$ taken by the proposed algorithm at time $t$. Then, we bound the instantaneous regret $r_t:=f(\bbx^*)-f( \bbx_t)$ in terms of the problem parameters such as $\beta_t$, $\delta$, $C_1$,  compression budget $\epsilon$, and information gain $\gamma_{M_T(\epsilon)}$ using the fact that the upper-confidence bound overshoots the maximizer. By summing over time instances for which we selection actions $\bbx_t$ which is from $t=1$ to $T$, we build an upper-estimate of cumulative regret based on instantaneous regret $r_t$. Unsurprisingly, the obtained upper bound on the regret depends upon  $\epsilon$, which we select as $\epsilon=\frac{1}{2}\log\left(1+T^{-\alpha}\right)$ with $\alpha=(0,1/p)$. In Table \ref{results}, we summarize the specific values of regret and model order for different values of compression metric $\epsilon$ (or for different values of $\alpha$).}
\blue{\begin{table}[h]
		\centering
		\color{black}{\resizebox{1\textwidth}{!}	{\begin{tabular}{|c|c|c|c|c|c|c|} 
					\cline{1-7}
					& \multicolumn{2}{|c|}{{{$\alpha=1/p$  }}}                                                                                                                                                                                                                                                               & \multicolumn{2}{|c|}{$\alpha=1/(2p)$}
					& \multicolumn{2}{|c|}{$\alpha=0$}
					\\
					\cline{1-7}
					&${\textbf{Reg}}_T$		& $M_T$		 &${\textbf{Reg}}_T$		& $M_T$	&${\textbf{Reg}}_T$		& $M_T$	
					\\
					\cline{1-7}	
					CGP-UCB$^{1}$			&\blue{$I+\mathcal{O}\left(T^{1-(1/4p)}\right)$}		& \blue{$\mathcal{O}(T)$}		&\blue{I+$\mathcal{O}\left(T^{1-(1/2p)}\right)$}		& \blue{$\mathcal{O}(\sqrt{T})$	}&\blue{I+$\mathcal{O}\left(T\right)$}		& \blue{$\mathcal{O}(1)$	}
					\\
					\cline{1-7}	
					CGP-UCB$^{2}$			&\blue{$I+\mathcal{O}\left(T^{1-(1/4p)}+\frac{\pi^2}{6}\right)$}		& \blue{$\mathcal{O}(T)$}	 &\blue{$I+\mathcal{O}\left(T^{1-(1/2p)}+\frac{\pi^2}{6}\right)$}		& \blue{$\mathcal{O}(\sqrt{T})$	}&\blue{$I+\mathcal{O}\left(T+\frac{\pi^2}{6}\right)$}		& \blue{$\mathcal{O}(1)$	}
					\\
					\cline{1-7}	
					CGP-EI			&\blue{$I+\mathcal{O}\left(T^{1-(1/4p)}\right)$}			& \blue{$\mathcal{O}(T)$}		  &\blue{$I+\mathcal{O}\left(T^{1-(1/4p)}+\frac{\pi^2}{6}\right)$}	&\blue{$\mathcal{O}(\sqrt{T})$}		&\blue{$I+\mathcal{O}\left(T+\frac{\pi^2}{6}\right)$}	&\blue{$\mathcal{O}(1)$}
					\\
					\hline
		\end{tabular}}}
		\caption{{\blue{The term $I$ in the table which is same for all cases is given by $I:=\sqrt{C_{1}T\beta_{T} \gamma_{M_T(\epsilon)} }$.  Table provides the results for Regret (cf. \eqref{eq:regret}) and model order (cf. \eqref{main_theorem_2}) for different values of $\alpha$ (and hence $\epsilon$).  The bounds specializes to existing results in the literature for $\epsilon=0$. We can clearly see the tradeoff between regret and model order for $\alpha=1/p$ and $\alpha=0$. Note that for $\alpha=1/p$ the regret is clearly sublinear but obtain linear model order complexity. On the other extreme, for $\alpha=0$, we have finite model order complexity but with linear regret. But for $\alpha=1/(2p)$, we can see that we obtain sublinear regret as well as the model order.}}}
		\label{results}
\end{table}}


\blue{We note that since the regret bounds developed in Theorem \ref{lemma:main} depend upon the amount of information gain, we further develop an upper bound on the information gain for the case when $\mathcal{X}$ is a finite decision set, which is extendible to general decision sets using the analysis presented in \cite[Appendix III]{srinivas2012information}.  For simplicity and better intuition of results, we consider $p=1$ for the next Corollary \ref{cor:1}.
\begin{corollary}\label{cor:1}
	For finite decision set $\mathcal{X}$, with exploration parameter $\beta_t$ selected as $\beta_t=2\log(X t^2 \pi^2/6\delta),$ the regret bound of Algorithm \ref{alg:cub}  with probability $1-\delta$ is given by  
	\begin{equation}\label{corollary}
		\mathbb{P}\left\{{\textbf{Reg}}_T \leq \sqrt{C_{1}\beta_{T}}\max\{\sqrt{Z_1},1\} T^{3/4}\right\} \geq 1-\delta,
	\end{equation}
	where $Z_1=\frac{\sum_{t=1}^{|\mathcal{X}|}\lambda_t}{\sigma^2(1-e^{-1})}$. The model order is $M_T(\epsilon)\leq \mathcal{O}(\sqrt{T})$.
\end{corollary}
The proof of Corollary \ref{cor:1} is provided in Appendix \ref{cor_1_proof}. From the result in the corollary, it is clear that it is possible to obtain a sublinear regret while still maintaining sublinear growth rate for dictionary size. This is advantageous relative to existing results for GP-UCB require linear growth rate of the dictionary size in order to  obtain sublinear regret-- see \cite{srinivas2012information} for details. Thus, the proposed scheme refines the state of the art.}
%
%
Next, we analyze the performance of Algorithm \ref{alg:cub} when actions are selected according to the expected improvement \eqref{eq:expected_improvement}. 
%

\blue{\begin{theorem}\label{theorem:ei_regret}{(\bf Regret of Compressed GP-EI)}
	Suppose we select actions based upon Expected Improvement \eqref{eq:expected_improvement} together with the conditional entropy-based rule \eqref{eq:compression_rule} for retaining past points into the GP posterior, as detailed in Algorithm \ref{alg:cub}. Then, under the same conditions and parameter selection $\beta_t$ as in Theorem \ref{lemma:main}, when $\ccalX$ is a finite discrete set, the regret $\widetilde{\textbf{Reg}}_T$ is {given below} with probability $1-\delta$, i.e., 
	\begin{align}\label{eq:ei_regret}
		\mathbb{P}&\!\!\left\{\!{\textbf{Reg}}_T \! \leq \!\frac{\sqrt{3(\beta_T +1 +R^2)} + \sqrt{\beta_T}}{\log\left(1+\sigma^{-2}\right)}\left(\sqrt{2T \gamma_{M_T(\epsilon)} }
	+ {T\sqrt{\log\left(1+T^{-\alpha}\right)} }\right) \right\} \geq 1- \delta \; ,
	\end{align}
	where $$R:=\sup_{t\geq 0}\sup_{\bbx\in\ccalX} \frac{|\mu_{\bbD_{t-1}}(\bbx) - y^{{\max}}|}{\sigma_{\bbD_{t-1}}(\bbx)}$$ is the maximum value of the $z$ score, is as defined in Lemma \ref{lemma:z_score}.
\end{theorem}}
%
%
The proof is proved in Appendix \ref{sec:appendix_ei}. In Theorem \ref{theorem:ei_regret}, we have characterized how the regret of Algorithm \ref{alg:cub} depends on the compression budget $\epsilon$ for when the actions are selected according to the EI rule. We note  \eqref{eq:ei_regret} holds for the discrete action space $\mathcal{X}$. The result for the continuous action space $\mathcal{X}$ follows from the proof of statement (ii) of Theorem \eqref{lemma:main} and the proof of Theorem \ref{theorem:ei_regret}. The proof of Theorem \ref{theorem:ei_regret} follows a similar path presented in the \citep{nguyen2017regret}. We start by upper bounding the instantaneous improvements achieved by the proposed compressed EI algorithm in terms of the acquisitions function in Lemma \ref{lemma:expected_instantaneous_improvement}. Further, the sum of the predictive variances for the compressed version over {$t\in\mathcal{M}_T(\epsilon)$} instances is upper bounded in terms of the maximum information gain {$\gamma_{M_T(\epsilon)}$} in Lemma \ref{lemma:posterior_variance_info_gain}. Then we upper bound the cumulative sum of the instantaneous regret $r_t=f(\bbx^*) - f({\bbx}_t)$ in terms of the model parameters such as $\gamma_T$, $\sigma$, $\beta_T$, $R$, and $\epsilon$.  Moreover, note that $\epsilon=0$ reduces to the result of \citep{nguyen2017regret}.   
%

In the next section, we evaluate the merit of these conceptual results on experimental settings involving black box non-convex optimization and hyper-parameter tuning of linear logistic regressors.


\section{Experiments}\label{sec:exp}
In this section, we evaluate the performance of the statistical compression method under a few different action selections (acquisition functions). Specifically, Algorithm \ref{alg:cub} employs the Upper Confidence Bound (UCB) or Expected Improvement (EI) \citep{nguyen2017regret} acquisition function, but the key insight here is a modification of the GP posterior, not the action selection. Thus, we validate its use for Most Probable Improvement (MPI) \citep{wang2014theoretical} as well, defined as
\begin{align}\label{eq:EIMPI}
	\alpha^{\text{MPI}}(\bbx)&=\sigma_{\bbD_{t-1}}\phi(z) + [\mu_{\bbD_{t-1}}(\bbx) - \xi]\Phi(z)\;, \nonumber \\
	\xi&=\argmax_{\bbx} \mu_{\bbD_{t-1}}(\bbx) \ \ \nonumber 
\end{align}
where $\phi(z)$ and  $\Phi(z)$ denote the standard Gaussian density and distribution functions, and $z=(\mu_{\bbD_{t-1}}(\bbx) - \xi)/\sigma_{\bbD_{t-1}}(\bbx)$ is the centered $z$-score.
We further compare the compression scheme against Budgeted Kernel Bandits (BKB) proposed by \citep{calandriello2019gaussian} which proposes to randomly add or drop points according to a distribution \blue{that is directly proportional to the posterior variance}, also on the aforementioned acquisition functions. 

Unless otherwise specified, the squared exponential kernel is used to represent the correlation between the input, the length-scale is set to $\theta$ = 1.0, the noise prior is set to $\sigma^2$ = 0.001, the compression budget $\epsilon$ = $10^{-4}$ and the confidence bounds hold with probability of at least $\delta$ = 0.9. As a common practice across all three problems, we initialize the Gaussian priors with $2^{d}$ training data randomly collected from the input domain, where d is the input dimension. We quantify the performance using the Mean Average Regret over the iterations and the clock time. In addition, the model order, or the number of points defining the GP posterior, is visualized over time to characterize the compression of the training dictionary. To ensure fair evaluations, all the listed simulations were performed on a PC with a 1.8 GHz Intel Core i7 CPU and 16 GB memory. Same initial priors and parameters are used to assess computational efficiency in terms of the compression.

\begin{figure}[t]\hspace{-1mm}
	\subfigure[UCB]{\includegraphics[scale=0.35]
		{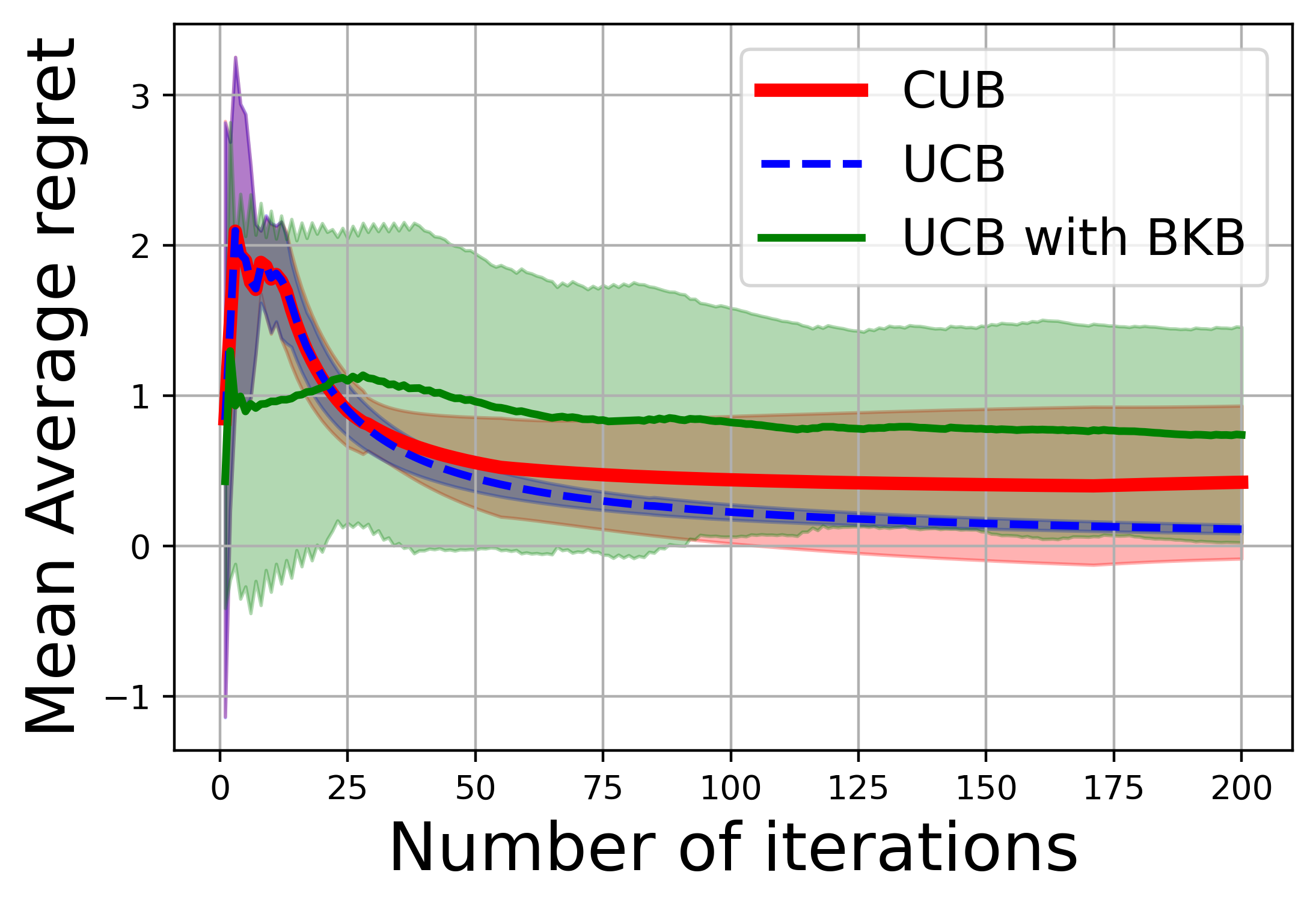}\label{subfig:xamplUCBmar}}\hspace{-1mm}
	\subfigure[EI]{\includegraphics[scale=0.35]
		{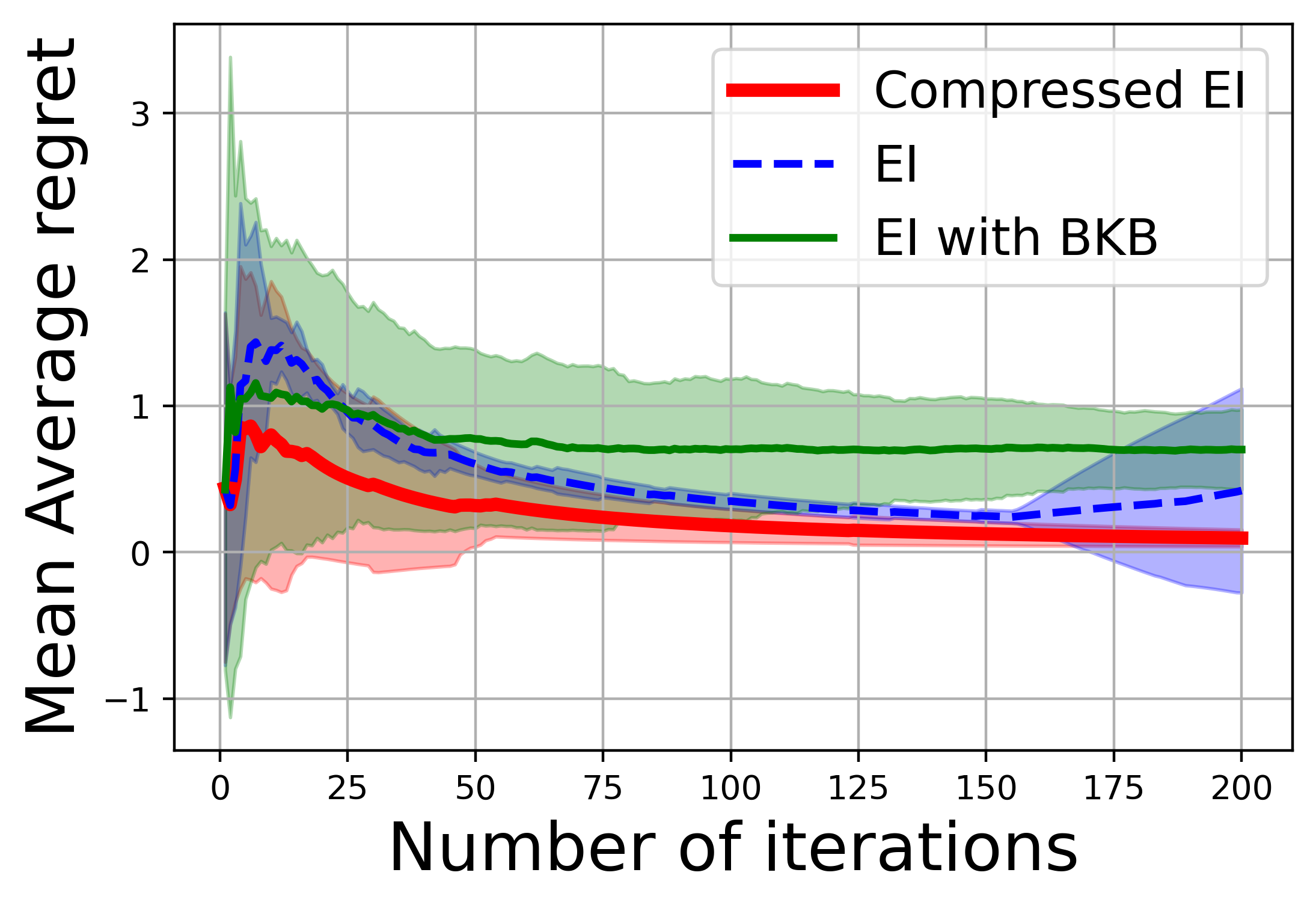}\label{subfig:xamplEImar}}\hspace{-1mm}
	\subfigure[MPI]{\includegraphics[scale=0.35]
		{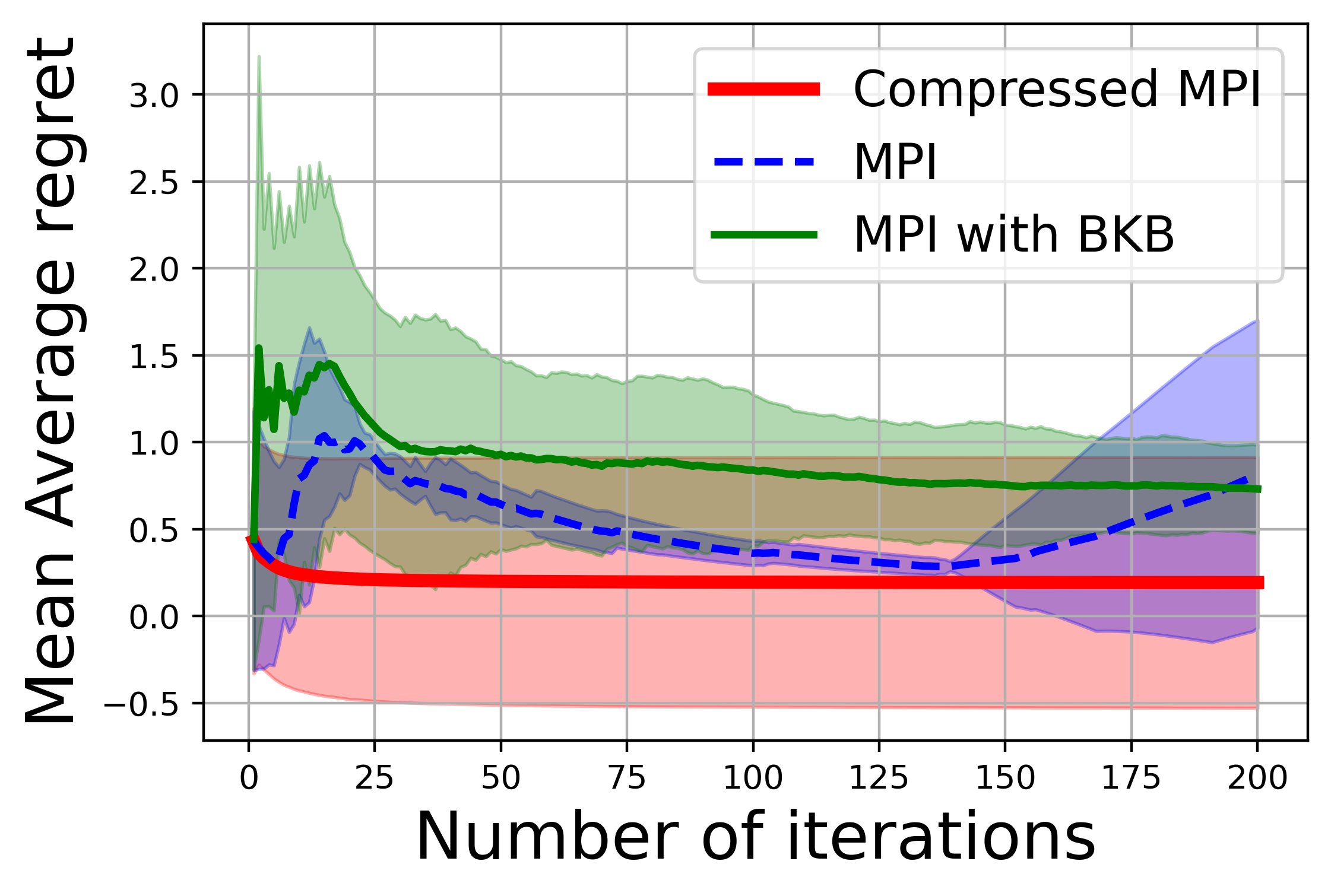}\label{subfig:xamplMPImar}} \vspace{-0mm}
	\\
 	\subfigure[UCB]{\includegraphics[scale=0.35]
		{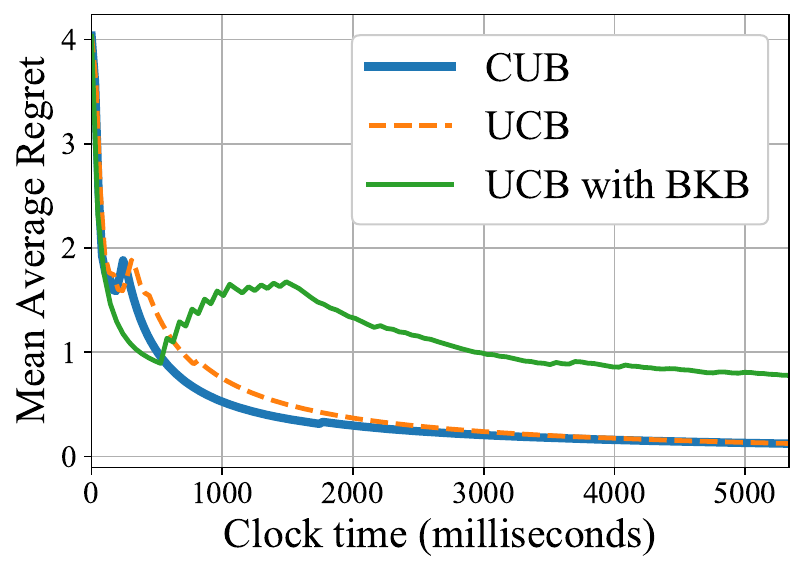}\label{subfig:xamplUCBCT}}\hspace{-0mm}
	\subfigure[EI]{\includegraphics[scale=0.35]
		{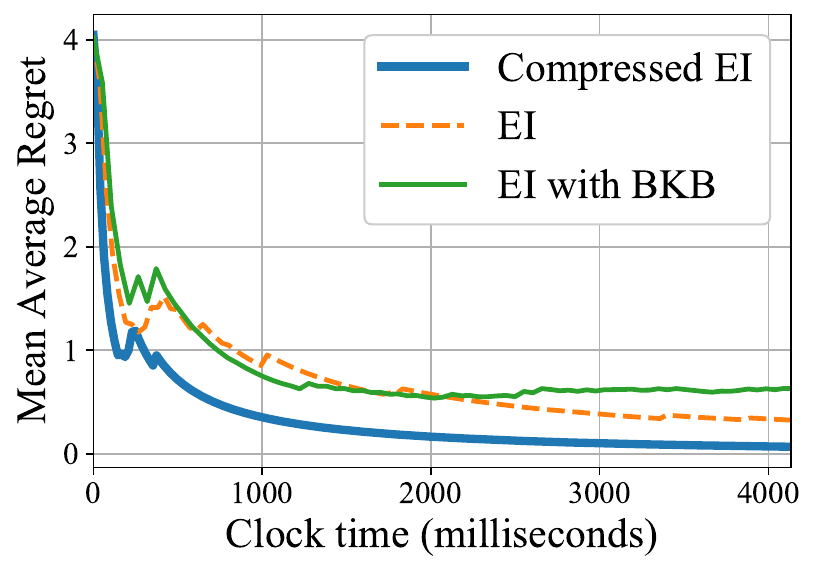}\label{subfig:xamplEICT}}\hspace{0mm}
	\subfigure[MPI]{\includegraphics[scale=0.35]
		{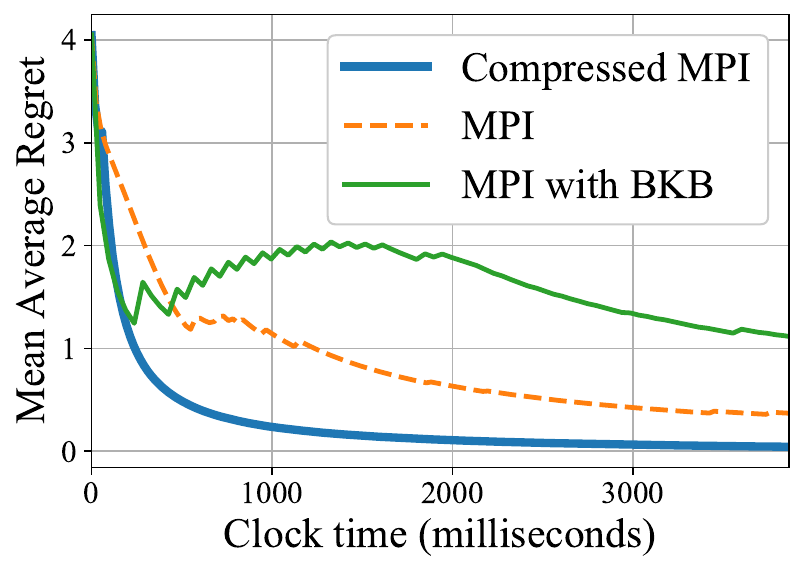}\label{subfig:xamplMPICT}}\vspace{0mm}
	\\
	\subfigure[UCB]{\includegraphics[scale=0.35]
		{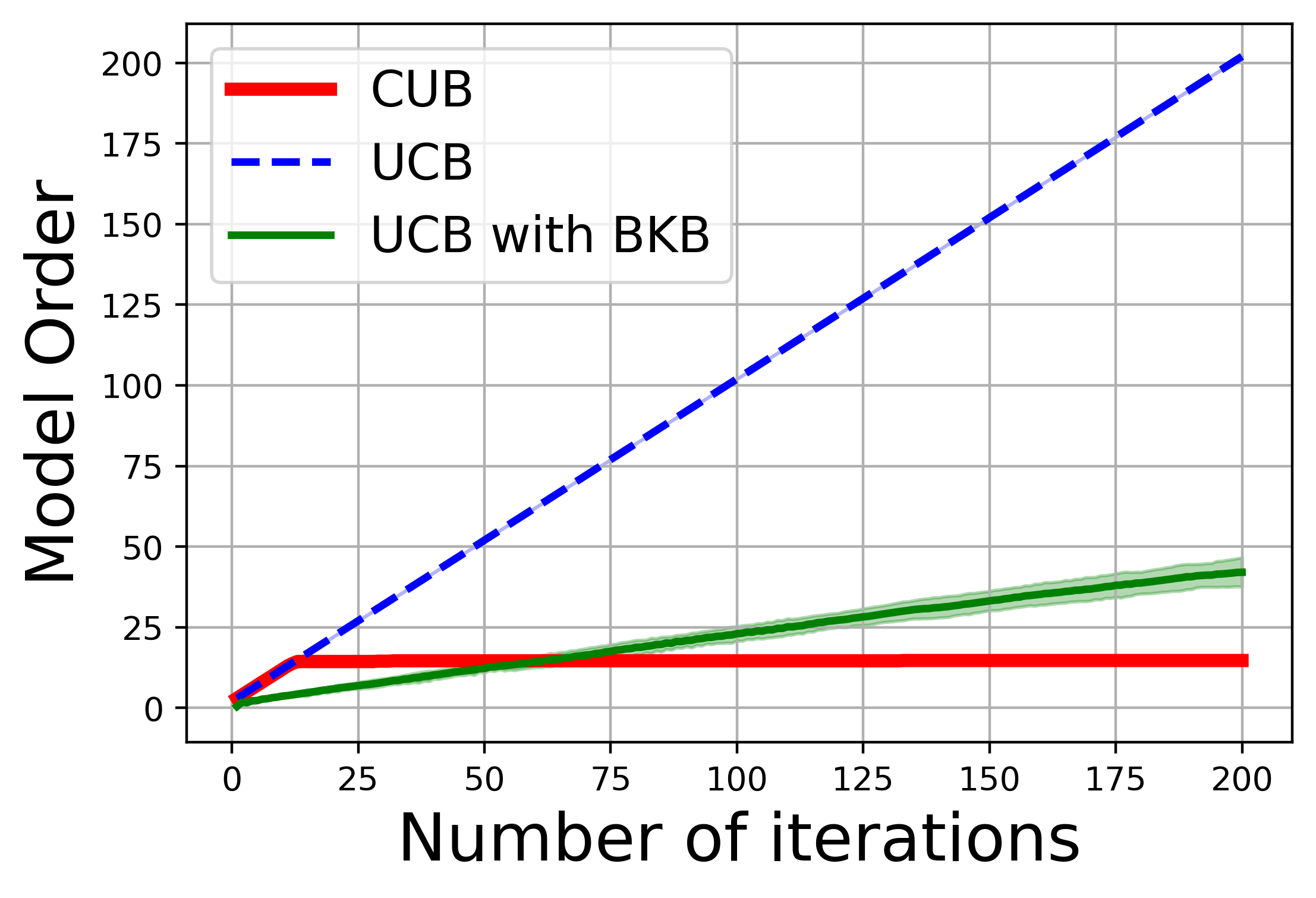}\label{subfig:xamplUCBM}}\hspace{-1mm}
	\subfigure[EI]{\includegraphics[scale=0.35]
		{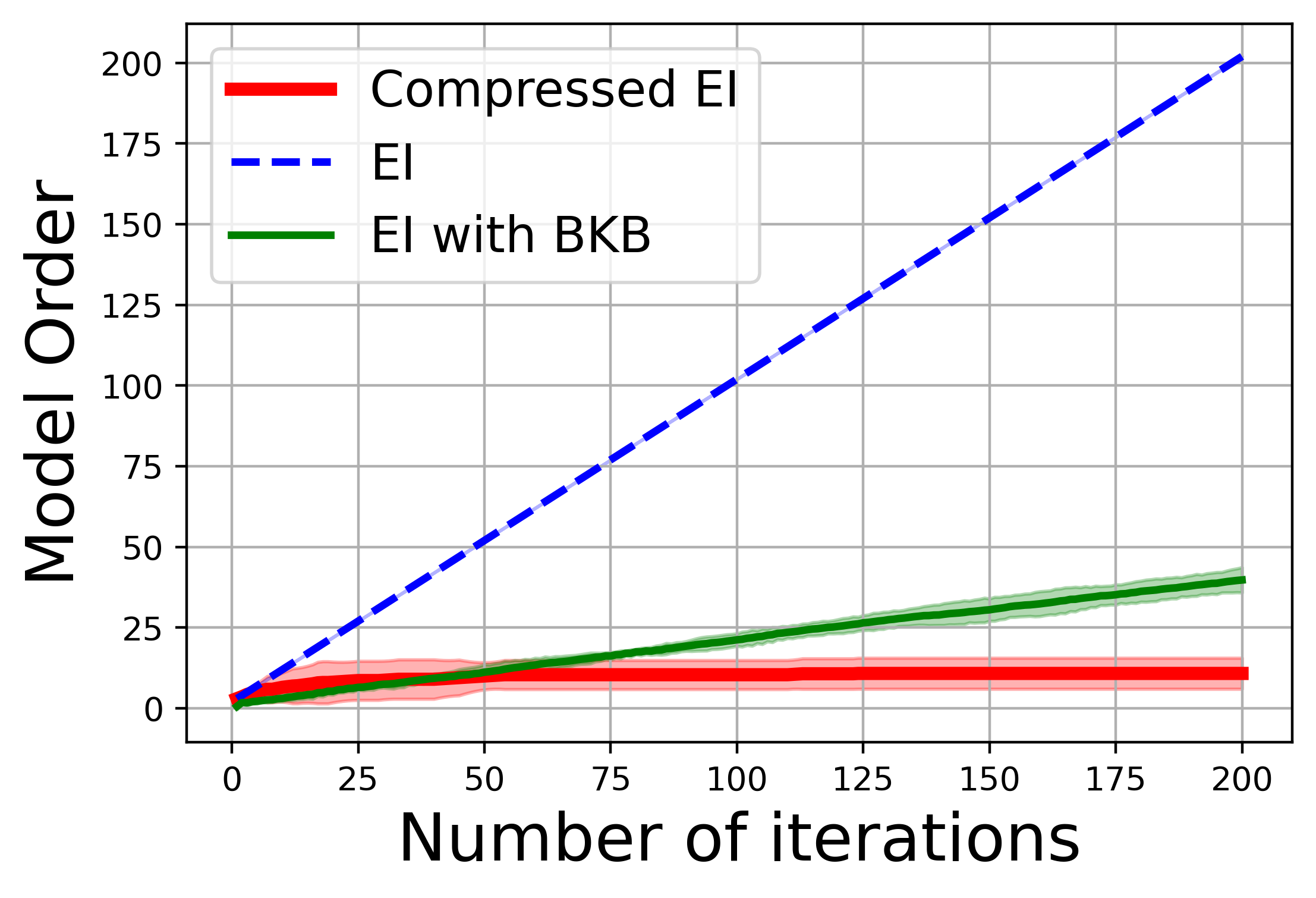}\label{subfig:xamplEIM}}\hspace{-1mm}
	\subfigure[MPI]{\includegraphics[scale=0.35]
		{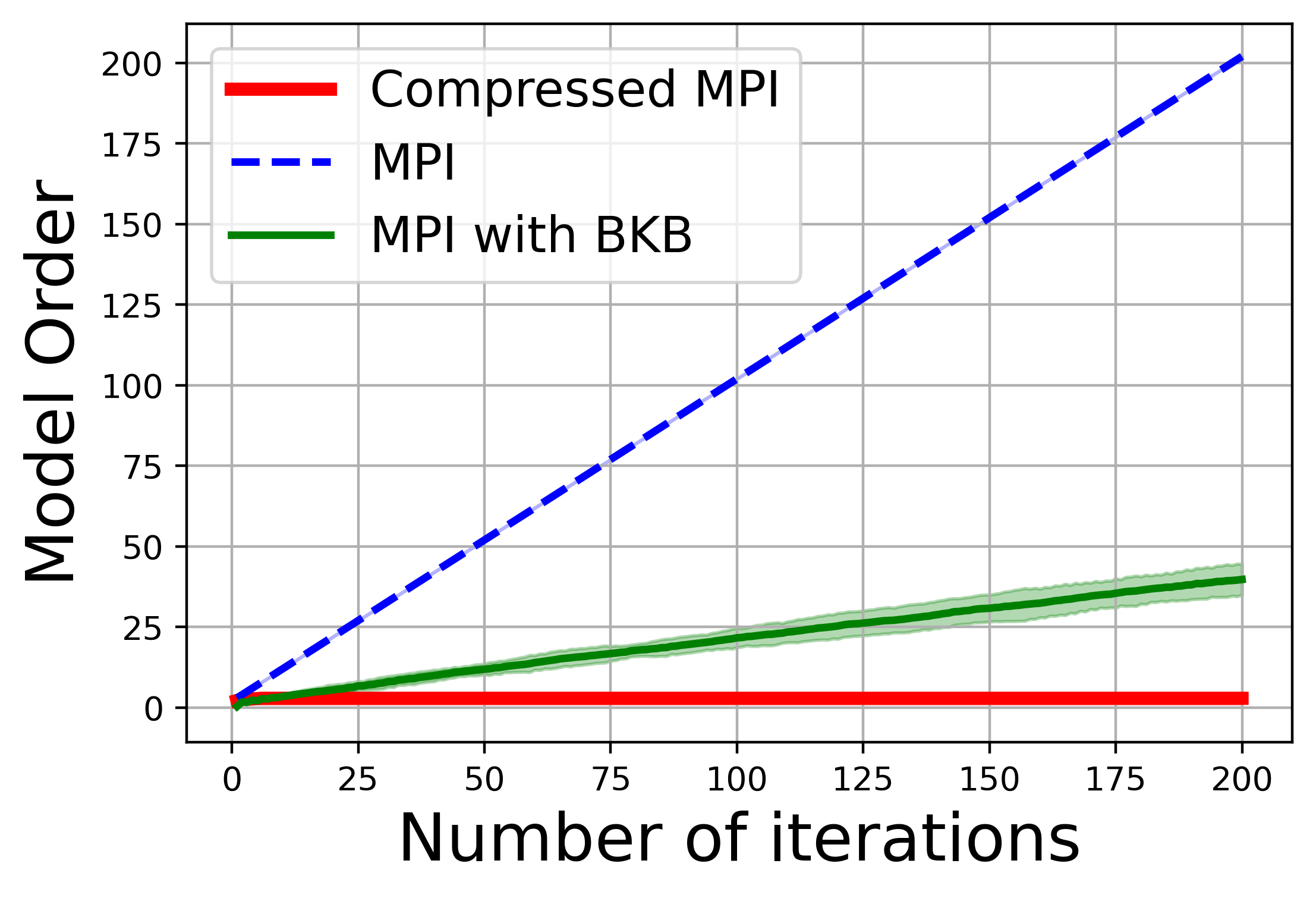}\label{subfig:xamplMPIM}}\vspace{-0mm}
	\caption{\blue{We display mean average regret vs. iteration (top row) and clock time (middle row) for the proposed algorithm compared against uncompressed and BKB variants on the {\bf example function} for various acquisition functions. Observe that our proposed compression scheme attains comparable regret to the dense GP. Moreover, the associated model complexity of the GP settles to an intrinsic constant discerned by the learning process (bottom row), as compared with alternatives which grow unbounded.}}\label{fig:xampl1}
\end{figure} 
%
%
%
\begin{table}[t]
	\centering
	\begin{tabular}{||c c c c ||} 
		\hline
		Acquisition & Uncompressed & Compressed & BKB \\ [0.5ex] 
		\hline\hline
		UCB & 6.756 & \textbf{5.335} & 9.56 \\ 
		EI & 7.594 & \textbf{4.133} & 10.578 \\
		MPI & 5.199 & \textbf{3.864} & 9.429\\ [1ex] 
		\hline
	\end{tabular}
	\caption{Clock Times (in seconds) with example function.}
	\label{table:xampl}
	\vspace{-0mm}
\end{table}
\subsection{Example function}\label{subsec:xampl}
Firstly, we evaluate our proposed method on an example function given by Equation \ref{eq:xampl}
\begin{equation}\label{eq:xampl}
	f(x) = \sin(x)  +  \cos(x)  +  0.1 x
\end{equation}

Random Gaussian noise is induced at every observation of $f$ to emulate the practical applications of Bayesian Optimization, where the black box functions are often corrupted by noise. 

The results of this experiment are shown in Figure \ref{fig:xampl1}, and the associated wall clock times are demonstrated in Table \ref{table:xampl}. Observe that the compression rule \eqref{eq:compression_rule} yields regret that is typically comparable to the dense GP, with orders of magnitude reduction in model complexity. This complexity reduction, in turn, permits a state of the art trade-off in regret versus wall clock time for certain acquisition functions, i.e., the UCB and EI, but not MPI. Interestingly, the model complexity of Algorithm \ref{alg:cub} settles to a constant discerned by the covering number (metric entropy) of the action space, validating the conceptual result of Theorem \ref{model_order}.

%
%
\subsection{Rosenbrock Function}\label{subsec:rosen}

For the second experiment, we compare the compressed variants with their baseline algorithm on a two-dimensional non-convex function popularly known as the Rosenbrock Function, given by:
\begin{equation}\label{eq:rosen}
	f(x, y) = (a-x)^2  +  b(y - x^2)^2
\end{equation}
The Rosenbrock function is a common benchmark non-convex function used to validate the performance of global optimization methods. Here we set its parameters as $a=1$ and $b=10$ for simplicity throughout. Again, we run various (dense and reduced-order) Gaussian Process bandit algorithms with different acquisition functions. 

The results of this experiment are displayed in Figure \ref{fig:rosen} with associated wall clock times collected in Table \ref{table:rosen}. Again, we observe that compression with respect to conditional entropy yields a minimal reduction in performance in terms of regret while translating to a significant reduction of complexity. Specifically, rather than growing linearly with the number of past actions, as is standard in nonparametric statistics, the model order settles down to an intrinsic constant determined by the metric entropy of the action space. This means that we obtain a state of the art trade-off in model complexity versus regret, as compared with the dense GP or probabilistic dropping inversely proportional to the variance, as in \citep{calandriello2019gaussian}.
\begin{figure}[t]\hspace{-1mm}
	\subfigure[UCB]{\includegraphics[scale=0.35]
		{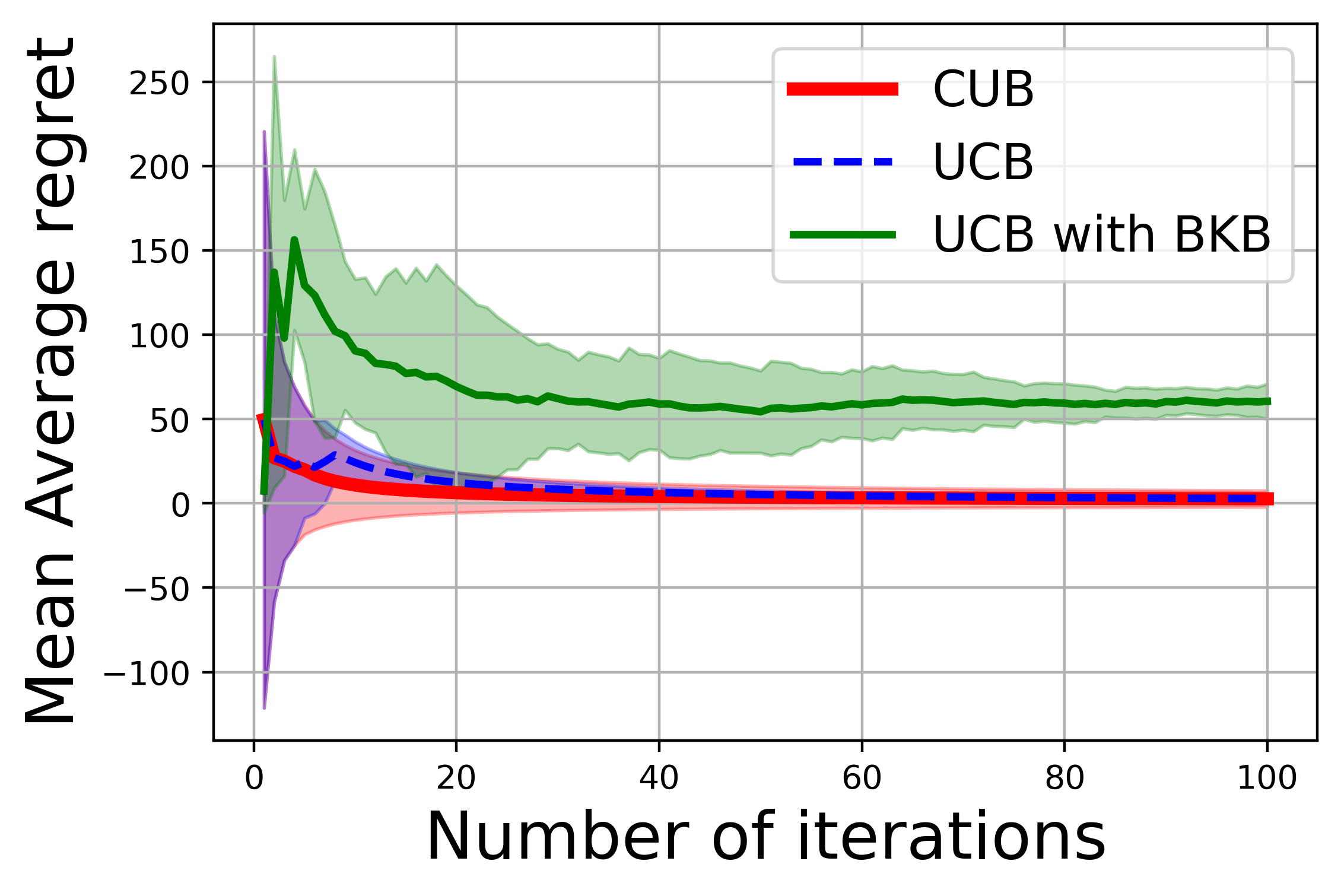}\label{subfig:rosenUCBmar}}\hspace{-1mm}
	\subfigure[EI]{\includegraphics[scale=0.35]
		{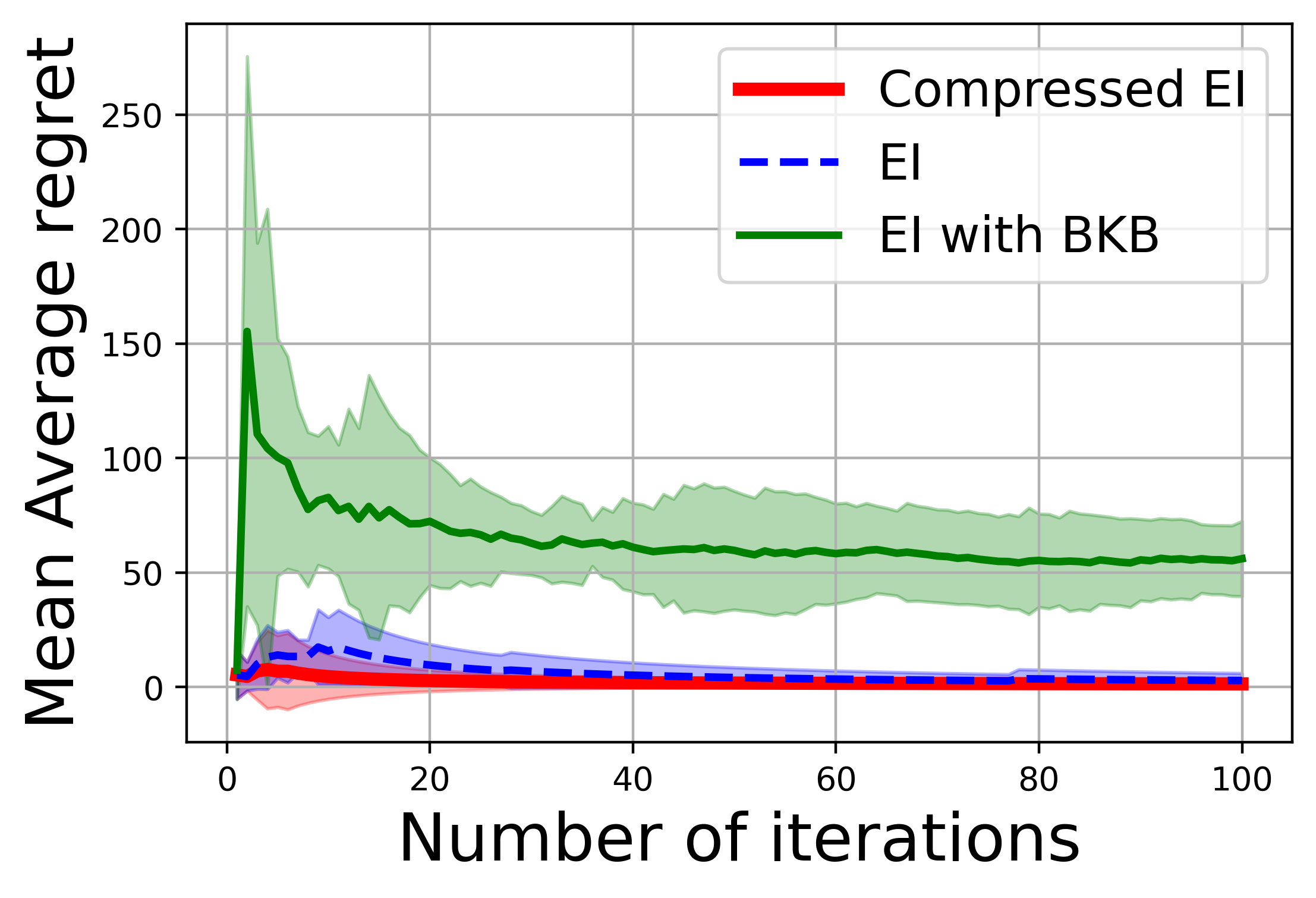}\label{subfig:rosenEImar}}\hspace{-1mm}
	\subfigure[MPI]{\includegraphics[scale=0.35]
		{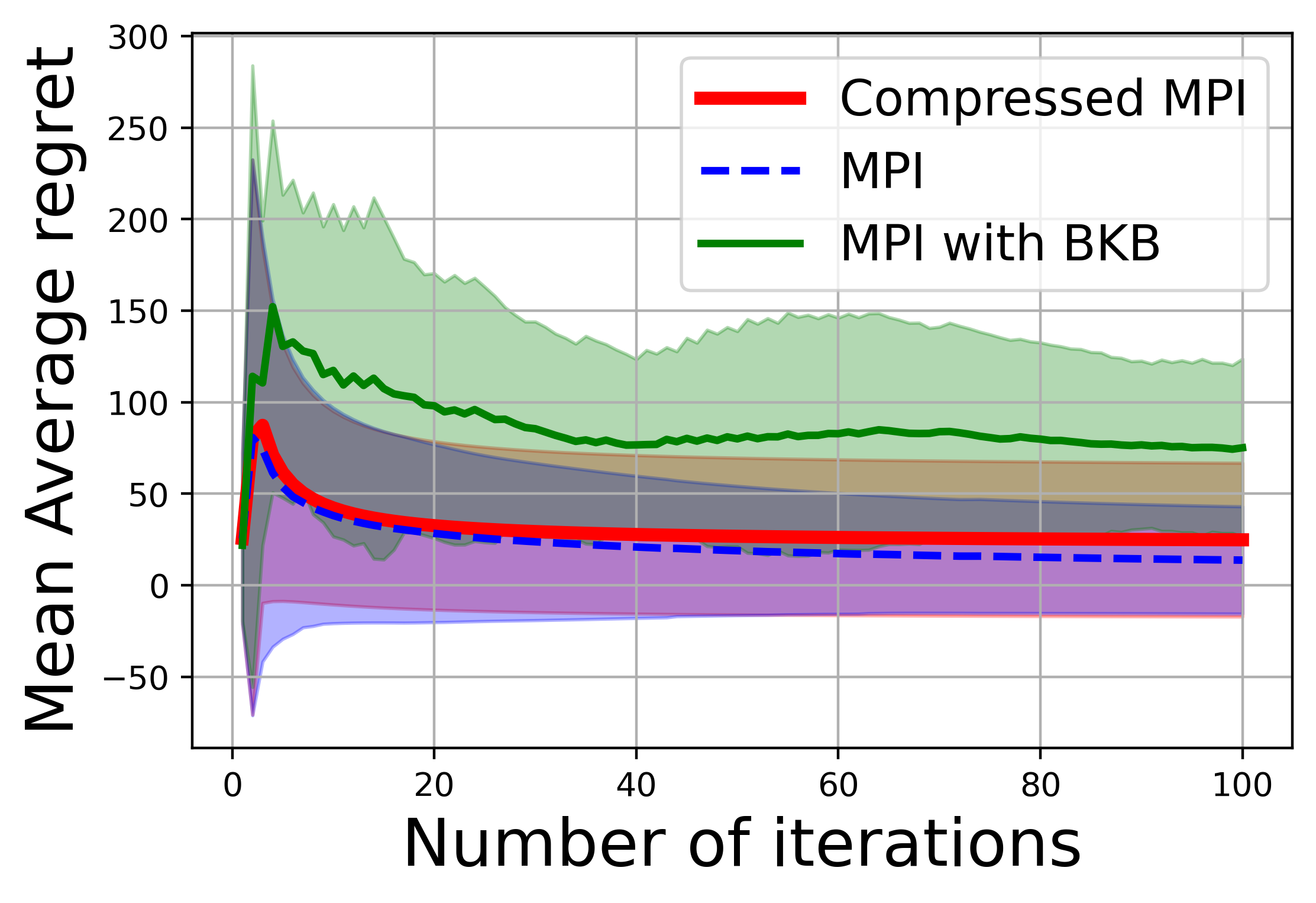}\label{subfig:rosenMPImar}} \vspace{-1mm}\\
	\hspace{-1mm}       	\subfigure[UCB]{\includegraphics[scale=0.35]
		{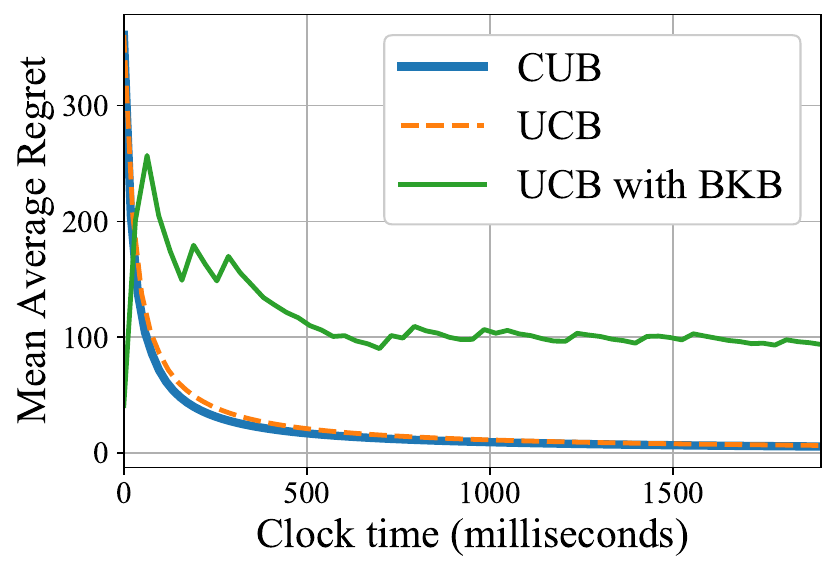}\label{subfig:rosenUCBCT}}\hspace{-1mm}
	\subfigure[EI]{\includegraphics[scale=0.35]
		{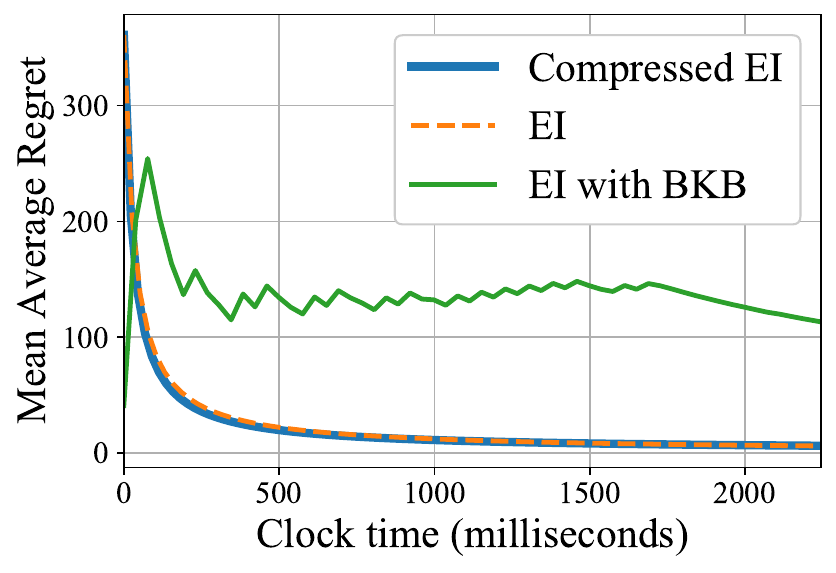}\label{subfig:rosenEICT}}\hspace{-1mm}
	\subfigure[MPI]{\includegraphics[scale=0.35]
		{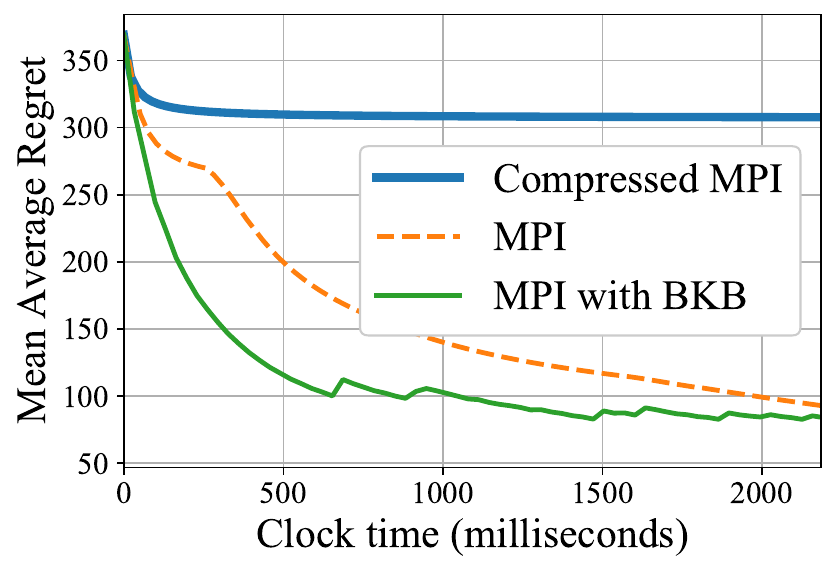}\label{subfig:rosenMPICT}}\vspace{-1mm}\\
	\subfigure[UCB]{\includegraphics[scale=0.35]
		{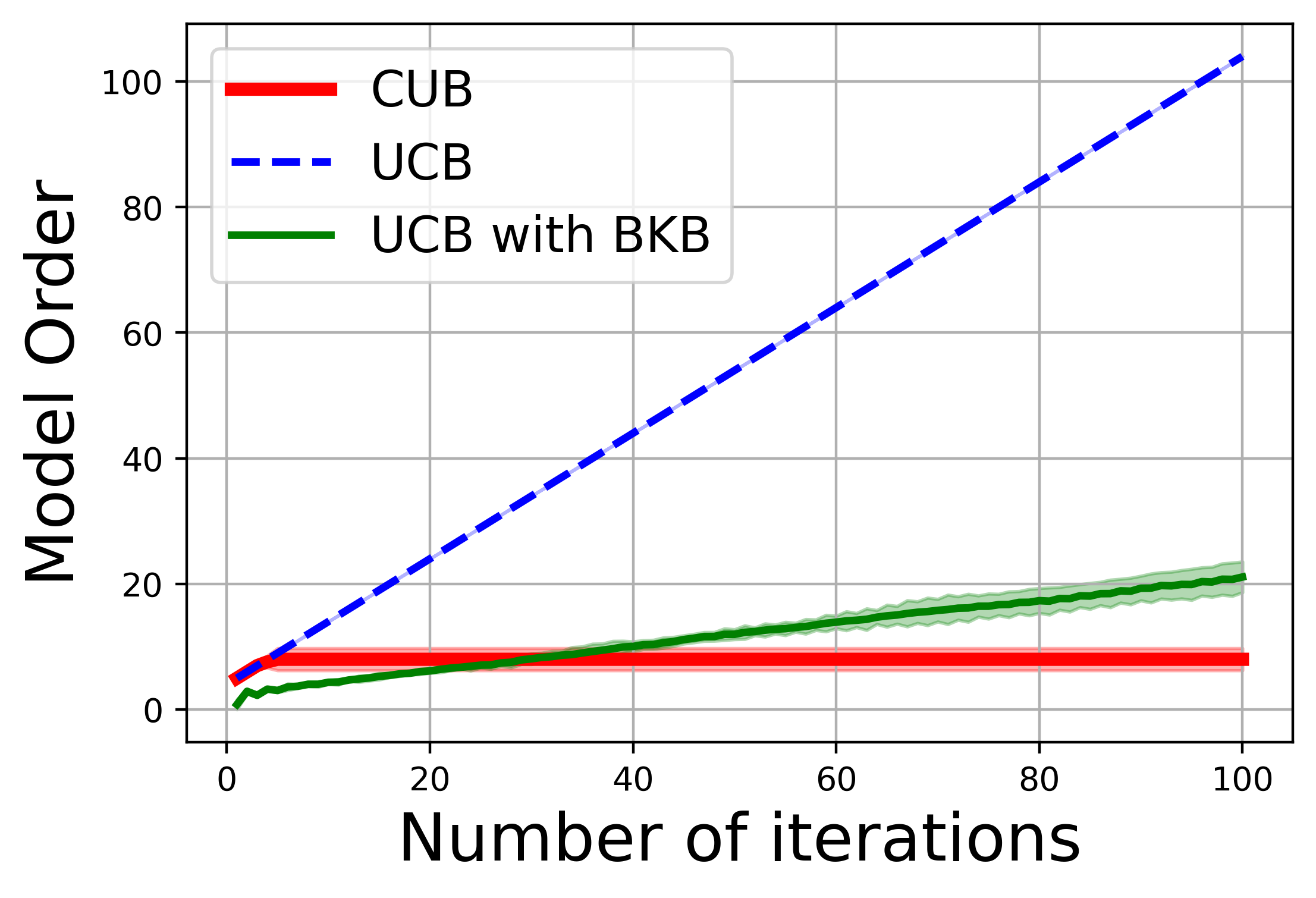}\label{subfig:rosenUCBM}}\hspace{-1mm}
	\subfigure[EI]{\includegraphics[scale=0.35]
	{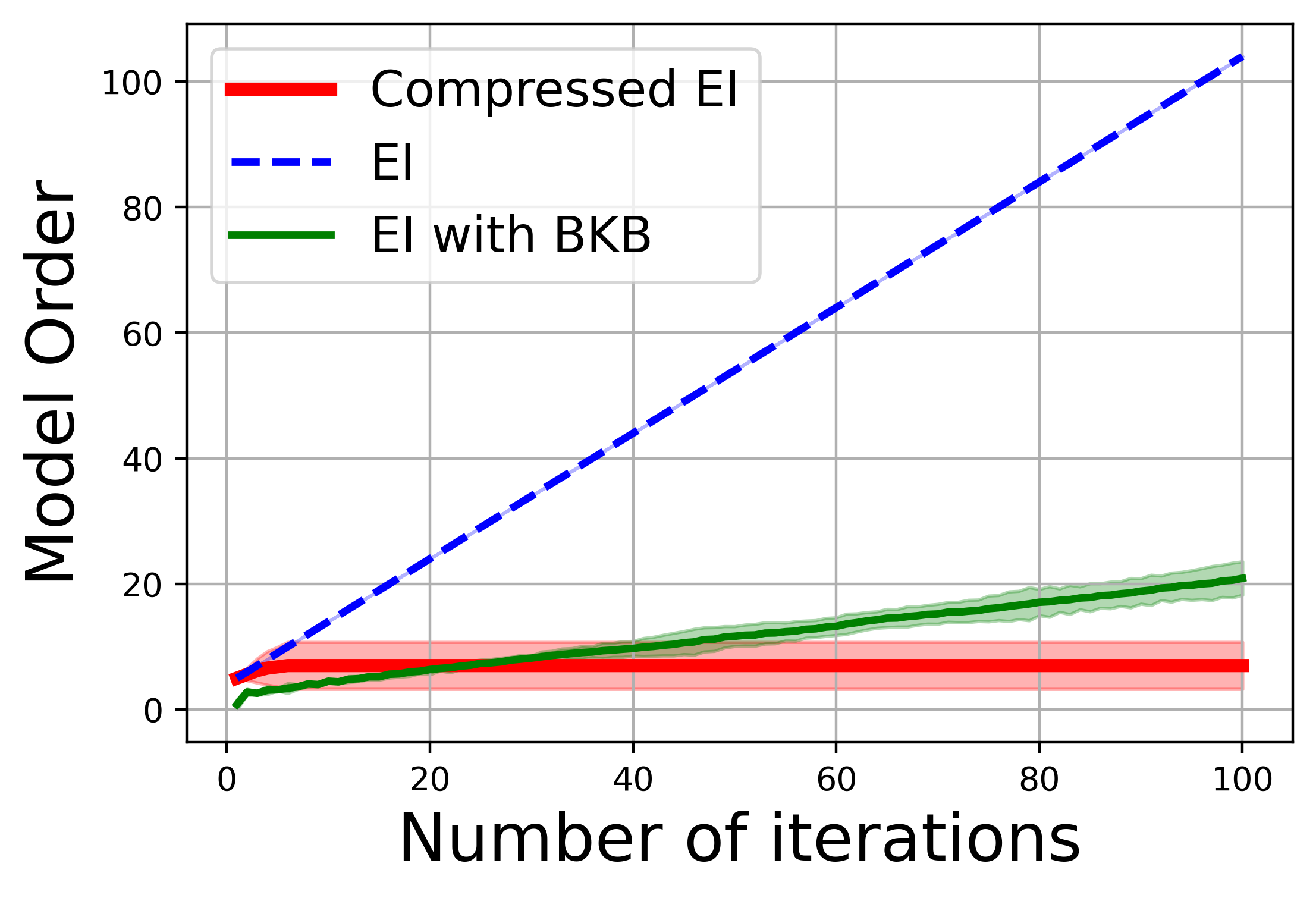}\label{subfig:rosenEIM}}\hspace{-1mm}
	\subfigure[MPI]{\includegraphics[scale=0.35]
	{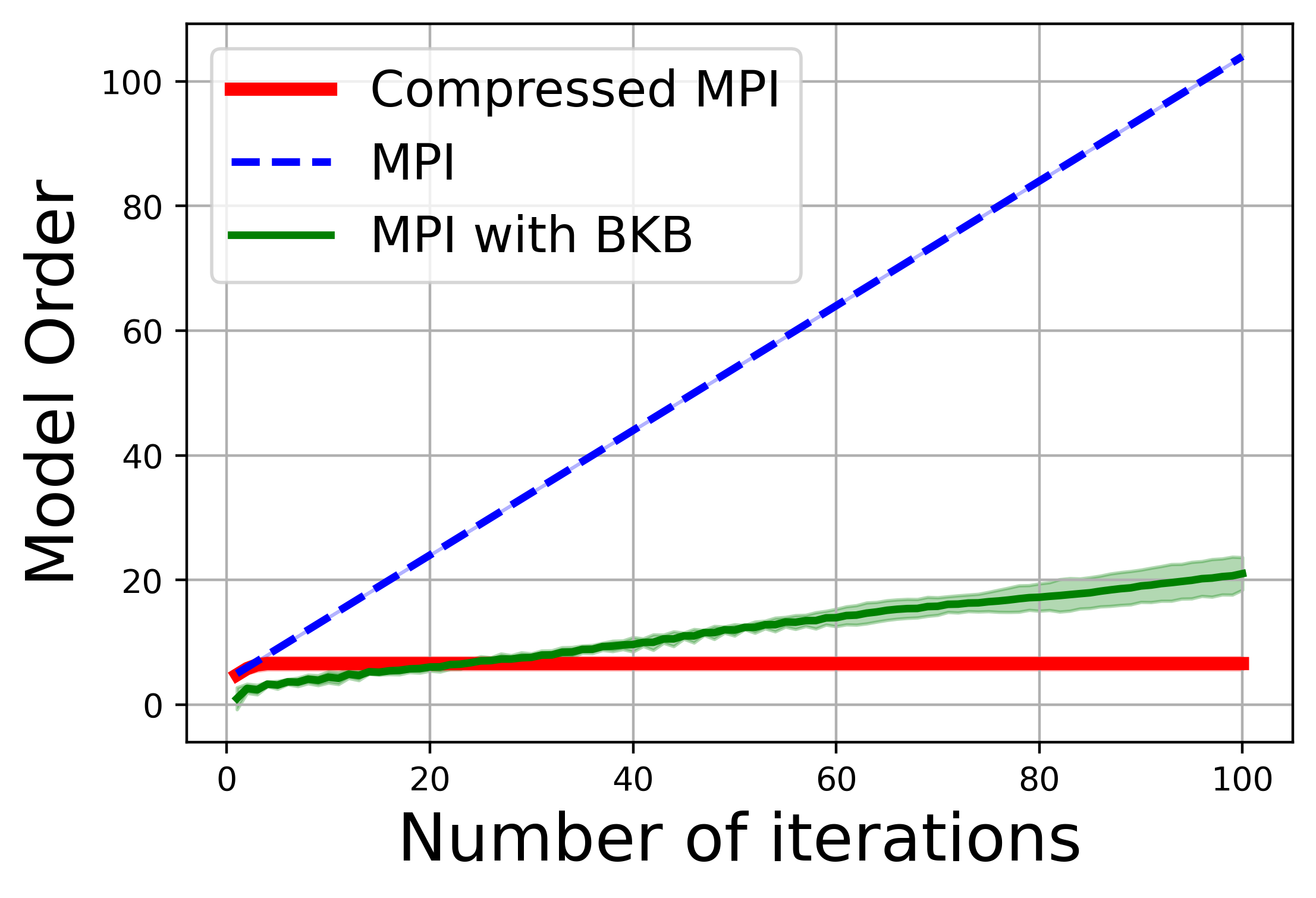}\label{subfig:rosenMPIM}}\vspace{-1mm}
	\caption{\blue{We display mean average regret vs iteration (top row) and clock time (middle row) for the proposed algorithm with uncompressed and BKB variants on the {\bf Rosenbrock function} for various acquisition functions. The compression based on conditional entropy yields  regret to comparable to the dense GP, with an associated model order that settles to a constant extracted by the optimization process (bottom row), as compared with alternatives which grow unbounded (for dense GP and BKB).}}\label{fig:rosen}\vspace{-0cm}
\end{figure} 
%

\begin{table}[h]
	\centering
	\begin{tabular}{||c c c c ||} 
		\hline
		Acquisition & Uncompressed & Compressed & BKB \\ [0.5ex] 
		\hline\hline
		UCB & 2.412 & \textbf{1.905} & 3.143\\ 
		EI & 2.604 & \textbf{2.246}  & 3.801\\
		MPI & 2.533 & \textbf{2.186} & 3.237\\ [1ex] 
		\hline
	\end{tabular}
	\caption{Clock Times (in seconds) on the Rosenbrock}
	\label{table:rosen}
\end{table}

\begin{figure}[h]\hspace{-1mm}
	\subfigure[UCB]{\includegraphics[scale=0.35]
		{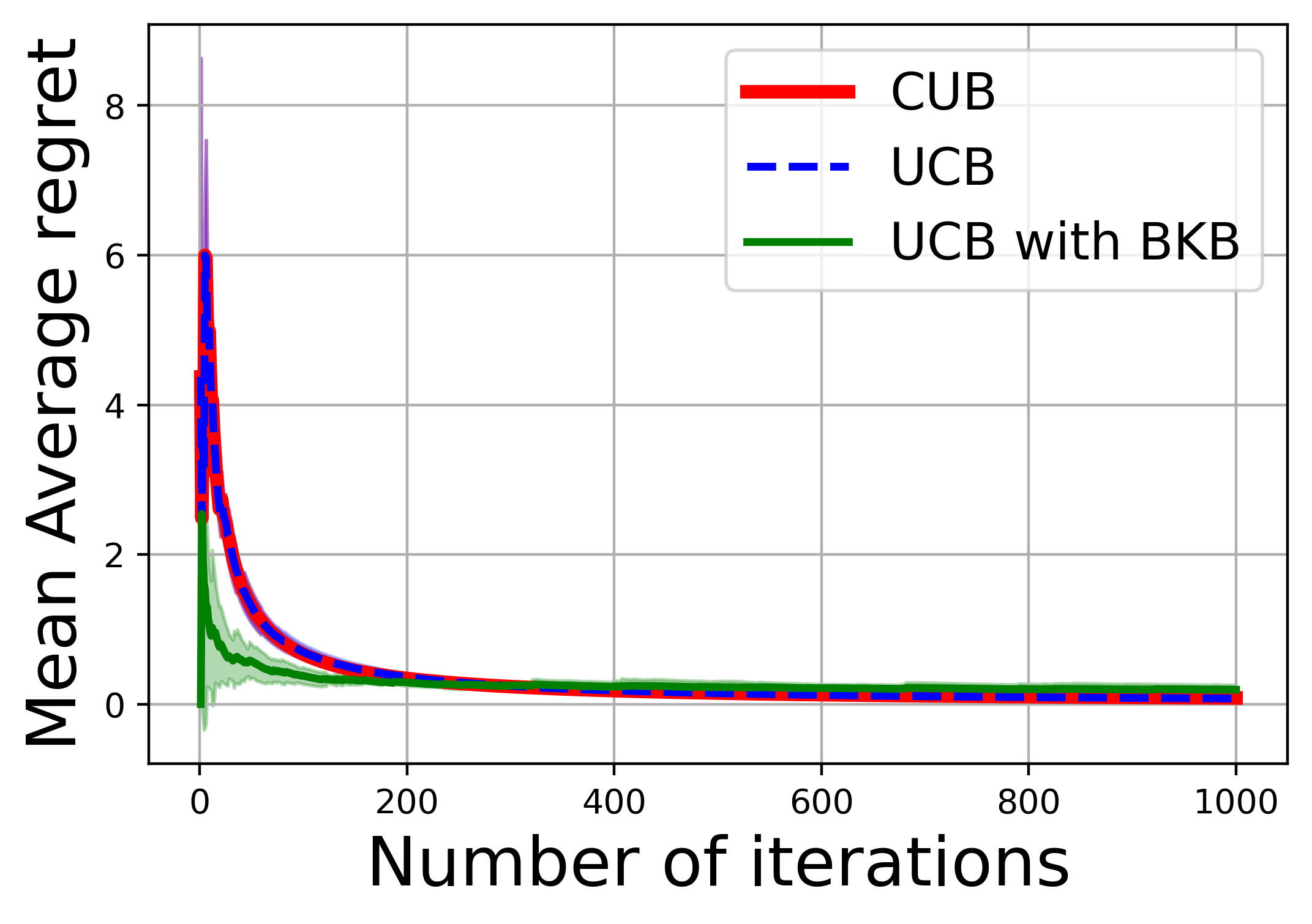}\label{subfig:mnistUCBmar}}\hspace{-1mm}
	\subfigure[EI]{\includegraphics[scale=0.35]
		{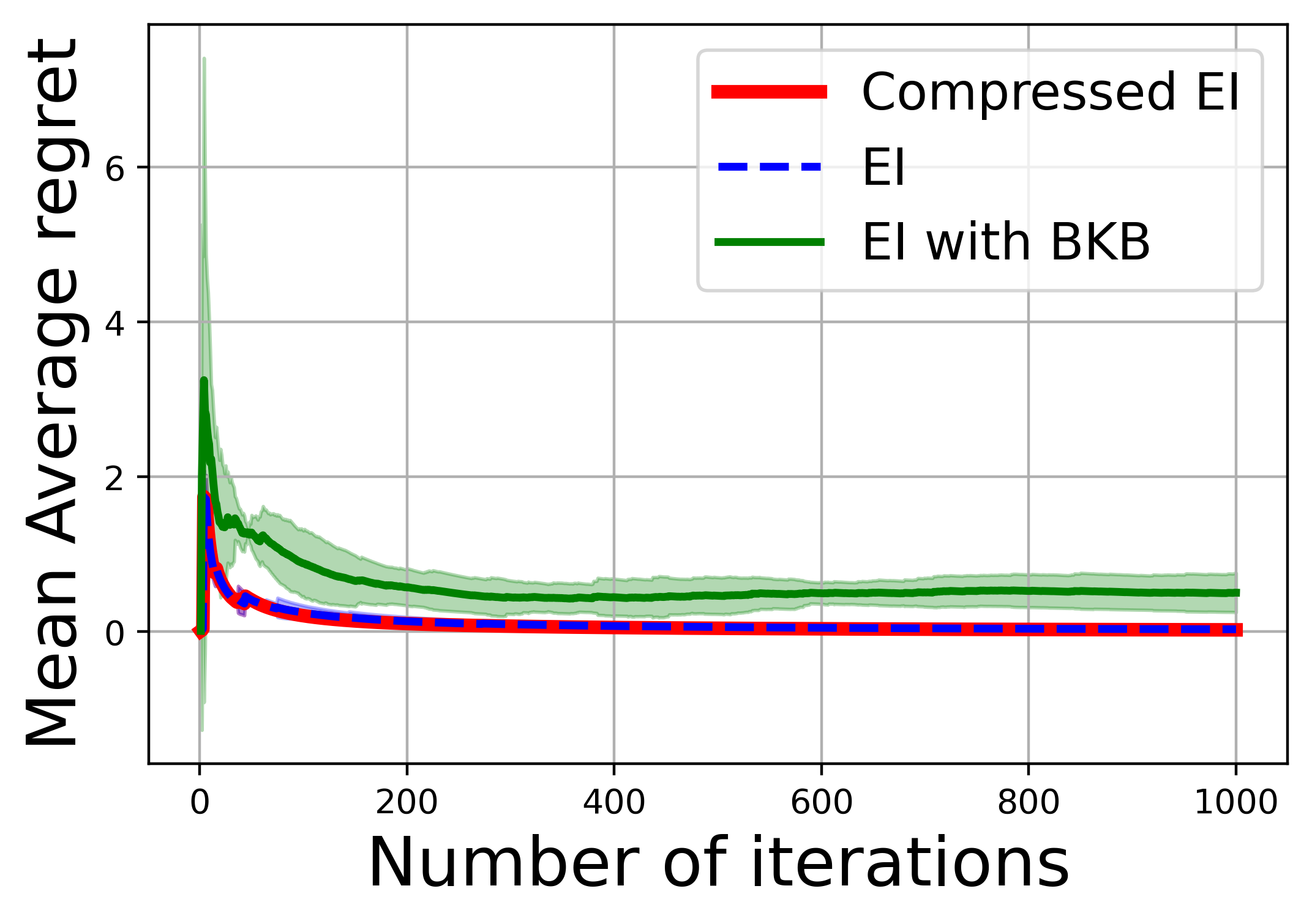}\label{subfig:mnistEImar}}\hspace{-1mm}
	\subfigure[MPI]{\includegraphics[scale=0.35]
	{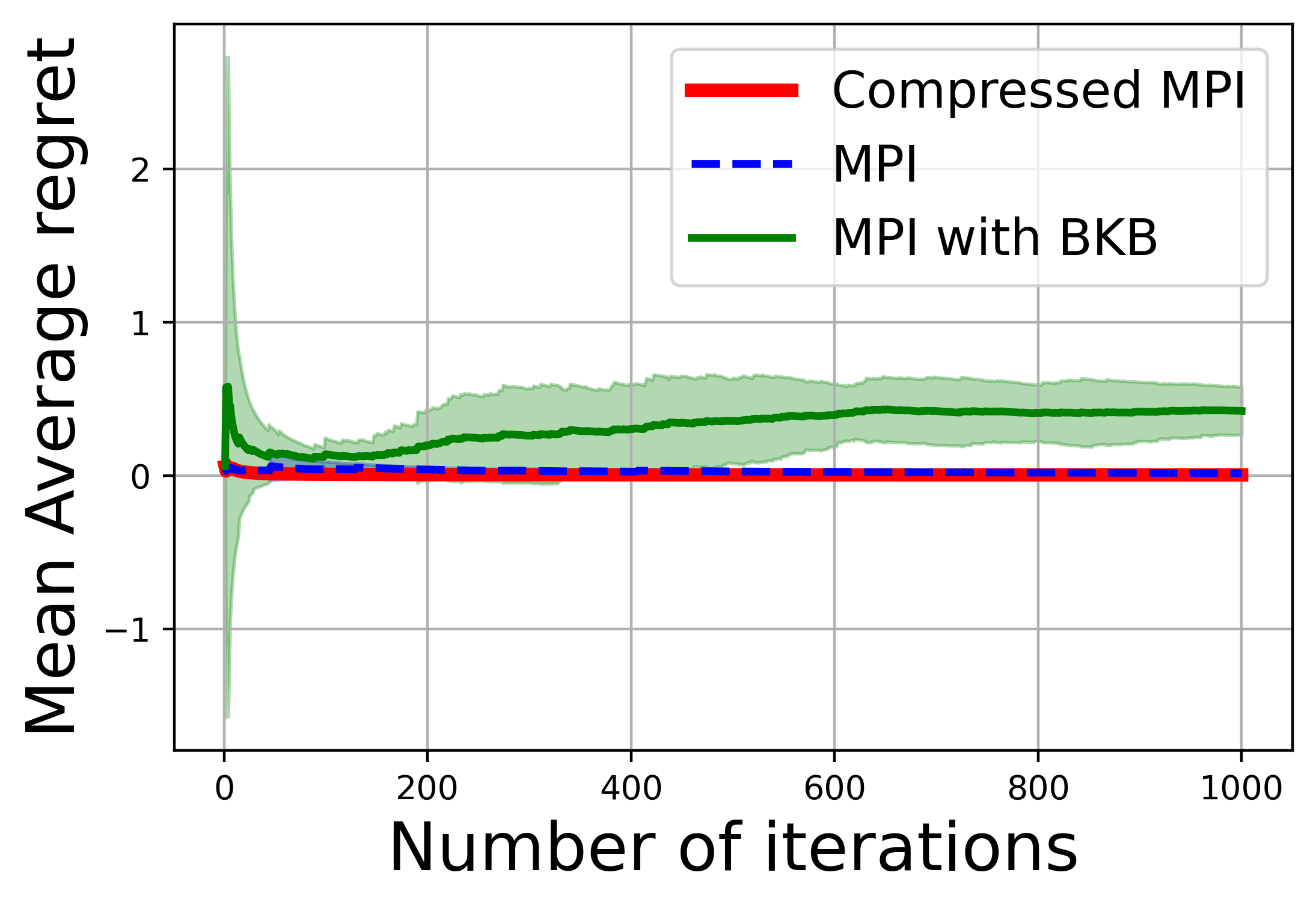}\label{subfig:mnistMPImar}} \vspace{-0mm}\\
	\hspace{-1mm}       	\subfigure[UCB]{\includegraphics[scale=0.35]
		{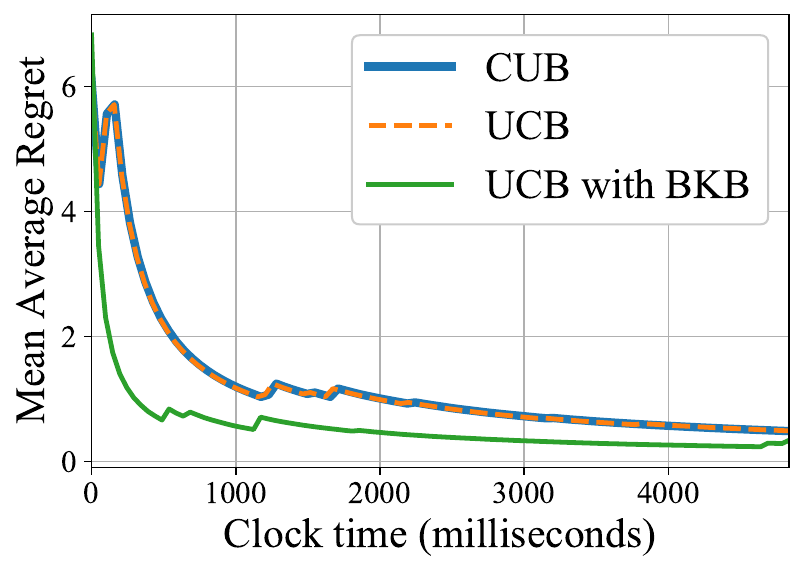}\label{subfig:mnistUCBCT}}\hspace{-1mm}
	\subfigure[EI]{\includegraphics[scale=0.35]
		{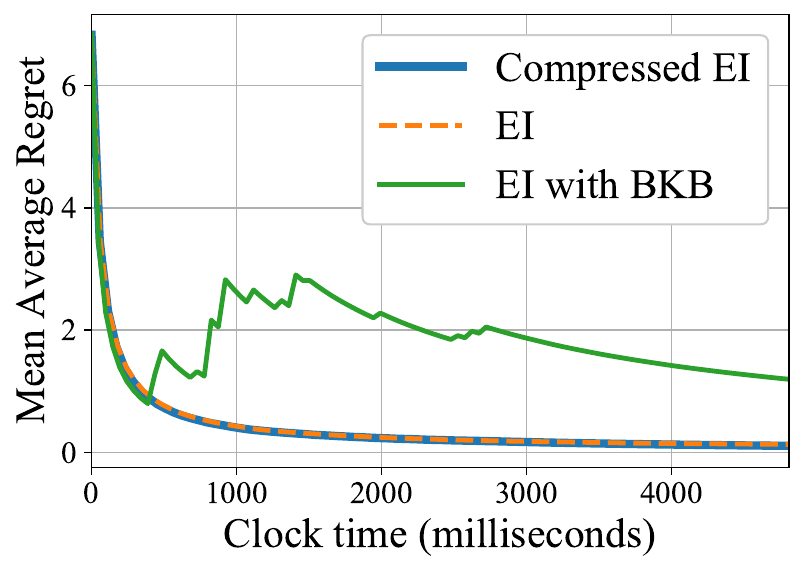}\label{subfig:mnistEICT}}\hspace{-1mm}
	\subfigure[MPI]{\includegraphics[scale=0.35]
		{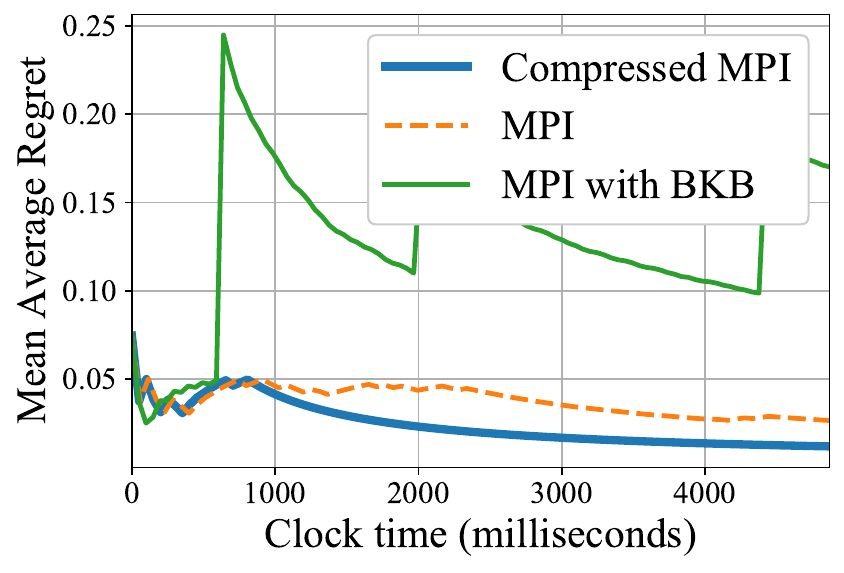}\label{subfig:mnistMPICT}}\vspace{-0mm}\\
	\subfigure[UCB]{\includegraphics[scale=0.35]
		{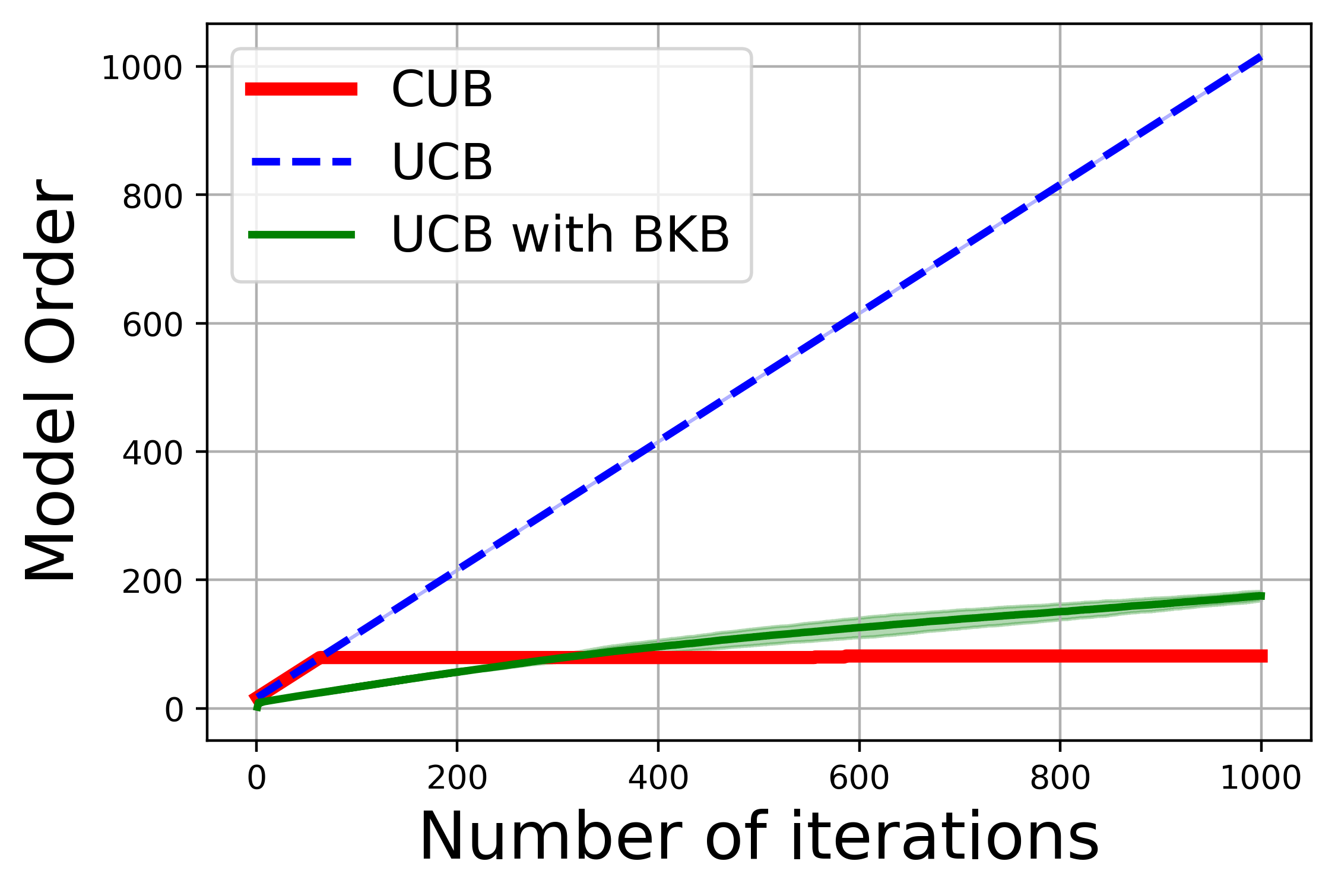}\label{subfig:mnistUCBM}}\hspace{-1mm}
	\subfigure[EI]{\includegraphics[scale=0.35]
		{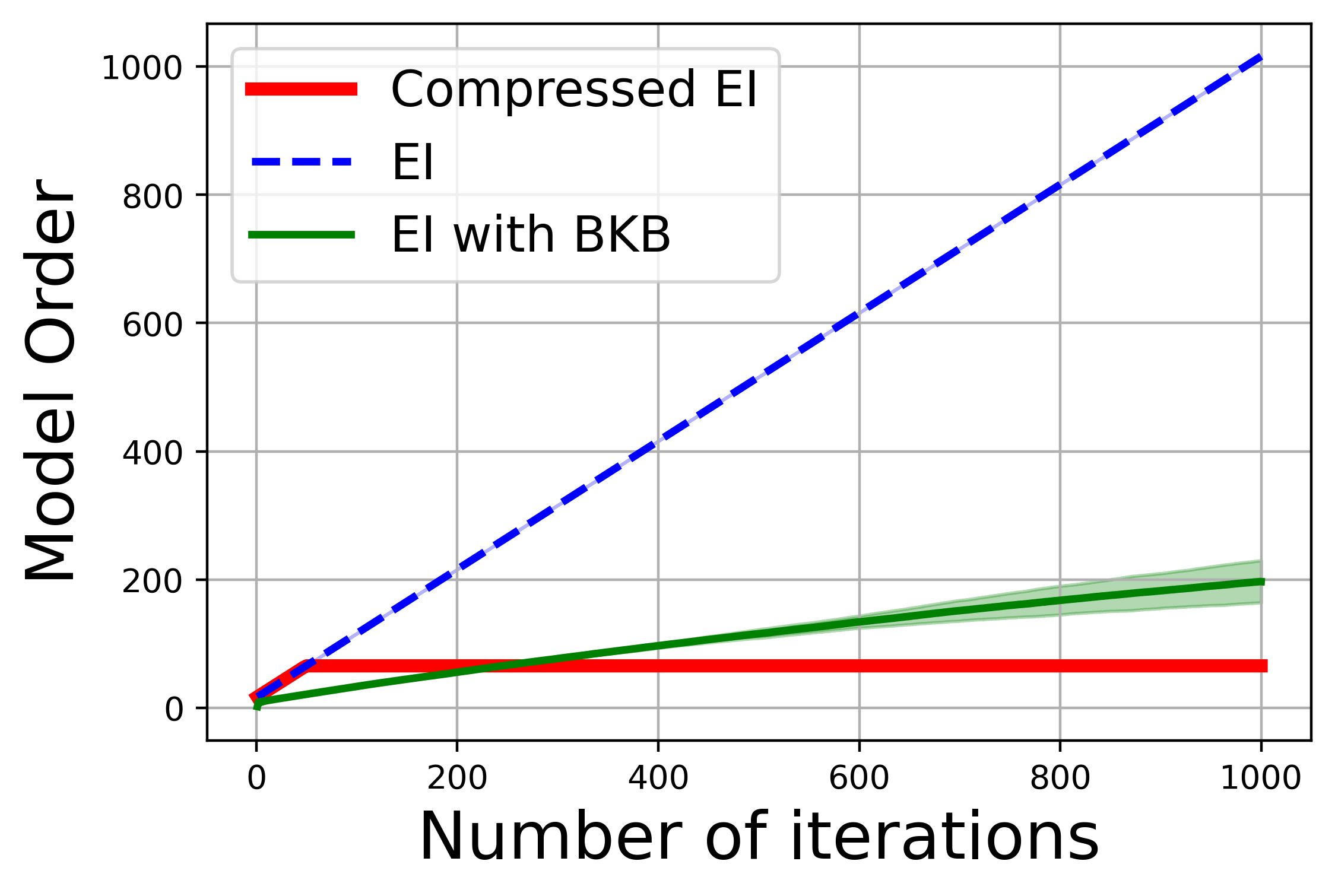}\label{subfig:mnistEIM}}\hspace{-1mm}
	\subfigure[MPI]{\includegraphics[scale=0.35]
		{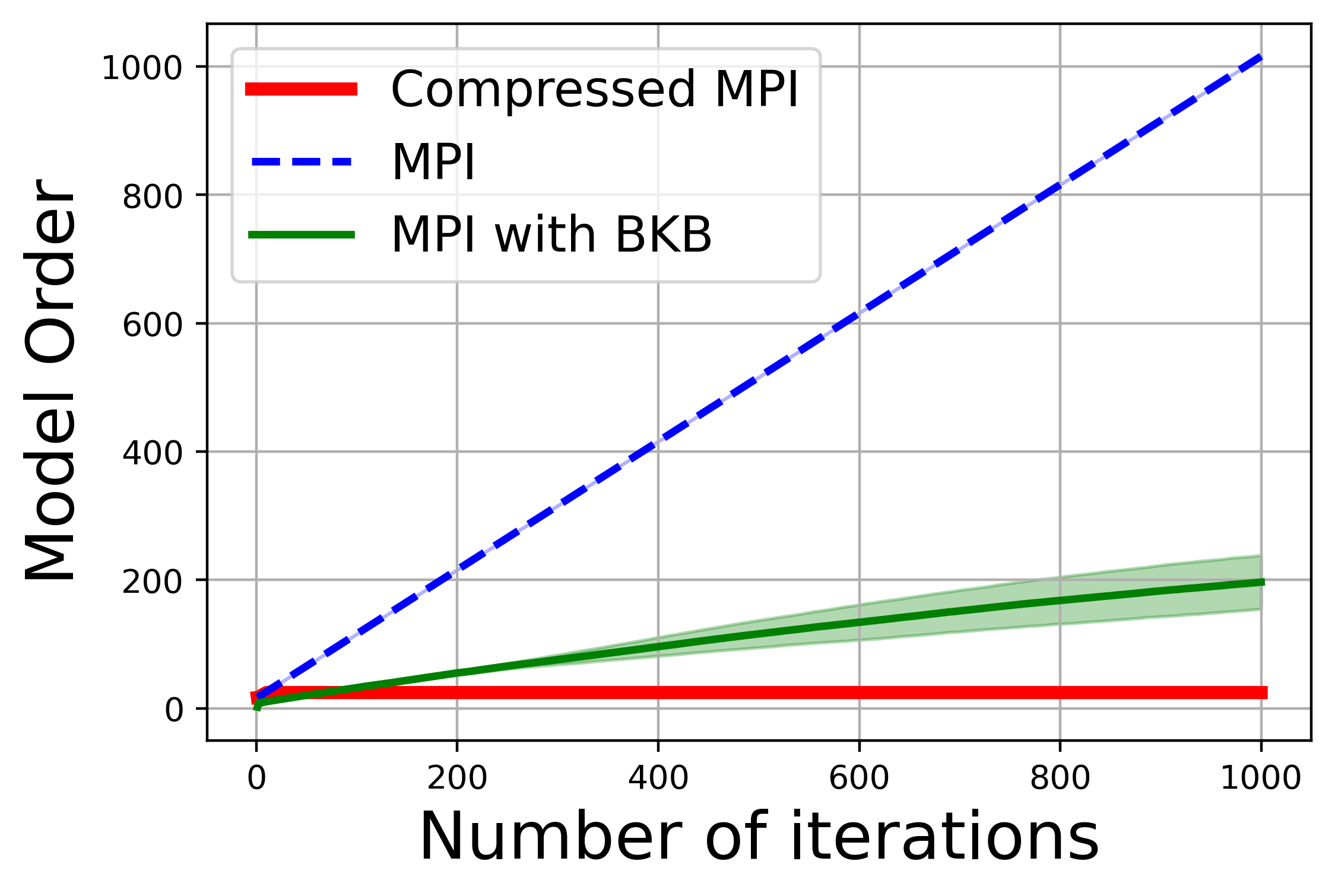}\label{subfig:mnistMPIM}}\vspace{-0mm}
	\caption{\blue{For problem of tuning the regularization and step-size {\bf hyper-parameters} of a logistic regressor on MNIST, we illuminate the mean average regret vs iteration (top row) and clock time (middle row) for the proposed algorithm with uncompressed and BKB variants for various acquisition functions. The compression based on conditional entropy yields regret to comparable to the dense GP, while obtaining a model complexity that is constant and determined by the optimization problem (bottom row), as compared with alternatives that grow unbounded. Overall, then, one can run any GP bandit algorithm with the compression rule \eqref{eq:compression_rule} in perpetuity  on the back-end of any training scheme for supervised learning in order to automate the selection of hyper-parameters in perpetuity without worrying about eventual slowdown.}
	}\label{fig:mnist}\vspace{-0cm}
\end{figure} 
\subsection{Hyper-paramter Tuning in Logistic Regression}\label{subsec:mnist}
In this subsection, we propose using bandit algorithms to automate the hyper-parameter tuning of machine learning algorithms. More specifically, we propose using Algorithm \ref{alg:cub} and variants with different acquisition functions to tune the following hyper-parameters of a supervised learning scheme, whose concatenation forms the action space: the learning rate, batch size, dropout of the inputs, and the $\ell_2$ regularization constant. The specific supervised learning problem we focus on is the training of a multi-class logistic regressor over the MNIST training set \citep{lecun} for classifying handwritten digits. The instantaneous reward here is the statistical accuracy on a hold-out validation set. 

Considering the high-dimensional input domain and the number of training examples, the GP dictionary may explode to a large size. In large-scale settings, the input space could be much larger with many more hyper-parameters to tune, in which case GPs may be computationally intractable.  The statistical compression proposed here ameliorates this issue by keeping the size of the training dictionary in check, which makes it feasible for hyper-parameter tuning as the number of training examples becomes large. 

\blue{The results of this implementation are given in Figure \ref{fig:mnist} with associated compute times in Table \ref{table:mnist}. Observe that the trend identified in the previous two examples translates into practice here: the compression technique \eqref{eq:compression_rule} yields algorithms whose regret is comparable to the dense GP, with a significant reduction in model complexity that eventually settles to a constant. This constant is a fundamental measure of the complexity of the action space required for finding a no-regret policy. Overall, then, one can run Algorithm \ref{alg:cub} on the back-end of any training scheme for supervised learning in order to automate the selection of hyper-parameters in perpetuity without worrying about the eventual slowdown. Table \ref{table:mnist} compares the clock time (in hours) of the proposed compressed method to uncompressed and BKB method. We note that the proposed compressed technique requires only $22\%$ of time required by uncompressed method to achieve similar performance. }

\begin{table}[t]
	\centering
		\color{black}\begin{tabular}{||c c c c ||} 
	\hline
	Acquisition & Uncompressed & Compressed & BKB \\ [0.5ex] 
	\hline\hline
	UCB & 0.994 $\pm$ 0.055 & \textbf{0.282 $\pm$ 0.025} &  1.817 $\pm$ 0.077 \\ 
	EI & 0.992 $\pm$ 0.028 & \textbf{0.266 $\pm$ 0.010} & 2.072 $\pm$ 0.543 \\
	MPI & 0.993 $\pm$ 0.013 & \textbf{0.227 $\pm$ 0.004} & 1.840 $\pm$ 0.299 \\  [1ex] 
	\hline
\end{tabular}
\caption{Clock Times (in hours) with Hyperparameter Tuning.}
	\label{table:mnist}
\end{table}

\section{Conclusions}

We considered bandit problems whose action spaces are discrete but have large cardinality or are continuous. The canonical performance metric, regret, quantifies how well bandit action selection is against the best comparator in hindsight. By connecting regret to maximum information-gain-based exploration, which may be quantified by variance, one may find no-regret algorithms through variance maximization. Doing so yields actions which over-prioritize exploration. To balance between exploration and exploitation, that is, moving towards the optimum in the finite time, we focused on upper-confidence bound based action selection. Following a number of previous works for bandits with large action spaces, we parameterized the action distribution as a Gaussian Process in order to have a closed-form expression for the a posteriori variance.

Unfortunately, Gaussian Processes exhibit complexity challenges when operating ad infinitum: the complexity of computing posterior parameters grows cubically with the time index. While numerous previous memory-reduction methods exist for GPs, designing compression for bandit optimization is relatively unexplored. Within this gap, we proposed a compression rule for the GP posterior explicitly derived by information-theoretic regret bounds, where the conditional entropy encapsulates the per-step progress of the bandit algorithm. This compression-only includes past actions whose conditional entropy exceeds an $\epsilon$-threshold to enter into the posterior.

As a result, we derived explicit trade-offs between model complexity and information-theoretic regret. In experiments, we observed a favorable trade-off between regret, model complexity, and iteration index/clock time for a couple of toy non-convex optimization problems, as well as the actual problem of how to tune hyper-parameters of a supervised machine learning model.

Future directions include extensions to non-stationary bandit problems, generalizations to history-dependent action selection strategies such as step-wise uncertainty reduction methods \citep{villemonteix2009informational}, and information-theoretic compression of deep neural networks based on bandit algorithms.

\bibliographystyle{icml2020}
\bibliography{bibliography}
%
\clearpage
\section*{\centering Appendix 
}  

\section{Preliminaries}\label{sec:preliminaries}
Before proceeding with the proofs in detail, we define some notation to clarify the exposition. In the proof, we utilize the compact notation  
$\mu_t:=\mu_{\bbD_t}$ and
$\sigma_t:=\sigma_{\bbD_t}$ interchangeably. 
\section{Proof of Theorem \ref{model_order}}\label{finite_model_order_proof}
%
\blue{Before proving the Theorem \ref{model_order}, we first describe a lemma which relate the compression rule in the proposed Algorithm \ref{alg:cub} to a Hilbert subspace distance.
	%
	\begin{lemma}\label{lemma_subspace_dist}
		Let us consider a reproducing kernel Hilbert space $\mathcal{H}$ associated with the feature space $\mathcal{X}$ defined for kernel $\kappa(\bbx,\cdot)$. This implies that for each $\mathcal{x}\in\mathcal{X}$, we have $\kappa(\bbx,\cdot)\in\mathcal{H}$. Let us define the distance  in the Hilbert space for an arbitrary feature vector $\bbx_t$ at $t$, evaluated by kernel $\kappa(\bbx_t, \cdot)$ to $\ccalH_{\bbD_{t-1}}=\text{span}\{\kappa(\bbx_n, \cdot) \}_{n=1}^{M}$, the subspace of the Hilbert space spanned by a dictionary $\bbD_{t-1}$ of size $M:=|\bbD_{t-1}|$,  as
		\begin{equation}\label{eq:hilbert_subspace_dist}
			\text{dist}^2( \kappa(\bbx_t, \cdot) , \ccalH_{\bbD_{t-1}}) 
			=  \min_{\bbv\in \reals^{M}}\left[\| \kappa(\bbx_t, \cdot) - \bbv^T \boldsymbol{k}_{\bbD_{t-1}}(\cdot) \|_{\ccalH}^2+\frac{\lambda}{2}\|\bbv\|^2\right] \; ,
		\end{equation}
		where $\boldsymbol{k}_{\bbD_{t-1}}(\cdot):=[\kappa(\bbx_1,\cdot),\kappa(\bbx_2,\cdot), \cdots,\kappa(\bbx_{M_{t-1}},\cdot)]^T$ and $\lambda$ is a regularization parameter.
		This set distance simplifies to following closed form expressions given by 
		\begin{equation}\label{eq:hilbert_subspace_dist_ls}
			\text{dist}^2( \kappa(\bbx_t, \cdot) , \ccalH_{\bbD_{t-1}}) 
			=   \kappa(\bbx_t,\bbx_t) 
			- \boldsymbol{k}_{\bbD_{t-1}}(\bbx_t)^T\left[\bbK_{\bbD_{t-1}, \bbD_{t-1}}+\lambda \bbI\right]^{-1}
			\boldsymbol{k}_{\bbD_{t-1}}(\bbx_t) \; .
		\end{equation}
	\end{lemma}
	\begin{proof}
		We solve the following regularized version of the problem in \eqref{eq:hilbert_subspace_dist} to obtain a unique solution $\bbv^*$ as 
		\begin{equation}\label{eq:hilbert_subspace_dist22}
			\bbv^*:=	\arg \min_{\bbv\in \reals^{M}} \left[\| \kappa(\bbx_t, \cdot) - \bbv^T \boldsymbol{k}_{\bbD_{t-1}}(\cdot) \|_{\ccalH}^2+\frac{\lambda}{2}\|\bbv\|^2\right] \; ,
		\end{equation}
		%
		Now we differentiate the objective, equate it equal to zero, to obtain $\bbv^*=\left[\bbK_{\bbD_{t-1}, \bbD_{t-1}}+\lambda \bbI\right]^{-1}\boldsymbol{k}_{\bbD_{t-1}}(\bbx)$ into \eqref{eq:hilbert_subspace_dist22}. This simplifies \eqref{eq:hilbert_subspace_dist} to the following
		\begin{align}\label{eq:subspace_dist2}
			\text{dist}^2(\kappa(\bbx_t, \cdot) , \ccalH_{\bbD_{t-1}}) 
			= &\left[\| \kappa(\bbx_t, \cdot) - (\bbv^*)^T \boldsymbol{k}_{\bbD_{t-1}}(\cdot) \|_{\ccalH}^2+\frac{\lambda}{2}\|\bbv^*\|^2\right] \; .
			\nonumber
			\\
			=&\kappa(\bbx_t,\bbx_t) 
			- \boldsymbol{k}_{\bbD_{t-1}}(\bbx_t)^T\left[\bbK_{\bbD_{t-1},\bbD_{t-1}}+\lambda \bbI\right]^{-1}
			\boldsymbol{k}_{\bbD_{t-1}}(\bbx_t).
		\end{align}
		Hence proved. 
	\end{proof}
	From the GP update in \eqref{eq:posterior_parameters}, we it holds that
	\begin{align}
		\sigma_{\bbD_{t-1}}^{2}\!(\bbx_t)\!=\kappa(\bbx_t,\!\bbx_t)\!-\!{\boldsymbol{k}}_{\bbD_{t-1}}(\bbx_t)^T(\bbK_{\bbD_{t-1},\bbD_{t-1}}\!\!+\!\!\sigma^{2}\bbI)^{-1}{\boldsymbol{k}}_{\bbD_{t-1}}(\bbx_t)
	\end{align}
	which is exactly equal to $	\text{dist}^2( \kappa(\bbx_t, \cdot) , \ccalH_{\bbD_{t-1}})$ in \eqref{eq:hilbert_subspace_dist_ls}. 
	For brevity, we denote the model order by $M_t:=M_t(\epsilon)$ in this subsection. Consider the model order of the dictionary $\bbD_{t-1}$  and $\bbD_{t}$ generated by Algorithm \ref{alg:cub} denoted by $M_{t-1}$ and $M_{t}$, respectively, at two arbitrary subsequent times $t-1$ and $t$. The number of elements in $\bbD_{t-1}$ are $M_{t-1}$. After performing the algorithm update at $t$, we either add a new sample $(\bbx_t, y_t)$ to the dictionary and increase the model order by one, i.e., $M_{t}=M_{t-1}+1$, or we do not, in which case $M_{t}=M_{t-1}$. 
	The evolution of the conditional entropy of the algorithm, from the update in \eqref{second}, allows us to write
	\begin{align}\label{second2}
		{H}(\bby_{t}) 
		&= {H}(\bby_{t-1})+{H}(y_{t}|\bby_{t-1}). 
		%
		%
		%
	\end{align}
	Suppose the model order $M_{t}$ is equal to that of $M_{t-1}$, i.e. $M_{t} = M_{t-1}$. We skip the posterior update if ${H}( y_{t}|\bby_{t-1})\leq \epsilon$. In other words, we increase the model order by $1$ if 
	\begin{align}\label{condition}
		\frac{1}{2}\log\big(2\pi e(\sigma^{2}+{\sigma}^{2}_{\bbD_{t-1}}(\bbx_{t}))\big)>{\log\big(\sqrt{2\pi e\sigma^{2}}\big)+\epsilon}. 
	\end{align}
	\blue{As detailed in Remark \ref{remark1},  note that for each action $\bbx_t$ (cf. Algorithm \ref{alg:cub}), the condition in \eqref{condition} is equivalent to check if the posterior covariance
		\begin{align}
			\label{check2}		\sigma^2_{\bbD_{t-1}}(\bbx_t) > {\sigma^2\left({\exp(2\epsilon)}-1\right)},
		\end{align}
		and only then we perform the action $\bbx_t$ to sample $y_t$ and update the posterior distribution via updating the current dictionary. 
		We   
		For $\epsilon=0$, the condition in \eqref{check2} trivially holds, and the algorithm reduces to the updates in \citep{srinivas2012information}. Next, from Lemma \ref{lemma_subspace_dist}, we can write the condition in \eqref{check2} equivalently as }
		\begin{align}
			\label{check4}	\text{dist}^2(\kappa(\bbx_t, \cdot) , \ccalH_{\bbD_{t-1}})  > {\sigma^2\left({\exp(2\epsilon)}-1\right)}.
		\end{align}
		Taking square root on both sides, we obtain
		\begin{align}
			\label{check5}	\text{dist}(\kappa(\bbx_t, \cdot) , \ccalH_{\bbD_{t-1}})  > {\sqrt{\sigma^2\left({\exp(2\epsilon)}-1\right)}}.
		\end{align}
		which states that we only update the dictionary if the distance of the kernel function $\kappa(\bbx_t, \cdot)$ evaluated at $\bbx_t$ to the hilbert space spanned by the current dictionary $\bbD_{t-1}$ is strictly greater than $\exp(\epsilon)$. 
		
		Therefore, for a fixed $\epsilon$, the criterion to not increase the model order by $1$ is violated for the new action $\bbx_t$ whenever for distinct $\bbd_i\in\bbD_{t-1}$ and $\bbd_j\in\bbD_{t-1}$, it holds that $\|\kappa(\bbd_i,\cdot)-\kappa(\bbd_j,\cdot)\|_{\mathcal{H}}>{\sqrt{\sigma^2\left({\exp(2\epsilon)}-1\right)}}$ for some $\epsilon>0$. We may now follow the argument provided in \cite[Theorem 3.1]{1315946}. Since we assumes that $\ccalX$ is compact and $\kappa$ is continuous, the range $\kappa(\bbx,\cdot)$ of the kernel transformation of $\ccalX$ would be compact.  Hence, the number of balls balls (covering number) of radius $\delta$ (here, $\delta = {\exp(\epsilon)}$) required to cover  $\phi(\ccalX)$ is finite for a fixed $\epsilon>0$. To sharpen this dependency, we utilize \cite[Proposition 2.2]{1315946} which states that for a Lipschitz continuous Mercer kernel $\kappa$ on compact set $\mathcal{X}\subseteq\mathbb{R}^p$, there exists a constant $Y$ such that for any training set $\{\bbx_u\}_{u\leq t}$ and any $\nu>0$, and it holds that $M_t\leq M$ for all $t$ where $M$ satisfies 
		\begin{align}\label{model_order_1}
			M_t\leq	M\leq Y\left(\frac{1}{\nu}\right)^p.
		\end{align}
		From the previous reasoning, we have that $\nu=e^\epsilon$, which we substitute into \eqref{model_order_1} to obtain
		\begin{align}\label{model_order_2}
			M_t\leq		M\leq Y{\left(\frac{1}{\sigma^2\left({\exp(2\epsilon)}-1\right)}\right)^p}.
		\end{align}
		as stated in \eqref{main_theorem_2}. 
}

\section{Proof of Theorem \ref{lemma:main}}\label{proof_theorem_2}

The statement of Theorem \ref{lemma:main} is divided into two parts for finite decision set (statement (i)) and compact convex action space (statement (ii)). Next, we present the proof for both the statements separately. 

\subsection{Proof of Theorem \ref{lemma:main} statement (i)}
The proof of Theorem \ref{lemma:main}(i) is based on upper bounding the difference $|f(\bbx)-{\mu}_{t-1}(\bbx)|$ in terms of a scaled version of the standard deviation ${\beta}_{t}^{1/2}{\sigma}_{t-1}(\bbx)$, which we state next.
%
\begin{lemma}\label{lemma1}
	Choose $\delta \in (0,1)$ and let $\beta_{t} = 2\log(|\mathcal{X}|\pi_{t}/\delta)$, for some $\pi_{t}$ such that $\sum_{t\geq 1} \pi_{t}^{-1} = 1, \pi_{t} >0$. Then, the parameters of the approximate GP posterior in Algorithm \ref{alg:cub} satisfies
	\begin{align}\label{eq:lemma1} 
		|f(\bbx)-{\mu}_{t-1}(\bbx)| \leq {\beta}_{t}^{1/2}{\sigma}_{t-1}(\bbx)\hspace{0.1in} \forall \bbx \in \mathcal{X}, \forall t \geq 1
	\end{align} 
	holds with probability at least $ 1-\delta$.
\end{lemma}
%
\begin{proof}
	At each $t$, we have dictionary $\bbD_{t-1}$ which contains the data points (actions taken so far) for the function $f(\bbx)$. For a given $\bbD_{t-1}$ and $\bby_{\bbD_t}$, $f(\bbx)  \sim \mathcal{N}({\mu}_{t-1},{\sigma}_{t-1})$.
	%
	where $({\mu}_{t-1},{\sigma}_{t-1})$ are the parameters of a Gaussian whose entropy is given by $H({\mathcal{G}}_{t-1})=\frac{1}{2}\log|2\pi e \sigma_{t-1} |$. This Gaussian is parametrized by the collection of data points $(\bbx, y)\in\ccalS_{\bbD_{t-1}}$. At $t$, we  take an action $ \bbx_t$ after which we observe $ y_t$. Then, we check for the conditional entropy ${H}( y_{t}|\bby_{t-1})$. If the conditional entropy is higher than $\epsilon$ then we update the GP distribution, otherwise not \eqref{eq:compression_rule}.  Hence, there is a fundamental difference between the posterior distributions \emph{and} action selections as compared to \cite{srinivas2012information}.  We seek to analyze the performance of the proposed algorithm in terms the regret defined against the optimal $f(\bbx^*)$. To do so, we explot some properties of the Gaussian, specifically, for random variable $r \sim N(0,1)$, the cumulative density function can be expressed
	\begin{align}\label{eq:gaussian_tail_integration}
		P(r>c) &= \frac{1}{\sqrt{2\pi}}\int_{c}^{\infty}e^{-\frac{r^2}{2}}dr\\
		&= e^{\frac{-c^2}{2}}\frac{1}{\sqrt{2\pi}}\int_{c}^{\infty}e^{((\frac{-r^2}{2}+rc-\frac{c^2}{2})+(-rc-c^2))}dr \nonumber \\
		&=e^{-\frac{c^2}{2}}\frac{1}{\sqrt{2\pi}}\int_{c}^{\infty}e^{-\frac{(r-c)^2}{2}}e^{-c(r-c)}dr. \nonumber
	\end{align}
	For $c > 0$ and $r\geq c$, we have that $e^{-c(r-c)}\leq 1$. Furthermore, the integral term scaled by $\frac{1}{\sqrt{2\pi}}$ resembles the Gaussian density integrated from $c$ to $\infty$ for a random variable $r$ with mean $c$ and unit standard deviation, integrated to $1/2$. Therefore, we get
	\begin{align}\label{eq:gaussian_tail_approximation}
		P(r>c) \leq e^{-c^{2}/2} P(r>0)=\frac{1}{2}e^{-c^{2}/2}.
	\end{align}
	Using the expression $r = (f(\bbx)-{\mu}_{t-1}(\bbx))/{\sigma}_{t-1}(\bbx)$  and $c = \beta_{t}^{1/2}$ for some sequence of nonnegative scalars $\{\beta_t\}_{t\geq 0}$. Substituting this expression into the left-hand side of \eqref{eq:gaussian_tail_integration} using the left-hand side of \eqref{eq:gaussian_tail_approximation}, we obtain
	\begin{align}
		P\Big\{|f(\bbx)-{\mu}_{t-1}(\bbx)|&>\beta_{t}^{1/2}{\sigma}_{t-1}(\bbx)\Big\}\leq e^{-\beta_{t}/2}.
	\end{align}
	Now apply Boole's inequality to the preceding expression to write
	\begin{align}
		P\Big\{{{\bigcup}}_{\bbx \in \mathcal{X}}& |f(\bbx)-{\mu}_{t-1}(\bbx)|>\beta_{t}^{1/2}{\sigma}_{t-1}(\bbx)\Big\}\nonumber
		\\
		&\leq \sum_{\bbx \in \mathcal{X}} P\Big\{|f(\bbx)-{\mu}_{t-1}(\bbx)|>\beta_{t}^{1/2}{\sigma}_{t-1}(\bbx)\Big\}\nonumber
		\\
		&\leq |\mathcal{X}|e^{-\beta_{t}/2}.
	\end{align}
	To obtain the result in the statement of Lemma \ref{lemma1}, select the constant sequence $\beta_t$ such that $|\mathcal{X}|e^{\beta_{t}/2}=\frac{\delta}{\pi_{t}}$, with scalar parameter sequence $\pi_{t} := \pi^{2}t^{2}/6$. Applying Boole's inequality again over all time points $t \in \mathbb{N}$, we get 
	\begin{align}
		P\Big\{\bigcup_{t=1}^{\infty}& |f(\bbx)-{\mu}_{t-1}(\bbx)|>\beta_{t}^{1/2}{\sigma}_{t-1}(\bbx)\Big\}\nonumber
		\\
		\leq& \sum_{t=1}^{\infty} P\Big\{|f(\bbx)-{\mu}_{t-1}(\bbx)|>\beta_{t}^{1/2}{\sigma}_{t-1}(\bbx)\Big\}\nonumber\\
		\leq& \sum^{\infty}_{t=1}\frac{\delta}{\pi_{t}}\nonumber\\ 
		=& \delta.
	\end{align}
	The last equality $\sum^{\infty}_{t=1}\frac{\delta}{\pi_{t}} = \delta$ is true since $\sum^{\infty}_{t=1}1/t^{2} = \pi^{2}/6$. 
	We reverse the inequality to obtain an upper bound on the absolute difference between the true function and the estimated mean function for all $\bbx \in \mathcal{X}$ and $t \geq 1$ such that
	\begin{align}
		|f(\bbx)-{\mu}_{t-1}(\bbx)| \leq {\beta}_{t}^{1/2}{\sigma}_{t-1}(\bbx),\hspace{0.1in} \forall \bbx \in \mathcal{X}, \ \forall t \geq 1
	\end{align}
	holds with probability $1-\delta$, as stated in Lemma \ref{lemma1}.
\end{proof}

\begin{lemma}\label{lemma2}
	Fix $t \geq 1$. If 	$|f(\bbx)-{\mu}_{t-1}(\bbx)| \leq \beta_{t}^{1/2}{\sigma}_{t-1}(\bbx)$ for all $\bbx \in \mathcal{X}$, the instantaneous regret is bounded as
	\begin{align}\label{eq:lemma2}
		r_{t} \leq 2\beta_{t}^{1/2}{\sigma}_{t-1}(\bbx_{t}).
	\end{align}
	{with probability $1-\delta$ for $t\in\mathcal{M}_T(\epsilon)$. }
\end{lemma}
%
\begin{proof}
	Since Algorithm \ref{alg:cub} chooses the next sampling point $\bbx_{t}=\argmax {\mu}_{t-1}(\bbx)+\sqrt{\beta_{t}}{\sigma}_{t-1}(\bbx)$ at each step, we have
	\begin{align}
		{\mu}_{t-1}(\bbx_{t})+\sqrt{\beta_{t}}{\sigma}_{t}(\bbx_{t}) &\geq  {\mu}_{t-1}(\bbx^{*})+\sqrt{\beta_{t}}{\sigma}_{t-1}(\bbx^{*})\nonumber 
		\\
		&\geq f(\bbx^{*}),
	\end{align} 
	by the definition of the maximum,  where $\bbx^{*}$ is the optimal point. {\blue{Even if we do not store $\bbx_t$ into the dictionary, but we have selected it as an action, which is sufficient to incurs a regret, we need to consider regret for each $t$}. } The instantaneous regret is then bounded as
	\begin{align}
		r_{t} =& f(\bbx^{*})-f(\bbx_{t}) \nonumber 
		\\
		\leq&  {\mu}_{t-1}(\bbx)+\beta^{1/2}_{t}{\sigma}_{t-1}(\bbx_{t})-f(\bbx_{t}).
	\end{align}
	But from Lemma \ref{lemma1} we have that $|f(\bbx)-{\mu}_{t-1}(\bbx)| \leq \beta_{t}^{1/2}{\sigma}_{t-1}(\bbx)$ holds with probability $1-\delta$. This implies that 
	\begin{align}\label{regret1}
		r_{t} = f(\bbx^{*})-f(\bbx_{t}) \leq  2\sqrt{\beta_{t}}{\sigma}_{t-1}(\bbx_{t}).
	\end{align}
	Then, \eqref{regret1} quantifies the instantaneous regret of action ${\bbx}_t$ taken by Algorithm \ref{alg:cub}, as stated in Lemma \ref{lemma2}.

\end{proof}
\begin{lemma}\label{lemma3}
	The information gain of actions selected by Algorithm \ref{alg:cub}, denoted as {$\bbf_T=(f(\bbx_t))\in\mathbb{R}^{M_T(\epsilon)}$}, admits a closed from in terms of the posterior variance of the compressed GP and the variance of the noise prior as
	\begin{align}\label{infomation_gain}
		I(\bby_{T};\bbf_{T})=& 
		\frac{1}{2} \sum_{t\in \mathcal{M}_T(\epsilon)}\log(1+\sigma^{-2}{\sigma}^{2}_{t-1}(\bbx_{t}))
	\end{align}
	where $\mathcal{M}_T(\epsilon)$ denotes the collection of instance for which we perform the posterior distribution update.
\end{lemma}
%
\begin{proof}
	The standard GP \eqref{eq:posterior_parameters} incorporates all past actions $\bbX_t=[\bbx_1, \bbx_2, \cdots, \bbx_t]$ and observations  $\bby_t=[y_1, y_2, \cdots, y_t]^T$ into its representation. In contrast, in Algorithm \ref{alg:cub}, due to conditional entropy-based compression, we retain only a subset of the elements $\ccalS_{\bbD_t}$  with $M_t(\epsilon)$ points such that $|\bbD_t| =M_t(\epsilon) \leq t$ for all $t$. Next, we note that for a dense GP with covariance matrix $\sigma^{2}\bbI$, the information gain is given as \citep{cover2012elements} 
	\begin{align}\label{information}
		I(\bby_{t};\bbf_{t}) = {H}(\bby_{t})-\frac{1}{2}\log|2\pi e\sigma^{2}\bbI|,
	\end{align}
	where it holds that
	\begin{align}\label{first0}
		\frac{1}{2}\log|2\pi e\sigma^{2}\bbI| = \frac{1}{2}\sum^{T}_{t=1}\log(2\pi e\sigma^{2})
	\end{align}
	since $\bby_t\in\mathbb{R}^t$. In contrast, for Algorithm \ref{alg:cub}, we have $\bby_t\in\mathbb{R}^{M_t(\epsilon)}$. 
	%
	Next, expand the entropy term ${H}(\bby_{t})$ where $\bby_{t}=[\bby_{t-1}; y_{t}]$ before compression to write
	\begin{align}\label{second}
		{H}(\bby_{t}) 
		&= {H}(\bby_{t-1})+{H}(y_{t}|\bby_{t-1}) 
		\\
		&= {H}(\bby_{t-1})+\frac{1}{2}\log\big(2\pi e(\sigma^{2}+{\sigma}^{2}_{t-1}(\bbx_{t}))\big). \nonumber
		%
		%
		%
	\end{align}
	%
	We add the current point $( \bbx_t, y_t)$ only if its conditional entropy ${H}( y_{t}|\bby_{t-1})$ is more than $\epsilon$. Otherwise, the GP is unchanged, and we drop the update. This further allows us to \emph{avoid sampling} $y_t$ for $t\notin \mathcal{M}_T(\epsilon)$. That is, the GP parameters remain constant for $|{H}(\bby_{t})- {H}(\bby_{t-1})| \leq \epsilon$. The above expression holds for each $t$, now take summation over $t=1$ to $T$, since $H(\bby_0)=0$, we get 
	\begin{align}\label{third}
		{H}(\bby_{T}) 
		=& \frac{1}{2}\sum_{t\in\mathcal{M}_T(\epsilon)}\log(2\pi e\sigma^{2})+\frac{1}{2} \sum_{t\in\mathcal{M}_T(\epsilon)}\log(1+\sigma^{-2}{\sigma}^{2}_{t-1}(\bbx_{t})).
	\end{align}
	%
	where the summation is only over the instances for which we perform the dictionary update collected in $\mathcal{M}_T(\epsilon)$. Note that for the trivial case of $\epsilon=0$, we have $\mathcal{M}_T(\epsilon)=\{1,2,\cdots,T\}$. The first summand is the conditional entropy, and the later is the information gain. That is, from \eqref{information}, {we may write the corresponding expression for the $\bby_T\in\mathbb{R}^{M_T(\epsilon)}$ as follows}
	\begin{align}\label{fourth}
		I(\bby_{T};\bbf_{T})=& 
		\frac{1}{2}\sum_{t\in\mathcal{M}_T(\epsilon)}\log(1+\sigma^{-2}{\sigma}^{2}_{t-1}(\bbx_{t}))
	\end{align} 
	which is as stated in Lemma \ref{lemma3}.
\end{proof}
	\blue{\begin{lemma}\label{lemma_compression_info_gain}
	Let us define  $\beta_t$ as in Lemma \ref{lemma1} and choose $\delta\in(0,1)$, then for Algorithm \ref{alg:cub}, with probability at least $\geq 1-\delta$ we have 
	\begin{align}
		\sum_{t=1}^Tr_t^2\leq \beta_{T}C_1\left[\gamma_{M_T(\epsilon)}+\epsilon |\mathcal{M}_T^C(\epsilon)|\right]
	\end{align}
	where $C_1=\frac{8}{\log(1+\sigma^{-2})}$.
\end{lemma}}
\begin{proof}
	By Lemma \ref{lemma1} and \ref{lemma2}, we have 
	\begin{align}
		r_{t}^{2} \leq 4\beta_{t}{\sigma}_{t-1}^{2}(\bbx_{t})
	\end{align}
	{\blue{for all $t$ with probability $1-\delta$. Since $\beta_{t}$ is non-decreasing}}, we can write
	\begin{align}\label{regret}
		r_{t}^{2} \leq 4\beta_{t}{\sigma}_{t-1}^{2}(\bbx_{t}) \leq 4\beta_{T}\sigma^2(\sigma^{-2}{\sigma}_{t-1}^{2}(\bbx_{t})).
	\end{align}
	%
	%
	In addition, note that, by definition, we restrict $\kappa(\bbx,\bbx') \leq 1$. Thus, ${\sigma}^{2}_{t-1}(\bbx_{t}) = \kappa(\bbx_{t},\bbx_{t}) \leq 1$ for all $t$. Furthermore, 
	using the fact that $\frac{s}{\log(1+s)}$ is monotonically increasing for positive $s$, we get
	\begin{align}
		\frac{\sigma^{-2}{\sigma}^{2}_{t-1}(\bbx_{t})}{\log(1+\sigma^{-2}{\sigma}^{2}_{t-1}(\bbx_{t}))} \leq \frac{\sigma^{-2}}{\log(1+\sigma^{-2})}.
	\end{align} 
	This implies that
	\begin{align}
		{\sigma^{-2}{\sigma}^{2}_{t-1}(\bbx_{t})} \leq \frac{\sigma^{-2}}{\log(1+\sigma^{-2})}{\log(1+\sigma^{-2}{\sigma}^{2}_{t-1}(\bbx_{t}))}.
	\end{align} 
	Multiplying both sides by $4\beta_{T}\sigma^{2}$, we obtain
	\begin{align}\label{upper_bound}
		4\beta_{T}{\sigma}^{2}_{t-1}(\bbx_{t})%
		\leq 4\beta_{T}C_{2} \log(1+\sigma^{-2}{\sigma}^{2}_{t-1}(\bbx_{t}))
	\end{align}
	where $C_{2} = \frac{\sigma^{-2}}{\log(1+\sigma^{-2})}$. Next, substitute the upper bound in \eqref{upper_bound} on the right hand side of \eqref{regret}, we get
	\begin{align}\label{eq:dense_sum0}
		%
		%
		r_{t}^{2} &\leq 4\beta_{T}{\sigma}^{2}C_{2}\log(1+\sigma^{-2}{\sigma}^{2}_{t-1}(\bbx_{t})).
	\end{align}
	{Next, taking sum over all instances for {\blue{$t=1$ to $T$}}, we obtain 
		%
		\blue{\begin{align}\label{eq:dense_sum1}
			%
			%
			\sum_{t=1}^Tr_t^2&\leq 4\beta_{T}{\sigma}^{2}C_{2}	\sum_{t=1}^T\log(1+\sigma^{-2}{\sigma}^{2}_{t-1}(\bbx_{t}))
			\nonumber
			\\
			&= 4\beta_{T}{\sigma}^{2}C_{2}	\left[\sum_{t\in\mathcal{M}_T(\epsilon)}\log(1+\sigma^{-2}{\sigma}^{2}_{t-1}(\bbx_{t}))+\sum_{t\notin\mathcal{M}_T(\epsilon)}\log(1+\sigma^{-2}{\sigma}^{2}_{t-1}(\bbx_{t}))\right] .
		\end{align}}
	\blue{	This is the key step. In \eqref{eq:dense_sum1}, note that even if the summation is over $t=1$ to $T$ similar to one performed in \cite{srinivas2012information}, each of the action $\{\bbx_t\}_{t=1}^T$ is not stored in the dictionary. So the effect on the regret and the complexity is not similar to \cite{srinivas2012information} and demands a separate study as done in this work.   Then, {let us further define the set $\mathcal{M}_T^C(\epsilon):=\{t\notin\mathcal{M}_T(\epsilon)\}$} and apply Lemma \ref{lemma3} to  the summation term on the right-hand side of \eqref{eq:dense_sum1}, define, $C_1=\frac{8}{\log(1+\sigma^{-2})}$, we get}
		%
\blue{		\begin{align}\label{eq:sequential_decomposition}
				\sum_{t=1}^Tr_t^2&\leq \beta_{T}C_1\left[I(\bby_{T};\bbf_{T})+	\sum_{t\in\mathcal{M}_T^C(\epsilon)}\frac{1}{2}\log(1+\sigma^{-2}{\sigma}^{2}_{t-1}(\bbx_{t}))\right].
		\end{align}}
		%
	
	\blue{From the compression rule in Algorithm \ref{alg:cub}, we know that $\frac{1}{2}\log(1+\sigma^{-2}{\sigma}^{2}_{t-1}(\bbx_{t}))\leq \epsilon$ for all $t\in\mathcal{M}_T^C(\epsilon)$. Hence, we can write the expression in \eqref{eq:sequential_decomposition} as }
\blue{		\begin{align}\label{eq:sequential_decomposition}
		\sum_{t=1}^Tr_t^2&\leq \beta_{T}C_1\left[I(\bby_{T};\bbf_{T})+ \epsilon |\mathcal{M}_T^C(\epsilon)|\right].
\end{align}
		Through the definition of the maximum information gain $\gamma_T$ in \eqref{eq:max_info_gain}, we can write
		\begin{align}\label{information_gain}
			\gamma_{M_T(\epsilon)}:=\max_{ \{\bbx_u \} } I(\{y_u\}_{{u=1}}
			^{M_T(\epsilon)}  ; f) \ \text{ such that } \ |\{\bbx_u \} |=M_T(\epsilon).
		\end{align}
		From the definition of the $\gamma_{M_T(\epsilon)}$, we have $I(\bby_{T};\bbf_{T})\leq \gamma_{|M_T(\epsilon)|}$. 
		%
		Utilizing \eqref{information_gain} into the right hand side of \eqref{eq:sequential_decomposition}, we obtain
		\begin{align}\label{eq:sequential_decomposition2}
		\blue{	\sum_{t=1}^Tr_t^2\leq \beta_{T}C_1\left[\gamma_{M_T(\epsilon)}+\epsilon |\mathcal{M}_T^C(\epsilon)|\right]}
	\end{align}}}
\end{proof}
\blue{{Now, it is straightforward to establish the result of Theorem \ref{lemma:main} statement (i). Note that by the Cauchy Schwartz inequality, we can write
	\begin{align}
		{\textbf{Reg}}_T\leq & \sqrt{T\sum_{t=1}^Tr_t^2}
		\nonumber
		\\
		\leq &  \sqrt{C_{1}T\beta_{T} \gamma_{M_T(\epsilon)} }+ \sqrt{C_{1}T\beta_{T} \epsilon |\mathcal{M}_T^C(\epsilon)|} \label{final_bound}. 
\end{align}}
	\blue{Next, we utilize the trivial upper bound on $|\mathcal{M}_T^C(\epsilon)|\leq T$ and write
		\begin{align}
		{\textbf{Reg}}_T\leq &  \sqrt{C_{1}T\beta_{T} \gamma_{M_T(\epsilon)} }+ T \sqrt{C_{1}\beta_{T} \epsilon } \label{final_bound2}. 
	\end{align}
We substitute the selection $\epsilon={\frac{1}{2}\log\left(1+T^{-\alpha}\right)}$ with $\alpha\in(0,\frac{1}{p})$ to obtain
\begin{align}
	{\textbf{Reg}}_T\leq &  \sqrt{C_{1}T\beta_{T} \gamma_{M_T(\epsilon)} }+ {T\sqrt{\log\left(1+T^{-\alpha}\right)} \sqrt{{C_{1}\beta_{T}}}} \label{final_bound3}. 
\end{align}
}
	which is as stated in Theorem \ref{lemma:main}(i). \blue{The above expression actually reflects the tradeoff because the regret is now decomposed into two parts: the first part is due to the information gain and the second part reflects the complexity of the algorithm controlled by parameter $\alpha$ which we are free to choose. }}\hfill \qed

\subsection{Proof of Theorem \ref{lemma:main} statement (ii)} Now, we present the regret analysis for the general settings where $\mathcal{X}\subset \mathbb{R}^d$ is a compact set. It is nontrivial to extend  Theorem \ref{lemma:main}(i) to the general compact action spaces. For instance, the result in Lemma \ref{lemma1} does not hold for infinite action space $\mathcal{X}$ since it involves the use of $|\mathcal{X}|$ which is infinite for the general compact space $\mathcal{X}$, which causes the bound in Lemma \ref{lemma1} to be infinite. We proceed with a different approach based on exploiting smoothness hypotheses we impose on the underlying ground truth function $f$.

We begin by stating an analog of Lemma \ref{lemma1} that holds for continuous spaces which quantify the confidence of the decisions taken using Algorithm \ref{alg:cub}.
\begin{lemma}\label{lemma4}
	Select exploration parameter $\beta_t=2\log(\pi_t/\delta)$ and choose likelihood threshold $\delta\in(0,1)$ with $\sum\limits_{t\geq}\frac{1}{\pi_t}=1,  \pi_t>0$. Then for the Algorithm \ref{alg:cub} we have that
	\begin{align}
		|f( \bbx_t)-\mu_{t-1}( \bbx_t)|\leq \beta_t^{1/2}\sigma_{t-1}(\bbx_t), \ \ \forall t\geq 1
	\end{align}
	holds with probability at least $ 1-\delta$.
\end{lemma}
\begin{proof}
	For a given $t$ and $\bbx\in\mathcal{X}$, the dictionary $\bbD_{t-1}$ elements are deterministic conditioned on the observations $ \bby_{t-1}$, which implies that $f(\bbx)\sim\mathcal{N}(\mu_{t-1(\bbx)},{\sigma}^2_{t-1}(\bbx))$. Following the similar steps to the proof of Lemma \ref{lemma1}, it holds that 
	\begin{align}
		P\Big\{|f(\bbx_t)-{\mu}_{t-1}(\bbx_t)|&>\beta_{t}^{1/2}{\sigma}_{t-1}(\bbx_t)\Big\}\leq e^{-\beta_{t}/2}.
	\end{align}
	Since we have $\beta_t=2\log(\pi_t/\delta)$, apply Boole's inequality (union bound) for $t\in\mathbb{N}$ to conclude Lemma \ref{lemma4}.
\end{proof}
Note that the result in Lemma \ref{lemma4} is for a particular action $ \bbx_t$ of Algorithm \ref{alg:cub} rather than for any action $\bbx$ as given by Lemma \ref{lemma1}. To derive the regret of the Algorithm \ref{alg:cub}, we need to characterize the confidence bound stated in Lemma \ref{lemma4} for the optimal action $\bbx^*$. To do so, we discretize the action space $\mathcal{X}$ into different sets $\mathcal{X}_t\subset\mathcal{X}$ and we use $\mathcal{X}_t$ at instance $t$. This discretization is purely for the purpose of analysis and has not been used in the algorithm implementation. We provide the confidence for these subsets  $\mathcal{X}_t$ in the next Lemma \ref{lemma5}.
\begin{lemma}\label{lemma5}
	Select exploration parameter $\beta_t=2\log(|\mathcal{X}_t|\pi_t/\delta)$ and likelihood tolerance $\delta\in(0,1)$ with $\sum\limits_{t\geq}\frac{1}{\pi_t}=1,  \pi_t>0$. Then Algorithm \ref{alg:cub} satisfies
	\begin{align}
		|f(\bbx)-\mu_{t-1}(\bbx)|\leq \beta_t^{1/2}\sigma_{t-1}(\bbx), \ \ \forall \bbx\in\mathcal{X}_t, \ \  \forall t\geq 1
	\end{align}
	with probability at least $ 1-\delta$.
\end{lemma}
The proof for the statement of Lemma \ref{lemma5} is analogous to Lemma \ref{lemma1}. The distinguishing feature is that we replace $\mathcal{X}$ with $\mathcal{X}_t$. Next, to obtain the regret bound for Algorithm \ref{alg:cub}, we need to characterize the confidence bound for optimal action $\bbx^*$. Doing so first requires bounding the error due to the discretization. From the hypothesis stated in Theorem \ref{lemma:main}(ii), we may write
\begin{align}
	\mathbb{P}\left\{\forall j, \forall \bbx, \ \  |\partial f / \partial \bbx_j | < L \right\} \geq 1- a de^{-(L/b)^2}
\end{align}
which states that the function $f$ is Lipschitz with probability greater than $1- a de^{-(L/b)^2} $, hence it holds that
\begin{align}\label{lips}
	|f(\bbx)-f(\bbx')|\leq L\|\bbx-\bbx'\|_1
\end{align}
for all $\bbx\in\mathcal{X}$. To obtain the confidence at $\bbx^*$, choose the discretization such that the size of each set $\mathcal{X}_t$ is $(\tau_t)^d$ so that for each $\bbx\in\mathcal{X}$, it holds that
\begin{align}\label{bound}
	\|\bbx-[\bbx]_t\|_1\leq \frac{rd}{\tau_t}
\end{align}
where $[\bbx]_t$ is the closest point in $\mathcal{X}_t$ to the original point $\bbx$. Next, we present the result which provides the confidence for $\bbx^* $ in Lemma \ref{lemma8}. 
\begin{lemma}\label{lemma8}
	Suppose that exploration parameter is selected as $\beta_t=2\log(2\pi_t/\delta)+4d\log(dtbr\sqrt{\log(2da/\delta)})$ and fix likelihood tolerance $\delta\in(0,1)$ such that $\sum\limits_{t\geq}\frac{1}{\pi_t}=1,  \pi_t>0$ and $\tau_t=t^2rdb\sqrt{\log(2da/\delta)}$. Then Algorithm \ref{alg:cub} satisfies
	\begin{align}
		|f(\bbx^*)-\mu_{t-1}([\bbx^*]_t)|\leq& \frac{1}{t^2}+\beta_t^{1/2}\sigma_{t-1}([\bbx^*]_t)., \ \ \forall \bbx\in\mathcal{X}_t, \ \  \forall t\geq 1
	\end{align}
	with probability at least $1-\delta$.
\end{lemma}

\begin{proof}
	Let us denote $\frac{\delta}{2}=dae^{\frac{-L^2}{b^2}}$, then from the Lipschitz property in \eqref{lips}, we can write that
	\begin{align}
		|f(\bbx)-f(\bbx')|\leq b\sqrt{\log(2da/\delta)}\|\bbx-\bbx'\|_1
	\end{align}
	for all $\bbx\in\mathcal{X}$ with probability greater than $1-\frac{\delta}{2}$. Since the expression holds for any $\bbx'$, let use choose $\bbx'=[\bbx]_t$, we get
	\begin{align}
		|f(\bbx)-f([\bbx]_t)|\leq b\sqrt{\log(2da/\delta)}\|\bbx-[\bbx]_t\|_1.
	\end{align}
	From the bound in \eqref{bound}, we get
	\begin{align}
		|f(\bbx)-f([\bbx]_t)|\leq rdb\sqrt{\log(2da/\delta)}/\tau_t.
	\end{align}
	By selecting the discretization $\tau_t=t^2rdb\sqrt{\log(2da/\delta)}$, we can write
	\begin{align}\label{bound_f}
		|f(\bbx)-f([\bbx]_t)|\leq \frac{1}{t^2}.
	\end{align}
	for all $\bbx\in\mathcal{X}$. Next, we add and subtract the optimal discretized point $f([\bbx^*]_t)$ as
	\begin{align}\label{here}
		|f(\bbx^*)-\mu_{t-1}([\bbx^*]_t)|=&|f(\bbx^*)-f([\bbx^*]_t)-(f([\bbx^*]_t)-\mu_{t-1}([\bbx^*]_t))|\nonumber 
		\\
		\leq& |f(\bbx^*)-f([\bbx^*]_t)|+|(f([\bbx^*]_t)-\mu_{t-1}([\bbx^*]_t))|.
	\end{align}
	From Lemma \ref{lemma5} and the upper bound in \eqref{bound_f}, we can rewrite the inequality in \eqref{here} as follows
	\begin{align}\label{result}
		|f(\bbx^*)-\mu_{t-1}([\bbx^*]_t)|\leq& \frac{1}{t^2}+\beta_t^{1/2}\sigma_{t-1}([\bbx^*]_t).
	\end{align}
	which is stated in Lemma \ref{lemma8}.
\end{proof}
Next, we provide a Lemma which characterizes the regret $r_t$ at each instant $t$ for the general compact action spaces. The result is stated in Lemma \ref{lemma9}.
\begin{lemma}\label{lemma9}
	Suppose the exploration parameter is selected as $\beta_t=2\log(4\pi_t/\delta)+4d\log(dtbr\sqrt{\log(4da/\delta)})$ and with likelihood tolerance $\delta\in(0,1)$ chosen such that $\sum\limits_{t\geq}\frac{1}{\pi_t}=1,  \pi_t>0$ and discretization parameter satisfying $\tau_t=t^2rdb\sqrt{\log(2da/\delta)}$. Then Algorithm \ref{alg:cub} satisfies
	\begin{align}
		r_t\leq& 2\beta_t^{1/2}\sigma_{t-1}( \bbx_t)+\frac{1}{t^2},
	\end{align}
	with probability at least $ 1-\delta$.
\end{lemma}

\begin{proof}
	In Lemma \ref{lemma4} and Lemma \ref{lemma8}, $\delta/2$ is used to make the probability of the events more than $1-\delta$. Next, note that for a general compact set $\mathcal{X}$, from the definition of the  action $\bbx_t$ in Algorithm \ref{alg:cub}, it holds that
	\begin{align}
		\mu_{t-1}(\bbx_t)+\beta_t^{1/2} \sigma_{t-1}(\bbx_t) \geq  \mu_{t-1}([\bbx^*]_t)+\beta_t^{1/2} \sigma_{t-1}([\bbx^*]_t).
	\end{align}
	From the statement of Lemma \ref{lemma8}, it holds that 
	\begin{align}
		\mu_{t-1}([\bbx^*]_t)+\beta_t^{1/2} \sigma_{t-1}([\bbx^*]_t) +\frac{1}{t^2}\geq f(\bbx^*). 
	\end{align}
	Consider the regret at $t$, which may be related to the over-approximation by the upper-confidence bound as
	\begin{align}
		r_t=& f(\bbx^*)-f(\bbx_t)\nonumber
		\\
		\leq&  \mu_{t-1}([\bbx^*]_t)+\beta_t^{1/2} \sigma_{t-1}([\bbx^*]_t) +\frac{1}{t^2} -f(\bbx_t)
		\\
		\leq&  \mu_{t-1}(\bbx_t)+\beta_t^{1/2} \sigma_{t-1}(\bbx_t) +\frac{1}{t^2} -f(\bbx_t)
		\\
		=&  \beta_t^{1/2} \sigma_{t-1}(\bbx_t) +\frac{1}{t^2} + \mu_{t-1}(\bbx_t)-f(\bbx_t).
	\end{align}
	Using the result in Lemma \ref{lemma5}, we can write
	\begin{align}\label{regret_instant}
		r_t\leq & 2 \beta_t^{1/2} \sigma_{t-1}(\bbx_t) +\frac{1}{t^2}
	\end{align}
	which completes the proof. 
\end{proof}
\blue{We are ready to present the proof of statement (ii) in Theorem \ref{lemma:main}. 
Begin by noting that the first term on the right-hand side of \eqref{regret_instant} coincides with the left-hand side of \eqref{upper_bound}, and therefore we can write 
	\begin{align}
\sum_{t=1}^T4\beta_{t}{\sigma}_{t-1}^{2}(\bbx_{t})\leq \beta_{T}C_1\left[\gamma_{M_T(\epsilon)}+\epsilon |\mathcal{M}_T^C(\epsilon)|\right]
\end{align}
Apply the Cauchy-Schwartz inequality to the preceding expression to obtain
\begin{align}\label{last2}
\sum_{t=1}^T2\beta_{t}^{1/2}{\sigma}_{t-1}(\bbx_{t})&\leq \sqrt{T{\sum_{t=1}^T4\beta_{t}{\sigma}_{t-1}^{2}(\bbx_{t})}}
	\nonumber
	\\
	&\leq \sqrt{C_{1}T\beta_{T} \gamma_{M_T(\epsilon)} }+ T \sqrt{C_{1}\beta_{T} \epsilon }
\end{align}
We substitute the selection $\epsilon={\frac{1}{2}\log\left(1+T^{-\alpha}\right)}$ with $\alpha\in(0,\frac{1}{p})$ to obtain
\begin{align}
	\sum_{t=1}^T2\beta_{t}^{1/2}{\sigma}_{t-1}(\bbx_{t})\leq &  \sqrt{C_{1}T\beta_{T} \gamma_{M_T(\epsilon)} }+ {T\sqrt{\log\left(1+T^{-\alpha}\right)} \sqrt{{C_{1}\beta_{T}}}} \label{final_bound3}. 
\end{align}
 From the statement of Lemma \eqref{lemma9} and then calculating the summation over the instances in $\mathcal{M}_T(\epsilon)$, we can write 
\begin{align} \label{final}
	\sum_{t=1}^Tr_t\leq \sqrt{C_{1}T\beta_{T} \gamma_{M_T(\epsilon)} }+ {T\sqrt{\log\left(1+T^{-\alpha}\right)} \sqrt{{C_{1}\beta_{T}}}} + \frac{\pi^2}{6}
\end{align}
which is as stated in the Theorem \ref{lemma:main}(ii).  We have also used Euler's formula to the upper bound the summation of the second term on the right-hand side of Lemma \ref{lemma9} across time  $\sum\limits_{t=1}^{T}\frac{1}{t^2}\leq \frac{\pi^2}{6}$ to conclude \eqref{final}.}\hfill \qed

\qed

{\subsection{Proof of Corollary \ref{cor:1}}\label{cor_1_proof}}
%
\blue{Here we develop the generalized bounds on the term $\gamma_{M_T(\epsilon)}$. 
	As discussed in \cite[Sec. II]{srinivas2012information}, we could upper bound  $\gamma_{M_T(\epsilon)}$ as 
	\begin{align}\label{eq_corro1}
		\gamma_{M_T(\epsilon)}\leq (1-1/\exp)^{-1}I(\bby_{\mathcal{A}_{M_T(\epsilon)}};\bbf)
	\end{align}
	where $\mathcal{A}_{M_T(\epsilon)}$ contains the actions selected by the greedy procedure. Note that each element of $\mathcal{A}_{M_T(\epsilon)}$ belongs to the finite set of available actions in $\mathcal{X}$. From \cite[Lemma 7.6]{srinivas2012information}, we can upper bound the information gain on the right hand side of \eqref{eq_corro1} as 
	\begin{align}\label{bound1}
		\gamma_{M_T(\epsilon)}\leq \frac{1}{2((1-1/\exp)^{-1})}\max_{\{m_t\}} \sum_{t=1}^{|\mathcal{X}|}\log(1+\sigma^{-2}m_t\lambda_t)
	\end{align}
	where $\sum_{t}m_t=M_T(\epsilon)$ and $\{\lambda_1\geq \lambda_2\geq \cdots\}$ denotes the eigen values of the covariance kernel matrix $\bbK_\mathcal{X}=[\kappa(\bbx,\bbx')]_{\bbx,\bbx'\in \mathcal{X}}$. 
		Utilizing the bound $\log(1+\sigma^{-2}m_t\lambda_t)\leq \sigma^{-2}m_t\lambda_t$, we get
		%
		
	%
	\begin{align}\label{bound3}
		\gamma_{M_T(\epsilon)}\leq \frac{1}{2((1-e^{-1}))}\sum_{t=1}^{|\mathcal{X}|}(\sigma^{-2}m_t\lambda_t)
	\end{align}
	Since it holds that $m_t\leq M_T(\epsilon)$, utilizing this in \eqref{bound3}, we obtain
	\begin{align}\label{bound6}
		\gamma_{M_T(\epsilon)}\leq& \frac{M_T(\epsilon)}{\sigma^2(1-e^{-1})}\sum_{t=1}^{|\mathcal{X}|}\lambda_t\nonumber
		\\
		=&Z_1M_T(\epsilon).
	\end{align}
	where $Z_1=\frac{\sum_{t=1}^{|\mathcal{X}|}\lambda_t}{\sigma^2(1-e^{-1})}$. Substituting the upper bound of \eqref{bound6} into \eqref{final_bound}, we get
	{	\begin{align}\label{cor_proof}
			{\textbf{Reg}}_T \leq & \sqrt{C_{1}T\beta_{T} \gamma_{M_T(\epsilon)} }+ \sqrt{C_{1}T\beta_{T} \epsilon |\mathcal{M}_T^C(\epsilon)|}
			\nonumber
			\\
			\leq &  T^{\frac{1+{\alpha p}}{2}}\sqrt{C_{1}\beta_{T}Z_1}+ T^{1-\frac{\alpha}{2}}\sqrt{C_{1}\beta_{T}},
	\end{align}}
where we utilize the results from Theorem \ref{model_order} and $\epsilon=\frac{1}{2}\log\left(1+T^{-\alpha}\right)$ where $\alpha\in(0,1/p)$. 	Next, after further simplification, we can write
	{	\begin{align}\label{cor_proof}
		{\textbf{Reg}}_T \leq &  \max\{\sqrt{C_{1}\beta_{T}Z_1},\sqrt{C_{1}\beta_{T}}\} T^{\max\{\frac{1+{\alpha p}}{2},1-\frac{\alpha}{2}\}}. 
\end{align}}
For simplicity, we consider $p=1$, and we will get the optimal regret for $\alpha=\frac{1}{2}$, we obtain
{	\begin{align}\label{cor_proof}
		{\textbf{Reg}}_T \leq &  \sqrt{C_{1}\beta_{T}}\max\{\sqrt{Z_1},1\} T^{3/4}. 
\end{align}}
From the statement of Theorem \ref{model_order}, we obtain $M_T(\epsilon)=\mathcal{O}\left(\sqrt{T}\right)$. Hence proved. }
%
%
\hfill \qed
\section{Proofs for Expected Improvement Acquisition Function  }\label{sec:appendix_ei}

\subsection{Definitions and Technical Lemmas}\label{sec:appendix_ei_preliminaries}

We expand upon the details of the expected improvement acquisition function. First we review a few key quantities. Define the improvement $I_t(\bbx)=\max\{0,f(\bbx) - \xi \}$ over incumbent $\xi=y^{\text{max}}_{t-1} = \max\{y_u \}_{u \leq t}$, which is the maximum over past observations. Denote by $z=z_{t-1}(\bbx) = (\mu_{t-1}(\bbx) - y^{\text{max}}_{t-1} )/\sigma_{t-1}(\bbx)$ as the $z$-score of $y^{\text{max}}_{t-1}$. Then, the expected improvement computes the expectation over improvement $I_t(\bbx)$ which may be evaluated using the Gaussian density $\phi(z)$ and distribution functions $\Phi(z)$ as:
\begin{equation}\label{eq:expected_improvement_appendix}
	\alpha^{\text{EI}}_t(\bbx)=\sigma_{t-1}\phi(z) + [\mu_{t-1}\!(\bbx) - \xi]\Phi(z)\;, \quad \xi=y^{\text{max}}_{t-1} = \max\{y_u\}_{u\leq t}
\end{equation}
As the convention in \citep{nguyen2017regret}, when the variance $\sigma_{t-1}(\bbx) = 0$, we set $\alpha^{\text{EI}}(\bbx) = 0 $. Define the function $\tau(z) = z \Phi(z) + \phi(z)$ to alleviate the notation henceforth. 

Let us define maximum observation ${y}^{\text{max}}_{t-1} = \max\{y_u \}_{u \in \ccalM_t}$ over  $\ccalM_t$, the set of indices associated with past selected points \eqref{eq:compression_rule}, the \emph{compressed improvement} ${I}_t(\bbx)=\max\{0,f(\bbx) - {y}_{t-1}^{\text{max}} \}$, and the associated $z$-scores as ${z}={z}_{t-1}(\bbx):= (\mu_{t-1}(\bbx) - {y}^{\text{max}}_{t-1} )/{\sigma}_{t-1}(\bbx)$. These definitions then allow us to define the compressed variant of the expected improvement acquisition function as
\begin{equation}\label{eq:compressed_expected_improvement}
	{\alpha}^{\text{EI}}_t(\bbx)={\sigma}_{t-1}\phi(z) + [{\mu}_{t-1}\!(\bbx) - \xi]\Phi(z)\;, \quad \xi={y}^{\text{max}}_{t-1} = \max\{y_u\}_{u\in\ccalM_t}
\end{equation}

Before proceeding with the proof, we first verify several properties and lemmas key to the regret bound in \citep{nguyen2017regret} to illuminate whether there is a dependence on the GP dictionary as $\bbX_t$ or the subset $\bbD_t$. 
\begin{lemma}\label{lemma:ei_tau_expression}
	The acquisition function $\alpha_t^{\text{EI}}(\bbx)$ in \eqref{eq:expected_improvement} may be expressed in terms of the variance, and the density $\phi$ and distribution $\Phi$ functions of the Gaussian as $\alpha_t^{\text{EI}}(\bbx) = \sigma_{t-1}(\bbx) \tau(z_{t-1}(\bbx))$. Moreover, $\alpha_t^{\text{EI}}(\bbx) \leq \tau(z_{t-1}(\bbx))$ for $\sigma_{t-1}(\bbx) \leq 1 $.
\end{lemma}
\begin{proof}
	Begin with \eqref{eq:expected_improvement}:
	$$\alpha_t^{\text{EI}}(\bbx)=\sigma_{t-1}\phi(z) + [\mu_{t-1}\!(\bbx) - \xi]\Phi(z)\;, \quad y^{\text{max}}_{t-1} = \max\{y_u \in \ccalS_t \} \; .$$
	Now substitute in the definition of the $z$-score: $z_{t-1}(\bbx) = (\mu_{t-1}(\bbx) - y^{\text{max}}_{t-1} )/\sigma_{t-1}(\bbx)$ and $\tau(z) = z \Phi(z) + \phi(z)$ to write
	\begin{align}\label{ref:lemma1_step}
		\alpha_t^{\text{EI}}(\bbx)&
		=\sigma_{t-1}(\bbx) [z \Phi(z) + \phi(z) ]  \nonumber \\
		&=\sigma_{t-1}(\bbx) \tau(z_{t-1}(\bbx) )
	\end{align}
	Using $\sigma_{t-1}(\bbx) \leq 1$ allows us to conclude Lemma \ref{lemma:ei_tau_expression}. 
\end{proof}
We underscore that \eqref{ref:lemma1_step} exploits properties of $\tau$ independent of whether $y^{\text{max}}_{t-1} $ is computed over points in $\{y_u\}_{u\leq t}$ or amongst only a subset. Therefore, as a corollary, we have that an identical property holds for the compressed expected improvement \eqref{eq:compressed_expected_improvement}.
\begin{corollary}\label{corollary:ei_tau_expression}
	The compressed expected improvement acquisition function ${\alpha}_t^{\text{EI}}(\bbx)$ in \eqref{eq:compressed_expected_improvement} satisfies the identity ${\alpha}_t^{\text{EI}}(\bbx) = {\sigma}_{t-1}(\bbx) \tau({z}_{t-1}(\bbx))$. Moreover, ${\alpha}_t^{\text{EI}}(\bbx) \leq \tau({z}_{t-1}(\bbx))$ for ${\sigma}_{t-1}(\bbx) \leq 1 $.
\end{corollary}

In contrast to \citep{nguyen2017regret}[Lemma 5] and \citep{srinivas2012information}[Theorem 6 hold], which require the target function $f^*$ to belong to an RKHS with finite RKHS norm, we focus on the case where the decision set $\ccalX$ has finite cardinality, whereby Lemma \ref{lemma1}. We consider this case to keep the analysis simple and elegant for the EI algorithm. The analysis for the general compact decision set follows similar steps as those taken for Compressed GP-UCB, but would instead employ Lemma \ref{lemma4} together with accounting for discretization-induced error, leading to an additional constant factor on the right-hand side of the regret bound.

Next, we relate the instantaneous improvement minus the scaled standard deviation to the expected improvement \eqref{eq:expected_improvement}.

\begin{lemma}\label{lemma:expected_instantaneous_improvement}
	The expected improvement \eqref{eq:expected_improvement} upper-bounds the instantaneous improvement  $I_t(\bbx)=\max\{0,f(\bbx) - y^{\text{max}}_{t-1} \}$ minus a proper scaling of the standard deviation, i.e.
	\begin{equation}\label{eq:expected_instantaneous_improvement}
		I_t(\bbx) - \sqrt{\beta_t}\sigma_{t-1}(\bbx) \leq \alpha_t^{\text{EI}}(\bbx) 
	\end{equation}
\end{lemma}
\begin{proof}
	If $\sigma_{t-1}(\bbx)= 0$, then $\alpha_t^{\text{EI}}(\bbx) =I_t(\bbx) = 0$, which makes the result hold with equality. Suppose $\sigma_{t-1}(\bbx)>0$. Then, define the following normalized quantities 
	\begin{equation}\label{eq:z_score_dense}
		q=\frac{f(\bbx) - y^{\text{max}}_{t-1}}{\sigma_{t-1}(\bbx)} \; , \quad z=\frac{\mu_{t-1}(\bbx) - y^{\text{max}}_{t-1}}{\sigma_{t-1}(\bbx)} 
	\end{equation}
	Now, consider the expression for the expected improvement \eqref{eq:expected_improvement}, using the identity of Lemma \ref{lemma:ei_tau_expression}:
	\begin{align}\label{eq:expected_instantaneous_proof}
		\alpha^{\text{EI}}(\bbx)&=\sigma_{t-1}(\bbx) \tau(z_{t-1}(\bbx) )
	\end{align}
	Now apply the upper-confidence bound, which says that $|\mu_t(\bbx) - f(\bbx) | \leq \sqrt{\beta_t} \sigma_t(\bbx)$ with probability $1-\delta$, since the action space is discrete, as in \ref{lemma1}. Doing so permits us to write 
	\begin{align}\label{eq:expected_instantaneous_proof2}
		\sigma_{t-1}(\bbx) \tau(z_{t-1}(\bbx) )
		& \geq \sigma_{t-1}(\bbx)\tau\left(q - \sqrt{\beta_t}\right)\quad \text{ with prob. } 1-\delta \\
		& \geq \sigma_{t-1}(\bbx) \left(q - \sqrt{\beta_t}\right) \quad \text{ with prob. } 1-\delta \nonumber
	\end{align}
	Subsequently, we suppress the with high probability qualifier with the understanding that it's implicit and applies to all subsequent statements.
	If $I_t(\bbx) = 0$, then \eqref{eq:expected_instantaneous_improvement} holds automatically. Therefore, suppose $I_t(\bbx) > 0$. Then, substitute the definition of $q$ into the right-hand side of \eqref{eq:expected_instantaneous_proof2} to obtain:
	\begin{align}\label{eq:expected_instantaneous_proof3}
		\sigma_{t-1}(\bbx) \left(\frac{f(\bbx) - y^{\text{max}}_{t-1}}{\sigma_{t-1}(\bbx)} - \sqrt{\beta_t}\right) \nonumber
		& = f(\bbx) - y^{\text{max}}_{t-1} - \sigma_{t-1}(\bbx)\sqrt{\beta_t} \\
		& = I_t(\bbx) - \sigma_{t-1}(\bbx)\sqrt{\beta_t}.
	\end{align}
	Thus, when we combine \eqref{eq:expected_instantaneous_proof} - \eqref{eq:expected_instantaneous_proof3}, we obtain the result stated in \eqref{eq:expected_instantaneous_improvement}.
\end{proof}
Again, we note by substituting the identity (Lemma \ref{lemma:expected_instantaneous_improvement}) that begins the proof of Lemma \ref{lemma:expected_instantaneous_improvement} by the statement of Corollary \ref{corollary:ei_tau_expression}, and defining the $z$-score quantities \eqref{eq:z_score_dense} but with substitution of ${y}^{\text{max}}_{t-1}$, we may apply properties of the upper-confidence bound (Lemma \ref{lemma1}), which continue to hold when we replace the posterior of the dense GP with that of the compressed GP. This logic permits us to obtain the following as a corollary.

\begin{corollary}\label{corollary:expected_compressed_instantaneous_improvement}
	The compressed expected improvement \eqref{eq:compressed_expected_improvement} upper-bounds the compressed instantaneous improvement  ${I}_t(\bbx)=\max\{0,f(\bbx) - {y}^{\text{max}}_{t-1} \}$ minus a proper scaling of the standard deviation, i.e.
	\begin{equation}\label{eq:expected_compressed_instantaneous_improvement}
		{I}_t(\bbx) - \sqrt{\beta_t}{\sigma}_{t-1}(\bbx) \leq {\alpha}_t^{\text{EI}}(\bbx).
	\end{equation}
\end{corollary}

Next, we present a variant of \citep{nguyen2017regret}[Lemma 7] which connects the accumulation of posterior variances to maximum information gain. This result is akin to previously stated Lemmas \ref{lemma1} and \ref{lemma2}.

\citep{nguyen2017regret}[Lemma 8] defines a constant $C$ such that the two terms on the right-hand side of Lemma \ref{lemma9} can be merged through the appropriate definition of a stopping criterion and modified definition of $\beta_t$. We obviate this additional detail through  the following modified lemma.

\blue{\begin{lemma}\label{lemma:posterior_variance_info_gain}
	The sum of the predictive variances is bounded by the maximum information gain ${\gamma_{M_T(\epsilon)}}$ [cf. \eqref{eq:max_info_gain}] as
	\begin{equation}\label{eq:posterior_variance_info_gain}
		{\sum_{t=1}^T} \sigma_{t-1}^2(\bbx_t) \leq \frac{2}{{ \log(1+\sigma^{-2})} }\left[ {\gamma_{M_T(\epsilon)}}+\epsilon |\mathcal{M}_T^C(\epsilon)|\right].
	\end{equation}
\end{lemma}
\begin{proof}
	Consider the posterior variances {at instance $t$} as 
	\begin{align}\label{eq:posterior_variance_sum1}
		\sigma_{t-1}^2(\bbx_t) 
		&= \sigma^{2} \underbrace{\sigma_{t-1}^2(\bbx_t)\sigma^{-2} }_{s^2}\nonumber \\
		&\leq \sigma^{2} \Bigg[\frac{\log(1 + s^2) }{\sigma^2 \log(1+\sigma^{-2})}  \Bigg]
	\end{align}
	where we have used the fact that the logarithm satisfies the inequality $\frac{x}{\log(1+x)} \geq 1$ for $x=\sigma^{-2}$ to write
	$\frac{1}{\sigma^{2}\log(1+\sigma^{-2}) }\geq 1$ together with $\frac{1}{\sigma^{2}\log(1+\sigma^{-2})} \geq \frac{s^2}{\log(1+s^2)}$ on the right-hand side of \eqref{eq:posterior_variance_sum1}. Now, pull the denominator outside the {square bracket}, and multiply and divide by $2$ to obtain the information for a single point $(\bbx_t, y_t)$ as $\frac{1}{2} \log(1+ \sigma^{-2} \sigma^2_{t-1}(\bbx_t)$ as:
	\begin{align}\label{eq:posterior_variance_sum2}
		\sigma^{2} \Bigg[\frac{\log(1 + s^2) }{\sigma^2 \log(1+\sigma^{-2})}  \Bigg] 
		&=\sigma^2 \frac{2}{{\sigma^2 \log(1+\sigma^{-2})} } \left[\frac{1}{2}\log(1 + \sigma_{t-1}^2(\bbx_t)\sigma^{-2})\right]
	\end{align}
	which after canceling a factor of $\sigma^2$. {Next, we take the sum over $t\in\mathcal{M}_T(\epsilon)$} and noting that the sum of information gain accumulates to that of the full set $\{y_t\}$ [cf. \eqref{eq:information_gain_gaussian}], similarly to \eqref{information} - \eqref{first0}, we obtain
	\begin{align}\label{eq:posterior_variance_sum3}
		\frac{2}{{ \log(1+\sigma^{-2})} } \frac{1}{2}{\sum_{t=1}^{T}} \log(1 + \sigma_{t-1}^2(\bbx_t)\sigma^{-2})
		& =  \frac{2}{{ \log(1+\sigma^{-2})} }  \left[I(\bby_{T};\bbf_{T})+	\sum_{t\in\mathcal{M}_T^C(\epsilon)}\frac{1}{2}\log(1+\sigma^{-2}{\sigma}^{2}_{t-1}(\bbx_{t}))\right]
		\nonumber
		\\
		& \leq \frac{2}{{ \log(1+\sigma^{-2})} }\left[ {\gamma_{M_T(\epsilon)}}+\epsilon |\mathcal{M}_T^C(\epsilon)|\right]
	\end{align}
	where ${\gamma_{M_T(\epsilon)}}$ is the maximum information gain over ${M_T(\epsilon)}$ points [cf. \eqref{eq:max_info_gain}].
\end{proof}}
Here is a key point of departure in the analysis of employing conditional entropy-based compression \eqref{eq:compression_rule} relative to the dense GP. Lemma \ref{lemma:posterior_variance_info_gain} necessitates summing over all $t\in\mathcal{M}_T(\epsilon)$. Thus, we obtain the following lemma which is unique to our analysis.

Next we present a technical result regarding a property of the centered density $\tau(z) = z \Phi(z) + \phi(z)$ at $z$-score $z_{t-1}(\bbx) = (\mu_{t-1}(\bbx) - y^{\text{max}}_{t-1} )/\sigma_{t-1}(\bbx)$.
\begin{lemma}\label{lemma:z_score}
	The negative $z$ score of the centered density function $\tau(z) = z \Phi(z) + \phi(z)$ at $z_{t-1}(\bbx) = (\mu_{t-1}(\bbx) - y^{\text{max}}_{t-1} )/\sigma_{t-1}(\bbx)$ may be upper-bounded as
	\begin{equation}\label{eq:z_score}
		\tau(-z_{t-1}(\bbx_t)) \leq 1 + R \text{ where } R:=\sup_{t\geq 0}\sup_{\bbx\in\ccalX} \frac{|\mu_{t-1}(\bbx) - y^{\text{max}}|}{\sigma_{t-1}(\bbx)}
	\end{equation}
\end{lemma}
\begin{proof}
	The properties of $\tau(z)$ depend on the sign of $\mu_{t-1}(\bbx) - y^{\text{max}}$. Thus, we break the proof into two parts. First, suppose $\mu_{t-1}(\bbx) - y^{\text{max}} > 0$. Then, we can apply the property $\tau(z) \leq 1 + z$  for $z\geq 0$ to $\tau(-z_{t-1}(\bbx_t))$ to write
	$$ 
	\tau(-z_{t-1}(\bbx_t)) \leq 1 + \frac{y^{\text{max}} - \mu_{t-1}(\bbx)}{\sigma_{t-1}(\bbx_t) } \leq R
	$$
	On the other hand, for $\mu_{t-1}(\bbx) - y^{\text{max}} \leq 0$, we may apply the property $\tau(z) \leq \phi(z) \leq 1$ for $z \leq 0$ to write:
	$$ 
	\tau(-z_{t-1}(\bbx_t)) \leq \frac{1}{\sqrt{2 \pi}} \exp\{ -\frac{1}{2} z^2_{t-1}(\bbx_t) ) \} \leq 1
	$$
	The preceding expressions taken together permit us to conclude \eqref{eq:z_score}. 
\end{proof}

We underscore that Lemma \ref{lemma:z_score} exploits properties of the shifted Gaussian density $\tau(z)$ which does not depend on whether the GP is dense or compressed, and therefore identical logic applies to $\tau$ in the context dense \eqref{eq:expected_improvement} or compressed expected improvement \eqref{eq:compressed_expected_improvement}. 
\subsection{Proof of Theorem \ref{theorem:ei_regret}}\label{sec:appendix_ei_theorem}
With these lemmas, we are ready to shift focus to the proof of the main theorem. We follow the general strategy of \citep{nguyen2017regret}[Theorem 4] except that we must also address compression-induced errors. Begin then by considering the instantaneous regret $r_t = f(\bbx^*) - f(\bbx_t)$, to which we add and subtract ${y}_{t-1}^{\text{max}}$:
\begin{equation}\label{eq:EI_regret_begin}
	r_t = f(\bbx^*) - f({\bbx}_t) = \underbrace{f(\bbx^*)   - {y}_{t-1}^{\text{max}}}_{A_t}+ \underbrace{{y}_{t-1}^{\text{max}} - f({\bbx}_t)}_{B_t}
\end{equation}
{Similar to the analysis of CGP-UCB, we note here that the regret in \eqref{eq:EI_regret_begin} defined for the instances $t\in\mathcal{M}_T(\epsilon)$.} We restrict focus to $A_t$, the first term on the right-hand side of the preceding expression, provided that ${I}_t(\bbx^*)= f(\bbx^*)   - {y}_{t-1}^{\text{max}} > 0$:
\begin{align}\label{eq:A_t}
	A_t= f(\bbx^*)   - {y}_{t-1}^{\text{max}}
	&= {I}_t(\bbx^*) \nonumber \\
	&\leq {\alpha}_t^{\text{EI}}(\bbx^*) + \sqrt{\beta_t}{\sigma}_{t-1}(\bbx^*)
\end{align}
where we apply the inequality that relates the expected improvement to the upper-confidence bound in Corollary \ref{corollary:expected_compressed_instantaneous_improvement}. Next, use the optimality condition of the action selection ${\alpha}_t^{\text{EI}}(\bbx^*) \geq {\alpha}_t^{\text{EI}}(\bbx) $ with the identity $ {\alpha}_t^{\text{EI}}(\bbx) = {\sigma}_{t-1}(\bbx) \tau({z}_{t-1}(\bbx))$ in Corollary \ref{corollary:ei_tau_expression} to write 
\begin{align}\label{eq:A_t2}
	{\alpha}_t^{\text{EI}}(\bbx^*) + \sqrt{\beta_t}{\sigma}_{t-1}(\bbx^*) 
	&\leq  {\alpha}_t^{\text{EI}}(\bbx) + \sqrt{\beta_t}{\sigma}_{t-1}(\bbx^*) \nonumber \\
	&=  {\sigma}_{t-1}(\bbx) \tau({z}_{t-1}(\bbx))+ \sqrt{\beta_t}{\sigma}_{t-1}(\bbx^*) 
\end{align}

Now let's shift gears to $B_t$, the second term on the right-hand side of \eqref{eq:EI_regret_begin}. Add and subtract ${\mu}_{t-1}({\bbx}_t)$
\begin{align}\label{eq:B_t}
	B_t &= {y}_{t-1}^{\text{max}} - {\mu}_{t-1}({\bbx}_t) + {\mu}_{t-1}({\bbx}_t)-  f({\bbx}_t) \nonumber  \\
	&\leq {y}_{t-1}^{\text{max}} - {\mu}_{t-1}({\bbx}_t) + {\sigma}_{t-1}({\bbx}_t) \sqrt{\beta_t} \qquad \text{with prob. } \  1 - \delta \nonumber  \\
	&= {\sigma}_{t-1}({\bbx}_t)(-{z}_{t-1}({\bbx}_t)+ {\sigma}_{t-1}({\bbx}_t) \sqrt{\beta_t} \qquad \text{with prob. } \  1 - \delta 
\end{align}
The first inequality comes from the property of the upper-confidence bound (Lemma \ref{lemma1}) for finite discrete decision sets $\ccalX$, which holds with probability $1-\delta$. 
The second equality comes from the definition of ${z}_{t-1}({\bbx}_t):= (\mu_{t-1}({\bbx}_t) - {y}^{\text{max}}_{t-1} )/{\sigma}_{t-1}({\bbx}_t)$ by multiplying through by $-{\sigma}_{t-1}({\bbx}_t)$. Subsequently, we suppress the high probability qualifier, with the understanding that it's implicit.  We rewrite the preceding expression using the fact that $z=\tau(z) - \tau(-z)$ 
\begin{align}\label{eq:B_t2}
	{\sigma}_{t-1}({\bbx}_t)(-{z}_{t-1}({\bbx}_t)&+ {\sigma}_{t-1}({\bbx}_t) \sqrt{\beta_t} \nonumber \\
	&=  {\sigma}_{t-1}({\bbx}_t) \Big[\tau(-{z}_{t-1}({\bbx}_t)) + \sqrt{\beta_t} - \tau({z}_{t-1}({\bbx}_t)) \Big]
\end{align}
Now, let's return to \eqref{eq:EI_regret_begin}, substituting in the right-hand sides of \eqref{eq:A_t2} and \eqref{eq:B_t2} for $A_t$ and $B_t$, respectively, to obtain:
\begin{align}\label{eq:AB_combine}
	r_t &\leq {\sigma}_{t-1}({\bbx}_t)\Big[ \sqrt{\beta_t} + \tau(-{z}_{t-1}({\bbx}_t)) \Big] + \sqrt{\beta_t}{\sigma}_{t-1}(\bbx^*) \nonumber \\
	& \leq \underbrace{{\sigma}_{t-1}({\bbx}_t)\Big[ \sqrt{\beta_t} + 1+R \Big] }_{L_t}+ \underbrace{\sqrt{\beta_t}{\sigma}_{t-1}(\bbx^*) }_{U_t}
\end{align}
where we apply Lemma \ref{lemma:z_score} to $\tau(-{z}_{t-1}({\bbx}_t)) $ in the preceding expression. {Taking the square on both sides of \eqref{eq:AB_combine}, we get}
\begin{align}
	r_t^2\leq L_t^2+U_t^2
\end{align}
{which holds for all $t$.} \blue{ The definition of $R$ is in \eqref{eq:z_score}. First, we focus on the square of $L_t$ on the right-hand side of \eqref{eq:AB_combine}, which we sum \blue{for $t=1$ to $T$}:
\begin{align}\label{eq:L_t}
	{\sum_{t=1}^T} L_t^2 = {\sum_{t=1}^N} {\sigma}_{t-1}^2({\bbx}_t)\Big[ \sqrt{\beta_t} +1 + R \Big]^2
\end{align}
Apply the sum-of-squares inequality $(a+b+c)\leq 3 (a^2 + b^2 + c^2)$ to obtain
\begin{align}\label{eq:L_t2}
	{{\sum_{t=1}^T}}{\sigma}_{t-1}^2({\bbx}_t)\Big[ \sqrt{\beta_t} + 1+ R \Big]^2
	& \leq  {{\sum_{t=1}^T}} {\sigma}_{t-1}^2({\bbx}_t)3\Big[ \beta_t + 1 + R^2 \Big] \nonumber \\
	& \leq 3\Big[ \beta_T + 1 + R^2 \Big]{{\sum_{t=1}^T}} {\sigma}_{t-1}^2({\bbx}_t)
\end{align}
where we use the fact that $\beta_t \leq \beta_T$. Now, apply Lemma \ref{lemma:posterior_variance_info_gain} to the right-hand side of the preceding expression to obtain
\begin{align}\label{eq:L_t2}
	3\Big[ \beta_T + 1 + R^2 \Big] {\sum_{t=1}^T} {\sigma}_{t-1}^2({\bbx}_t) \leq  \frac{6(\beta_T + 1 + R^2)}{\log\left(1+\sigma^{-2}\right)} \left[ {\gamma_{M_T(\epsilon)}}+\epsilon |\mathcal{M}_T^C(\epsilon)|\right]
\end{align}
to which we further apply Cauchy-Schwartz to obtain
\begin{align}\label{eq:L_t3}
{\sum_{t=1}^T}L_t^2  \leq {\frac{6(\beta_T + 1 + R^2)}{\log\left(1+\sigma^{-2}\right)}}\left[ {\gamma_{M_T(\epsilon)}}+\epsilon |\mathcal{M}_T^C(\epsilon)|\right]
\end{align}
Now, we shift focus back to $U_t^2$ in \eqref{eq:AB_combine} to which we apply $\beta_t \leq \beta_T$, Cauchy-Schwartz ( in the form of the sum of squares inequality), and Lemma \ref{lemma:posterior_variance_info_gain}:
\begin{align}\label{eq:U_t}
	{\sum_{t=1}^T} U_t^2 = {\sum_{t=1}^T} {\beta_t}{\sigma}_{t-1}^2(\bbx^*) & \leq {\beta_T}{\sum_{t=1}^T}{\sigma}_{t-1}^2(\bbx^*)\nonumber 
	\\
	& \leq {\frac{2 \beta_T }{\log\left(1+\sigma^{-2}\right)}} \left[ {\gamma_{M_T(\epsilon)}}+\epsilon |\mathcal{M}_T^C(\epsilon)|\right]
\end{align}
Now, we can aggregate the inequalities in \eqref{eq:L_t3} and \eqref{eq:U_t}, together with the fact that ${\textbf{Reg}_T} = \sum_{t=1}^T(f(\bbx^*) - f({\bbx}_t))$ satisfies 
\begin{align}\label{eq:EI_proof_final}
	{\textbf{Reg}_T} \leq  \sqrt{T \sum_{t=1}^Tr_t^2}
	= &  \sqrt{T \sum_{t=1}^T L_t^2+U_t^2 }\nonumber 
	\\
	\leq&\sqrt{\frac{6T(\beta_T + 1 + R^2)\left[ {\gamma_{M_T(\epsilon)}}+\epsilon |\mathcal{M}_T^C(\epsilon)|\right]}{\log\left(1+\sigma^{-2}\right)}} 
	\nonumber
	\\
	&+ \sqrt{\frac{2T \beta_T \left[ {\gamma_{M_T(\epsilon)}}+\epsilon |\mathcal{M}_T^C(\epsilon)|\right] }{\log\left(1+\sigma^{-2}\right)}} \nonumber \\
	=&\left[\sqrt{3(\beta_T +1 +R^2)} + \sqrt{\beta_T}\right]\sqrt{ \frac{ 2 T\left[ {\gamma_{M_T(\epsilon)}}+\epsilon |\mathcal{M}_T^C(\epsilon)|\right]}{\log\left(1+\sigma^{-2}\right)}}.
\end{align}
Utilizing the bound $|\mathcal{M}_T^C(\epsilon)|\leq T$ and substituting the selection $\epsilon={\frac{1}{2}\log\left(1+T^{-\alpha}\right)}$ with $\alpha\in(0,\frac{1}{p})$ to obtain
\begin{align}
	{\textbf{Reg}_T} \leq &  \frac{\sqrt{3(\beta_T +1 +R^2)} + \sqrt{\beta_T}}{\log\left(1+\sigma^{-2}\right)}\sqrt{2T \gamma_{M_T(\epsilon)} }
	\nonumber
	\\
	&+ \frac{\sqrt{3(\beta_T +1 +R^2)} + \sqrt{\beta_T}}{\log\left(1+\sigma^{-2}\right)}{T\sqrt{\log\left(1+T^{-\alpha}\right)} } \label{final_bound3}. 
\end{align}
{Hence proved.}}

\end{document}